\documentclass{article}

\usepackage{arxiv}

\usepackage{amsmath, amssymb}
\usepackage{enumitem}
\usepackage[dvipsnames]{xcolor}
\definecolor{LinkColor}{HTML}{1F73C9}
\definecolor{CiteColor}{HTML}{3F6278}

\definecolor{myblue}{HTML}{1F73C9}
\definecolor{myred}{HTML}{B85E1A}
\definecolor{plotblue}{HTML}{1677FF}
\definecolor{plotorange}{HTML}{FA8C16}

\usepackage{amsthm}

\usepackage[utf8]{inputenc} %
\usepackage[T1]{fontenc}    %
\usepackage{hyperref}       %
\usepackage{url}            %
\usepackage{booktabs}       %
\usepackage{diagbox}
\usepackage{microtype}      %
\usepackage[nameinlink]{cleveref}       %
\usepackage{graphicx}
\usepackage{subcaption}
\usepackage{wrapfig}
\usepackage[round]{natbib}
\usepackage{algorithm}    %

\usepackage[most]{tcolorbox}

\newtcolorbox{keypropertybox}{
	colback=black!1,
	colframe=black!18,
	boxrule=0.5pt,
	arc=2pt,
	left=0.9em,
	right=0.9em,
	top=0.65em,
	bottom=0.65em,
	before skip=0.8em,
	after skip=0.8em,
}

\newtcolorbox{titledkeypropertybox}[1]{
	enhanced,
	colback=black!1,
	colframe=black!18,
	boxrule=0.5pt,
	arc=2pt,
	left=0.9em,
	right=0.9em,
	top=0.85em,
	bottom=0.65em,
	before skip=0.8em,
	after skip=0.8em,
	title={\strut #1},
	coltitle=black,
	fonttitle=\bfseries,
	attach boxed title to top left={
		xshift=1.2em,
		yshift=-\tcboxedtitleheight/2,
	},
	boxed title style={
		colback=black!1,
		colframe=black!1,
		boxrule=0pt,
		arc=0pt,
		left=0.4em,
		right=0.4em,
		top=0pt,
		bottom=0pt,
	},
}

\renewcommand{\epsilon}{\varepsilon}

\newcommand{\R}{\mathbb{R}}

\renewcommand{\P}{\mathbb{P}}

\newcommand{\E}{\mathbb{E}}
\DeclareMathOperator{\Var}{Var}

\theoremstyle{definition}
\newtheorem{proposition}{Proposition}
\newtheorem*{proposition*}{Proposition}
\newtheorem*{example}{Example}
\newtheorem{problem}{Problem}

\crefname{section}{Section}{Sections}
\Crefname{section}{Section}{Sections}
\crefname{subsection}{Section}{Sections}
\Crefname{subsection}{Section}{Sections}
\crefname{subsubsection}{Section}{Sections}
\Crefname{subsubsection}{Section}{Sections}
\crefname{appendix}{Appendix}{Appendices}
\Crefname{appendix}{Appendix}{Appendices}
\crefname{figure}{Fig.}{Figs.}
\Crefname{figure}{Fig.}{Figs.}
\crefname{table}{Table}{Tables}
\Crefname{table}{Table}{Tables}
\crefname{equation}{Eq.}{Eqs.}
\Crefname{equation}{Eq.}{Eqs.}
\crefname{proposition}{Proposition}{Propositions}
\Crefname{proposition}{Proposition}{Propositions}

\newcommand{\obsspace}{\mathcal{Y}}
\newcommand{\actspace}{\mathcal{U}}

\title{Robust Control under Stationary Ambiguity}

\date{}

\usepackage{authblk}

\makeatletter
\renewcommand\AB@affilsepx{ \quad }
\makeatother
\author[1,2]{Konrad J. Mueller\thanks{Corresponding author: \texttt{k.mueller23@imperial.ac.uk}}}
\author[2]{Amira Akkari}
\author[2]{Ben Wood}
\author[1,3]{Lukas Gonon}

\affil[1]{Imperial College London}
\affil[2]{JPMorgan Chase \& Co.\thanks{Opinions expressed in this paper are those of the authors, and do not necessarily reflect the view of JPMorgan Chase~\&~Co.}}
\affil[3]{University of St.\ Gallen}

\renewcommand{\headeright}{}
\renewcommand{\undertitle}{}
\renewcommand{\shorttitle}{}

\hypersetup{
pdftitle={Robust Control under Stationary Ambiguity},
pdfsubject={robust control, ambiguity, hedging},
pdfauthor={Konrad J. Mueller, Amira Akkari, Ben Wood, Lukas Gonon},
pdfkeywords={robust control, stationary ambiguity, domain randomization, hedging},
colorlinks=true,
citecolor=CiteColor,
linkcolor=LinkColor,
urlcolor=CiteColor,
}

\usepackage{titletoc}

\newcommand{\AppendixToC}{%
  \startcontents[appendix]%
  \printcontents[appendix]{}{1}{%
    \section*{Appendix}%
    \vspace{-0.75em}%
    \hrule%
    \vspace{0.5em}%
  }%
  \vspace{1em}
  \hrule
  \vspace{1em}
}

\begin{document}
\maketitle

\begin{abstract}

Control policies optimized in simulation can perform poorly in the real system
when the parameters~$x$ of the simulator are estimated from limited data
but the resulting parameter uncertainty is not represented inside the simulation.
A common way to incorporate such \emph{ambiguity}
is to simulate each trajectory of the system
under a randomly drawn value for $x$.
Since the policy cannot observe the drawn value,
it must initially choose controls that perform well across many possible parameter values.
However, if the policy progressively observes the system,
it can often gradually infer the value of $x$, so that ambiguity vanishes.
Over time, the policy then specializes to its estimate of $x$ and loses its robustness.
This is undesirable in many real systems, where latent factors are expected to shift.
In financial markets, for example,
a policy hedging a derivative payoff should remain robust to
changes in the volatility regime.
To induce such continual robustness,
we propose training policies in simulators where ambiguity
varies with the system's state but does not systematically decay over time.
We formalize this requirement as \emph{stationary~ambiguity}:
the simulator should induce a stationary filter process over the latent state.
We show how to construct such simulators and demonstrate, on hedging problems,
that policies trained under stationary ambiguity
preserve robustness to latent factors over time,
leading to strong performance on real market data.
As a modeling principle, stationary ambiguity informs many simulator design decisions:
which models make realistic simulators,
how their parameters should be randomized,
and how simulator and policy should be initialized.
While our experiments focus on hedging,
stationary ambiguity may also be useful for other sequential control problems
driven by exogenous stochastic processes with shifting latent structure.

\end{abstract}

\section{Introduction}\label{sec:intro}

\begin{wrapfigure}[19]{r}{0.5\textwidth}
	\vspace{-1em}
	\centering
	\includegraphics[width=\linewidth]{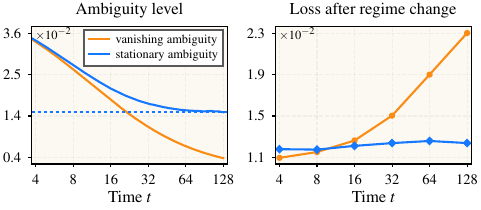}
	\vspace{-1.5em}
	\caption{
	Ambiguity decay leads to robustness decay
	in a hedging problem with uncertain volatility.
	We train one policy in a simulator with
	\textcolor{plotorange}{vanishing ambiguity}
	and one in a simulator with
	\textcolor{plotblue}{stationary ambiguity}.
	\emph{Left:} simulation-induced ambiguity over time;
	under stationary ambiguity, the level settles in steady state (dashed) after warm-up.
	\emph{Right:} worst-case loss in a regime-shift stress test.
	The policy first observes $t$ returns under a volatility regime $x^{\mathrm{pre}}$.
	At time $t$, the hedging problem starts and the regime shifts to
	$x^{\mathrm{eval}}$.
	If $t$ is large and most ambiguity has vanished,
	the policy has specialized to $x^{\mathrm{pre}}$
	and is less robust to the regime change.}
	\label{fig:fig1}
	\vspace{-1em}
\end{wrapfigure}

Finding effective policies for stochastic control problems is often challenging
when data from the real system is limited.
A common approach is then to optimize a policy
in a simulator that emulates the real system.
The simulator usually depends on parameters~$x$
that have to be estimated from the limited data
and are therefore subject to parameter uncertainty (\emph{ambiguity}).
Simply plugging a parameter estimate $\hat x$ into the simulator
and optimizing the policy under this plug-in simulator
treats $\hat x$ as known and ignores this ambiguity.
The resulting policy can become specialized
to the dynamics induced by $\hat x$
and perform poorly when deployed on the real system.

A common way to obtain a more robust policy
is to optimize in a simulator in which $x$ is treated as random
\citep{vidyasagar_2001_rando,tobin_2017_domai}.
For example, suppose that the control problem depends on an exogenous
stochastic process~$Y$,
and that the simulator generates $Y$ as a Markov process
with parameter $x$.
One way to introduce ambiguity into the simulation
is to replace the plug-in value $\hat x$
by a random parameter sampled from a distribution~$\nu$:
\begin{equation*}
	Y_t \mid Y_{t-1}
	\sim p(\cdot \mid Y_{t-1}, \hat x)
	\quad
	\xrightarrow{\ \text{randomize } x \ }
	\quad
	X \sim \nu,
	\quad
	Y_t \mid (Y_{t-1}, X)
	\sim p(\cdot \mid Y_{t-1}, X).
\end{equation*}
A policy trained in this randomized simulator
observes the process $Y$,
but not the sampled parameter $X$.
At the start of a simulated path,
the policy therefore faces ambiguity about the realized parameter value,
so it has to choose actions that perform well across
a range of plausible parameter values.
But because the policy progressively observes $Y$,
it can often estimate $X$ with increasing precision over time.
The policy then faces less ambiguity about $X$
and increasingly specializes its actions to its estimate of $X$.

Under this \emph{vanishing ambiguity},
the policy's initial robustness to the uncertain parameter
is therefore progressively lost (\cref{fig:fig1}).
This loss of robustness is undesirable when we expect
the parameters that best describe the real system to change over time.
For example, many control problems in finance
require simulating paths of asset prices or returns.
The simulator parameters that describe current market conditions
often drift or jump
\citep{lamoureux_1990_persi,ang_2012_regim,dangl_2012_predi},
so beliefs about those parameters should not be expected to concentrate over time
\citep{epstein_2007_learn,ju_2012_ambig,nagel_2022_asset}.

In such domains,
the policy should thus maintain its robustness to the simulator parameters
over time.
To achieve this, we argue that
ambiguity should be able to vary stochastically with the state of the simulated system,
but not systematically with the clock of the control problem.
We formalize this as follows:
consider a simulator
with a latent process $X$ and an observable process $Y$,
such that the joint process $(X, Y)$ is Markov.
We characterize ambiguity through the filter
and require the simulation setup to induce a stationary filter process:
\begin{titledkeypropertybox}{Stationary ambiguity}
	\begin{center}
	The simulator $(X,Y)$ induces a stationary filter process
	$(\pi_t)_t$,
	where
	$\pi_t := \mathbb{P}(X_t \in \cdot \mid Y_{\leq t})$.
	\end{center}
\end{titledkeypropertybox}
Under stationary ambiguity, the average amount of ambiguity about $X_t$,
measured by $\mathbb E[\Var(X_t \mid Y_{\leq t})]$,
is constant across $t$ in the stationary regime
(dashed line in left panel of \cref{fig:fig1}).
Stationary ambiguity therefore rules out the systematic ambiguity decay
induced by the randomization scheme above.

Throughout the paper,
we study how simulation-implied ambiguity dynamics affect policies
and show that policies trained under stationary ambiguity
maintain their robustness to latent factors over time.
We examine these effects in derivative hedging problems,
which are stochastic control problems that are fundamental to financial risk management
and commonly solved using simulated data \citep{buehler_2019_deep,wiese_2021_multia}.
Robustness to uncertain simulator parameters is especially important
for hedging problems,
because relevant real data is often limited
\citep{cohen_2023_black,limmer_2024_robus,he_2025_distr}.
Our contributions can be summarized as follows:

\begin{enumerate}
	\item We study how vanishing ambiguity in simulation
	affects policies and their robustness.
	In an optimal investment problem (\cref{sec:setup})
	and across three hedging problems (\cref{sec:synth_experiments}),
	we show that policies optimized under vanishing ambiguity
	specialize to their estimate of the latent simulator parameters
	and lose robustness over time.
	\item We introduce the modeling principle of stationary ambiguity,
	under which ambiguity
	does not systematically vary with the clock of the control problem.
	We show how to approximate stationary ambiguity in practice
	by using a jointly stationary Markov simulator
	and providing warm-up observations to the policy
	(\cref{sec:ambiguity}).
	\item We show that policies optimized under stationary ambiguity
	maintain robustness to latent simulator parameters over time
	(\cref{sec:setup,sec:synth_experiments}).
	In a large-scale hedging case study,
	we find that these policies outperform those optimized
	under vanishing ambiguity when deployed on real market data
	(\cref{sec:real_data}).
	\item We outline implications of stationary ambiguity
	for commonly used simulation schemes
	and give a practical recipe for choosing, initializing,
	and randomizing simulators under this principle
	(\cref{sec:practical_recommendations}).
\end{enumerate}

Stationary ambiguity can guide the design and use of
simulators with unobserved latent states,
even when the simulator itself is well established.
For example,
in the \citet{heston_1993_a} stochastic volatility model,
it is common practice to initialize the latent variance process $V$
deterministically ($V_0 = v_0$).
The policy then faces no ambiguity about $V_0$
but becomes ambiguous about the variance level at later times $t > 0$.
This violates stationary ambiguity and makes the policy overconfident
at the start of the control problem.
In \cref{sec:practical_recommendations},
we discuss how to initialize the variance process
and warm up the policy
so that ambiguity is approximately stationary from $t=0$ onward.

Moreover,
stationary ambiguity may be useful for
many real-world control problems beyond hedging and finance.
For instance, many inventory, storage, and queueing problems
depend on an exogenous stochastic process~$Y$
that the policy progressively observes but cannot affect.
When the dynamics of $Y$ are uncertain and the simulator parameters
that best describe them are expected to change over time,
stationary ambiguity is a natural modeling assumption.

\section{Simulator randomization with vanishing ambiguity}\label{sec:setup}
Throughout the paper,
random variables are defined on an underlying probability space
$(\Omega,\mathcal E,\mathbb P)$.
We use $\mathcal X$ to denote a latent state space
and $\mathcal Y$ to denote an observation space.
Both $\mathcal X$ and $\mathcal Y$ are assumed to be Polish spaces.

\subsection{Simulator-based policy optimization}

We study control problems that involve an exogenous stochastic process $Y$,
meaning the controller cannot affect it via its actions.
In financial control problems, such as optimal investment and hedging,
$Y_t$ may be the return of a traded asset over period $t$.
Many control problems of this form
have traditionally been studied under structured dynamics for~$Y$,
where optimal policies can be characterized analytically or
by specialized numerical methods.
More recently, it has become practical to approximate optimal policies
under broader dynamics for $Y$ by:
(i)~specifying a simulator for $Y$,
(ii)~parameterizing the policy as a neural network $f_\theta$,
and (iii)~optimizing $\theta$ by stochastic gradient descent on simulated paths of $Y$
\citep{bertsekas_2008_neuro,heess_2015_learna,han_2016_deep,buehler_2019_deep}.

The ability to optimize policies under broad simulator classes
still leaves the question of which simulator to train on.
A common approach is to choose a simulator class with parameters $x$,
compute an estimate~$\hat x$ from historical data,
and train the policy under the simulator with $x$ fixed at $\hat x$.
Such a \emph{plug-in} approach is simple,
but treats the fitted parameter $\hat x$ as known.
In reality, $\hat x$ is estimated from limited historical data and is therefore uncertain.
We call this uncertainty about the simulator parameter \emph{ambiguity}.
When the optimal policy is sensitive to the value of $x$,
training under the plug-in simulator
can produce a policy that is overly specialized to the estimate $\hat x$
and performs poorly when deployed on the real system
whose dynamics may differ from those implied by $\hat x$.

This observation motivates a variety of methods that aim
to make the learned policy robust to the simulator and its parameters
\citep{zhou_1998_essen,iyengar_2005_robus,wiesemann_2013_robus,delage_2010_distr,morimoto_2005_robus,pinto_2017_robus}.
One natural approach is to attempt to
incorporate ambiguity about $x$ directly into the simulation.
If the simulator is run with a fixed value of $x$, then the policy faces
no ambiguity about that parameter and can specialize to it.
To reflect ambiguity about $x$, one must therefore vary the parameter
in the simulation without revealing the realized value to the policy.
This is commonly referred to as \emph{randomizing} the
simulator \citep{tobin_2017_domai,peng_2018_simto}.
In such a randomized simulator it is no longer optimal
to specialize to a single value of $x$.
The policy should therefore learn to make decisions that perform well
across a range of plausible parameter values.

\subsection{Vanishing ambiguity from static randomization}\label{subsec:non_stationary_amb}
The simplest way to randomize is to draw the parameter once at the
beginning of each simulated path and keep it fixed over time
(\emph{static randomization}).
For example, consider a simulator with parameter $x \in \mathcal{X}$
and $\mathcal{Y}$-valued observations $Y_{1}, \dots, Y_T$
that are sampled i.i.d.\ from a parametric distribution~$p(\cdot \mid x)$.
Under a static randomization of this simulator,
the parameter~$x$ is replaced
by a random variable $X$ that is drawn at the start of the path:
\begin{equation}\label{eq:slm}
	X \sim \nu,
	\qquad
	Y_t \mid X \sim p(\cdot \mid X),
	\quad t \ge 1.
\end{equation}
We call this randomized simulator the \emph{static latent model} (SLM).
The realized draw of $X$ is not observable,
but we assume the policy can progressively observe the process $Y$.
The uncertainty about $X$
is therefore characterized through the posterior distribution
\begin{equation*}
	\mathbb{P}(X \in \cdot \mid \mathcal{G}_t),
	\quad \text{where} \quad
	\mathcal{G}_t = \sigma(Y_1, \dots, Y_t) .
\end{equation*}
At time $t = 0$, the policy has not made any observations yet
($\mathcal{G}_0$ is trivial),
and the posterior distribution coincides with the randomization distribution $\nu$.
In this sense, the static randomization scheme successfully transfers
the modeling ambiguity about $x$, as described by the distribution $\nu$,
into ambiguity faced by the policy within the simulation.

Because all observations $Y_{1:t}$ are generated under the same realization of $X$,
each new observation provides additional information about the realized value.
In many simulators, the policy can therefore infer that value
with increasing precision over time.
More precisely, under an additional identifiability assumption,
the posterior distribution concentrates over time
\citep{doob1949,miller_2018_aa}.
A detailed statement is given in \cref{app:subsec:main_proofs}.
\begin{proposition}[Doob's theorem]\label{prop:slm_vanishing_amb}
	Consider the static latent model of \cref{eq:slm}.
	Let $\varphi:\mathcal X\to\mathbb R$ be measurable with
	$\varphi(X)\in L^2$.
	If the model is identifiable so that
	${p(\cdot \mid x) \neq p(\cdot \mid x')}$
	whenever $x \neq x'$,
	then, $\mathbb{P}$-almost surely as $t \to \infty$,
	\begin{equation*}
		\mathbb{P}(X \in \cdot \mid \mathcal{G}_t) \Rightarrow \delta_{X}
		\quad
		\text{and}
		\quad
		\Var(\varphi(X) \mid \mathcal{G}_t)
		\rightarrow 0 ,
	\end{equation*}
	where $\Rightarrow$ denotes weak convergence.
\end{proposition}
The posterior variance $\Var(\varphi(X) \mid \mathcal{G}_t)$
can be interpreted as the \emph{amount} of ambiguity
about a property $\varphi$ of the latent parameter at time $t$.
In this sense,
the ambiguity introduced by static randomization
is progressively resolved as more observations are revealed.
We refer to this systematic decay of ambiguity over time as
\emph{vanishing ambiguity}.
While \cref{prop:slm_vanishing_amb} is stated for an i.i.d.\ simulator
with identifiable parameter $x$,
the issue of ambiguity decay due to static randomization is more general.
Indeed, for any simulator for the path $Y$ with parameter $x$,
static randomization implies that for any $\varphi(X) \in L^2$,
\begin{equation*}
	\mathbb E\bigl[\Var(\varphi(X) \mid \mathcal G_t)\bigr]
	\; \le \;
	\mathbb E\bigl[\Var(\varphi(X) \mid \mathcal G_s)\bigr],
	\qquad t \ge s,
\end{equation*}
by the law of total variance.
Under static randomization, ambiguity about any property $\varphi(X)$
of the latent parameter is therefore non-increasing in expectation.

Such ambiguity dynamics are unrealistic for systems
in which the latent regime that best describes the current dynamics may shift over time.
In particular, financial markets are non-episodic and our models of those markets
are subject to parameter drift and regime changes.
How much ambiguity is present at any given time depends on the current market
conditions and does not decay systematically over time.
Similar observations have been made in prior works
\citep{epstein_2007_learn,ju_2012_ambig}; see discussion in \cref{subsec:related_work}.
Optimizing a policy under static randomization therefore
creates a \emph{sim-to-real gap} in the ambiguity dynamics.
During training, the policy learns to act with increasing confidence
on its estimate of the latent regime over time.
During deployment, the policy may face regime changes
after it has specialized to its inferred regime,
leading it to make decisions that are poorly suited to the new regime.

\subsection{Why ambiguity dynamics matter for control}\label{subsec:opt_inv_example}

\begin{figure}[t]
	\centering
	\includegraphics[width=\textwidth]{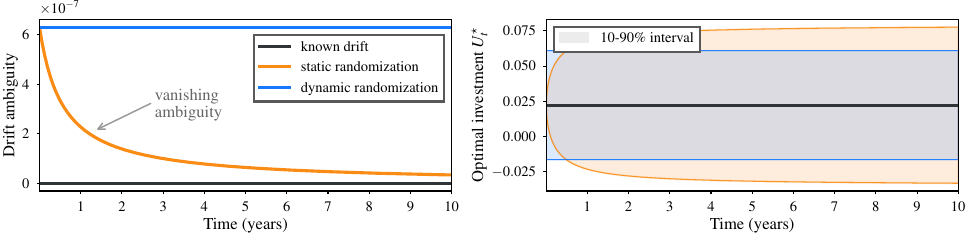}
	\caption{Vanishing ambiguity leads to specialization of optimal investment.
	\emph{Left:} drift ambiguity, measured by
	$\Var(X \mid \mathcal G_t)$ under static randomization
	and $\Var(X_t \mid \mathcal G_t)$ under dynamic randomization;
	known drift has zero ambiguity.
	\emph{Right:} $10$--$90\%$ interval of the optimal investment
	$U_t^\star$ across paths.
	Under static randomization, observations gradually reveal the fixed latent drift
	over time, so $U_t^\star$ specializes to a pathwise estimate of $X$
	and its distribution spreads out.
	Under dynamic randomization, both $\Var(X_t \mid \mathcal G_t)$ and the
	distribution of $U_t^\star$ are constant over time.}
	\label{fig:opt_inv}
\end{figure}

We illustrate the effect of ambiguity dynamics on controls
in a simple optimal investment problem.
Consider an asset with return process $Y$
and let $U_t$ denote the position invested into the asset at time $t$.
Let $\mathbb G$ be the filtration generated by $Y$.
The objective is
\begin{equation*}
	\inf_{U \; \mathbb G\text{-predictable}}
	\quad
	-
	\sum_{t=1}^T
	\mathbb E\bigl[
			u(U_t Y_t)
	\bigr],
	\qquad
	u(r) := 1 - \exp(-\lambda r),
\end{equation*}
with risk aversion $\lambda > 0$.\footnote{
This objective is additive across time,
so that the optimal investment at time $t$ is independent of investments at other dates.
We make this choice to isolate the effect of ambiguity dynamics on investment decisions.}
Consider a simple parametric simulator for the return process $Y$, in which
the returns are i.i.d.\ with law $N(x,\sigma^2)$.
We compare the base simulator under the assumption that the drift is known
($x = \bar x$),
with two randomization strategies that incorporate
uncertainty about the drift into the simulation.
First, we use the static randomization scheme of \cref{eq:slm},
sampling a static drift $X \sim N(\bar x, \tau^2)$ on each path for $Y$.
Second, we use an AR(1) process to
vary the drift over time (\emph{dynamic randomization}).
The randomized simulator is therefore the linear Gaussian state-space model (SSM)
\begin{equation*}
	X_t = \bar x + \phi(X_{t-1}-\bar x) + \varepsilon_t,
	\qquad
	Y_t = X_t + \eta_t,
	\qquad t \in \mathbb Z,
\end{equation*}
where $|\phi|<1$ and $\varepsilon_t,\eta_t$ are independent centered Gaussian
noises with $\Var(\eta_t)=\sigma^2$.
We run this simulator on $\mathbb Z$ in its stationary regime with full observation
history and
choose $\Var(\varepsilon_t)$ so that
$\Var(X_t \mid \mathcal G_t)=\tau^2$ for all $t$.
We provide all details in \cref{app:sec:optimal_investment_example}.

The left panel of \cref{fig:opt_inv} shows the drift ambiguity,
as measured by
$\Var(X \mid \mathcal{G}_t)$ for the static simulator
and
$\Var(X_t \mid \mathcal{G}_t)$ for the dynamic simulator.
When the drift is known, there is no ambiguity.
Under static randomization, ambiguity vanishes
as shown in \cref{prop:slm_vanishing_amb}.
Under dynamic randomization,
the same positive amount of ambiguity is present at all times
as the variation in the latent drift offsets the information gained from observing $Y$.

The right panel of \cref{fig:opt_inv} shows the distribution
of the optimal investment $U_t^\star$ across paths under all three simulators.
When the drift is known, $U_t^\star$ is deterministic and constant over time.
Under static randomization, the policy gradually estimates the drift realization
and specializes its investment to its estimate on each path.
Under dynamic randomization, the distribution of $U_t^\star$ does not change with time.
Pathwise, $U_t^\star$ varies with the drift estimate over time,
but it does not become systematically more specialized to it.

\begin{figure}[t]
	\centering
	\includegraphics[width=\textwidth]{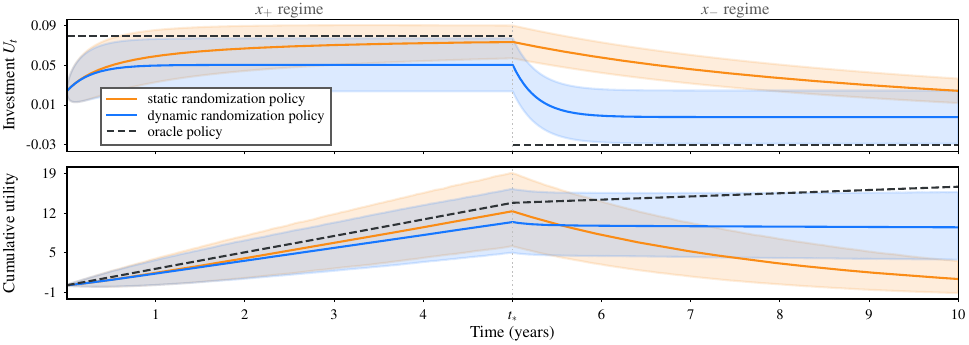}
	\caption{Static randomization policy fails to adapt to a regime switch.
	We evaluate the policies on scenarios where the latent drift switches from
	$x_+$ to $x_-$ at time $t_\ast$,
	where $x_{+}$ and $x_{-}$ are the $90\%$ and $10\%$ quantiles of $N(\bar x, \tau^2)$.
	\emph{Top:}
	mean and $10$--$90\%$ interval of $U_t$ across draws of $Y$.
	\emph{Bottom:} mean and $10$--$90\%$ interval of
	cumulative utility $\sum_{s \leq t}u(U_s Y_s)$.
	The static-randomization policy reacts slowly to the regime change,
	resulting in lower terminal cumulative utility.}
	\label{fig:opt_inv_regime_switch}
\end{figure}

The progressive specialization of the static randomization policy
is only appropriate if the latent drift is constant over time.
If latent market conditions can change,
the specialization may result in poor performance.
To illustrate this point, we evaluate the two
policies on a simple regime-switch scenario, in which the latent drift is
given by $X_t = x_+ > 0$ for $t \leq t_\ast$ and by $X_t = x_- < 0$ for $t > t_\ast$.
The regimes $x_{+}$ and $x_{-}$ are the $90\%$ and $10\%$ quantiles of the
randomization distribution $N(\bar x, \tau^2)$.
We also compare to the oracle policy that observes the drift process.

The top panel of \cref{fig:opt_inv_regime_switch} shows the investment $U_t$
and the bottom panel shows cumulative utility.
The oracle policy is deterministic and switches its investment
instantly once the regime changes.
Both the SLM and the SSM policies increase their investment on average during the first half
because most observed returns are high.
The SLM policy specializes strongly to its regime estimate,
with mean investment almost reaching the oracle investment.
The SSM policy specializes less aggressively,
resulting in lower mean cumulative utility at time $t_\ast$.

After $t_\ast$, returns are generated under the negative-drift regime,
which pulls both policies towards lower investment.
The SLM policy adjusts slowly because, even long after $t_\ast$, its drift estimate
is still affected by observations prior to $t_\ast$.
Because the SSM policy was optimized in a simulator in which regimes can change,
it puts less weight on older observations when estimating the drift.
The policy therefore reduces its investment much faster and accumulates substantially
higher utility on average over the full horizon.
Because ambiguity has largely decayed in the SLM by time $t_\ast$,
the SLM policy is not robust to the drift change.
In the SSM, ambiguity is present at all times,
and the SSM policy thus maintains robustness
with respect to changes in the latent drift regime.
In the next section, we generalize this observation
and make precise how ambiguity should be modeled
to induce such continually robust policies.

\section{Stationary ambiguity in simulation}\label{sec:ambiguity}

\subsection{Setup}\label{subsec:setup}

In this section, we introduce the general class of
stochastic control problems studied in this paper.
We assume a discrete time setting with time index set $\mathcal T$,
for which we study two choices:
(a) infinite past observations, so that $\mathcal T = \mathbb Z$,
and (b) a finite history of $H \in \mathbb N_0$ observations before time $t=1$,
so that $\mathcal T = \{t \in \mathbb Z : t \ge 1-H\}$.
Let $X = (X_t)_{t \in \mathcal T}$ and $Y = (Y_t)_{t \in \mathcal T}$
be stochastic processes,
taking values in the Polish spaces $\mathcal{X}$ and $\mathcal{Y}$ respectively,
such that the joint process
$(X_t, Y_t)_{t \in \mathcal T}$ is a Markov process.\footnote{
This is sometimes called a pairwise Markov model \citep{pieczynski_2003_pairw}.}
Let $\mathbb{G}$ denote the filtration generated by $Y$, so that
\begin{equation*}
\mathbb{G} = (\mathcal{G}_t)_{t \in \mathbb{Z}},
\quad \text{where} \quad
\mathcal{G}_t := \sigma(Y_s:s \in \mathcal{T}, s\le t) .
\end{equation*}
We consider a finite horizon stochastic control problem
over the interval $\{1, \dots, T\}$
with control process $(U_t)_{t = 1, \dots, T}$
taking values in an action space $\mathcal U$.
The control process has to be \emph{predictable} with respect to
the filtration $\mathbb G$, meaning that
$U_{t + 1}$ has to be $\mathcal{G}_t$-measurable
for all $t \in \{0,\dots,T-1\}$.
In particular, we assume that the control process $U$ does not affect
the law of $(X,Y)$.
We consider stochastic control problems of the form
\begin{equation}\label{eq:control_problem}
	\inf_{
		U \; \mathbb{G}\text{-predictable}
	}
	\quad
	\mathcal{R}\bigl(\, \ell\bigl(Y_{1:T}, U\bigr)\, \bigr).
\end{equation}
Here
$\ell : \obsspace^{T} \times \actspace^{T} \to \R$
computes the loss associated with a realized observation path and control sequence.
The functional $\mathcal{R}$
maps the resulting random loss to a scalar objective
and may for instance be a risk or deviation measure.

\paragraph{Simulation-implied ambiguity.}
The process $X$ is generally not observable to the policy,
meaning that $X_t$ is generally not $\mathcal{G}_t$-measurable.
Still, the observations $Y_{\leq t}$ may contain information about $X_t$,
as captured by the filter
\begin{equation*}
	\pi_t := \mathbb P(X_t \in \cdot \mid \mathcal G_t),
	\qquad t \in \mathcal T .
\end{equation*}
Because $\pi_t$ is the full conditional law of $X_t$ given the available
information~$\mathcal G_t$, it captures what any
$\mathbb G$-predictable policy can infer about the latent state
before deciding on the control $U_{t + 1}$.
In this sense,
the filter~$\pi_t$ characterizes the ambiguity
over the latent state $X_t$ at time $t$ in a policy-independent way
that depends only on the dynamics of $(X, Y)$.
Accordingly, we study ambiguity
via the infinite-dimensional filter process $\pi = (\pi_t)_{t \in \mathcal T}$
or scalar summaries of it.
For instance, to quantify the amount of ambiguity about a particular
property $\varphi(X_t)$ of the latent state, where
$\varphi : \mathcal X \to \mathbb R$ is measurable and
$\varphi(X_t) \in L^2$, we use the conditional variance
\begin{equation*}
	\mathcal{V}_t^\varphi := \Var(\varphi(X_t) \mid \mathcal G_t) .
\end{equation*}
Large values of $\mathcal{V}_t^\varphi$ indicate high
uncertainty about $\varphi(X_t)$ at time $t$, whereas small values indicate little
uncertainty.
Both $\pi_t$ and $\mathcal{V}_t^\varphi$ generalize the notions
of posterior distribution and posterior variance that we used to analyze
ambiguity in \cref{sec:setup}.
In particular, for the statically randomized simulator of \cref{eq:slm},
$\pi_t$ reduces to the posterior distribution of the latent parameter given
the observations up to time $t$.

\begin{keypropertybox}
The measure-valued filter process $\pi = (\pi_t)_{t \in \mathcal T}$
describes the ambiguity in the simulator $(X, Y)$
under the filtration $\mathbb{G}$.
We quantify how much uncertainty is present in the simulation
about a property~$\varphi(X_t)$
via the conditional variance process $\mathcal{V}^\varphi = (\mathcal{V}_t^\varphi)_{t \in \mathcal{T}}$.
\end{keypropertybox}

\subsection{Stationary ambiguity}

In \cref{sec:setup}, we argued that vanishing ambiguity,
as induced by static randomization,
is an unrealistic and undesirable model of ambiguity
for many real-world control problems with an exogenous noise process $Y$.
In the language of the present section, vanishing ambiguity means that the filter becomes
systematically more concentrated as the episode progresses.
We argue that more generally, systematic variations of the filter process
with the clock of the control problem are usually undesirable.
For example, financial markets are non-episodic, so ambiguity should be driven by
market conditions rather than by the age of the control problem.
We therefore propose that the simulator should induce a stationary filter process.
We formalize this requirement on the bi-infinite time axis
$\mathcal T=\mathbb Z$
and discuss simulators started at a finite time in
\cref{subsec:approx_stationarity}.

\begin{keypropertybox}
The simulator should
induce \textbf{stationary ambiguity}, meaning that its filter process~$\pi$
should be strictly stationary:
\begin{equation*}
	(\pi_t,\pi_{t+1},\dots,\pi_{t+k})
	\stackrel{d}{=}
	(\pi_{t+s},\pi_{t+s+1},\dots,\pi_{t+s+k}) ,
\end{equation*}
for all $t,s \in \mathbb Z$ and $k \ge 0$.
\end{keypropertybox}

Stationarity of $\pi$ implies that
for any measurable $\varphi:\mathcal X\to\mathbb R$ with $\varphi(X_0)\in L^2$,
the conditional variance process~$(\mathcal V_t^\varphi)_{t\in\mathbb Z}$
is also stationary.
The assumption therefore still allows for the amount of ambiguity
to vary over time on any given path of $(X, Y)$.
For instance, in financial market simulators, one can
model that during turbulent market phases,
ambiguity is higher than during calmer market phases.
However, the marginal laws of the filter and ambiguity levels
cannot vary with the clock of the control problem,
that is, with the position $t$ within the control horizon.

When $\mathcal{T} = \mathbb{Z}$,
a standard sufficient condition for filter stationarity
is that the joint process $(X, Y)$ is stationary.
We discuss a more general version of this statement
that applies to information structures other than the filtration $\mathbb{G}$,
in \cref{app:sec:proofs}.
\begin{proposition}[Sufficient condition for filter stationarity]
\label{prop:amb_stationary}
	Consider the setup from \cref{subsec:setup} with
	$\mathcal T=\mathbb Z$.
	Suppose that the joint process $(X_t,Y_t)_{t \in \mathbb{Z}}$
	is strictly stationary.
	Then the filter process
	$(\pi_t)_{t\in\mathbb Z}$ has a strictly stationary version.
\end{proposition}
When $(X, Y)$ is stationary, we write $\mu$ for the common law of $(X_t, Y_t)$.
Intuitively,
stationarity of $(X,Y)$ ensures that, in distribution, the problem of
estimating $X_t$ from the observation process is the same at every time $t$.
Because $\mathcal G_t$ contains an infinite past at every
time $t$, the amount of information available for this estimation problem is
also the same at every time.
As a result, the process $\pi$ is stationary.

\paragraph{Optimal investment revisited.}
In the optimal investment example of \cref{subsec:opt_inv_example},
the statically randomized simulator can be written as a simulator $(X, Y)$
with a latent process that is constant over time:
\begin{equation*}
		X_1 \sim \nu, \quad
		X_t \mid X_{t-1} \sim \delta_{X_{t-1}} \quad \text{for} \quad t \ge 2.
\end{equation*}
This simulator is initialized at time $t = 1$ and not in the infinite past.
The sigma-fields $\sigma(Y_1,\dots,Y_t)$
are generated by an increasing number of observations over time
and thus increasingly informative about the static latent state.

If the same simulator is instead placed on the bi-infinite time axis $\mathbb Z$,
then \cref{prop:amb_stationary} applies
and the induced filter process is stationary.
But since in this example,
$p$ satisfies the identifiability assumption of \cref{prop:slm_vanishing_amb},
the infinite observation history reveals the latent state, so that
\begin{equation*}
	\pi_t = \delta_{X_1} \quad
	\text{and} \quad
	\mathcal{V}_t^\varphi = 0
	\qquad \text{almost surely for all } t \in \mathbb Z .
\end{equation*}
Ambiguity is stationary here only because it is absent at all times.
Under static randomization on the bi-infinite time axis,
all ambiguity has already vanished by the time the investment problem starts.
Thus, stationarity of $\pi$ should not be understood as sufficient
for a good ambiguity representation.

The dynamically randomized simulator of \cref{subsec:opt_inv_example}
is the linear Gaussian SSM,
defined on the bi-infinite time axis $\mathbb Z$ in its stationary regime.
Hence, $(X,Y)$ is stationary and \cref{prop:amb_stationary} applies:
the filter process $\pi$ and
$\Var(X_t \mid \mathcal G_t)$ are stationary.
In fact,
$\Var(X_t \mid \mathcal G_t)$ is constant over time
(left panel of \cref{fig:opt_inv}),
which is specific to the linear Gaussian SSM
and does not follow from filter stationarity in general.

\paragraph{Stronger ambiguity conditions.}
Filter stationarity should not be understood as a sufficient condition
for representing ambiguity realistically.
As seen above, stationarity still allows ambiguity to be absent at all times.
Moreover, even when ambiguity is present at all times, stochastic, and stationary,
the simulator may still not accurately represent the ambiguity faced in the real system.
For example, a simulator may be such that
\begin{equation*}
	\mathcal V_t^\varphi
	=
	v_{\mathrm{low}}\mathbf 1_{\{B=0\}}
	+
	v_{\mathrm{high}}\mathbf 1_{\{B=1\}},
	\qquad t\in\mathbb Z ,
\end{equation*}
for some Bernoulli random variable $B$
and $0 < v_{\mathrm{low}} < v_{\mathrm{high}}$.
In this simulator, some paths have low ambiguity forever,
while others have high ambiguity forever,
which is unrealistic in any domain where levels of uncertainty can change over time.
To rule out such ambiguity patterns,
one could impose the additional condition that
the filter process $\pi$ should be ergodic.
Similarly, other control problems and domains may motivate
further restrictions on~$\pi$.
Here, we focus on stationarity because practical randomization and simulation
schemes commonly violate it.

\subsection{Simulator randomization with stationary ambiguity}\label{subsec:stationary_randomization}
So far, we have treated the latent process $X$ as given.
We now return to the question of how such a process should be constructed
to randomize a parametric base simulator.
Suppose we start from a base simulator for the process $Y$ with fixed parameter~$x$.
How should we randomize $x$ so that the resulting simulator induces
stationary ambiguity?

\paragraph{Case 1: i.i.d.\ base simulator.}
First consider a base simulator in which, for each fixed parameter $x$
taking values in a parameter space $\mathcal{X}$,
the observations $Y_t \mid x$ are i.i.d.
That is,
\begin{equation*}
	Y_t = F(x,W_t),
	\qquad t\in\mathbb Z,
\end{equation*}
for an i.i.d.\ noise process $(W_t)_{t\in\mathbb Z}$ on a space
$\mathcal W$ and a measurable map
$F:\mathcal X\times\mathcal W\to\mathcal Y$.
To randomize the simulator, we replace the fixed parameter $x$ by the
$\mathcal{X}$-valued process $X$ and set
\begin{equation*}
	Y_t = F(X_t,W_t),
	\qquad t\in\mathbb Z.
\end{equation*}
If $X$ is independent of $W$ and a Markov process,
then the joint process $(X,Y)$ is Markov.
If $X$ is stationary and independent of $W$, then $(X,Y)$ is stationary,
since $Y$ is obtained by applying the same map $F$ at each
time to the stationary process $(X,W)$.
By \cref{prop:amb_stationary},
the filter $\mathbb P(X_t\in\cdot\mid\mathcal G_t)$ is therefore stationary.
Thus,
for i.i.d.\ base simulators, replacing the fixed parameter $x$ by a
stationary Markov process $X$ induces stationary ambiguity.

\paragraph{Case 2: Markov base simulator.}
Now consider a base simulator in which, for each fixed parameter
$x\in\mathcal X$, the observation process $Y$ is Markov.
We write its update as $Y_t=F(x,Y_{t-1},W_t)$, where
$(W_t)_{t\in\mathbb Z}$ is an i.i.d.\ noise process on a space
$\mathcal W$ and
$F:\mathcal X\times\mathcal Y\times\mathcal W\to\mathcal Y$
is measurable.
To randomize, we again replace the fixed parameter $x$ by the
$\mathcal X$-valued process $X$ and set
\begin{equation}\label{eq:randomized_markov_sim}
	Y_t = F(X_t,Y_{t-1},W_t),
	\qquad t\in\mathbb Z.
\end{equation}
Again, if $X$ is independent of $W$ and Markov, the joint process $(X, Y)$ is Markov.
However, stationarity of $X$ no longer directly implies stationarity of $(X,Y)$.
Even if $Y_t=F(x,Y_{t-1},W_t)$ has a stationary regime for every
fixed $x$, replacing $x$ by a stationary process $X$ does
not guarantee that $(X, Y)$ is stationary.
We therefore need stronger conditions on the update map $F$.
The following proposition,
due to \citet{stenflo_2001_marko},
gives a general sufficient condition.

\begin{proposition}\label{prop:stationary_randomization}
Assume the following.

\begin{itemize}
\item[(i)] The process $X$ is stationary and independent of $W$.

\item[(ii)]
There exist a complete separable metric $d$ on $\mathcal Y$, a point $y_\ast\in\mathcal Y$, and a constant $c\in(0,1)$ such that
\begin{equation*}
\mathbb E\big[d(F(x,y,W_0),F(x,\tilde y,W_0))\big]
\le c\,d(y,\tilde y)
\end{equation*}
for all $x\in\mathcal X$ and all $y,\tilde y\in\mathcal Y$, and
\begin{equation*}
\sup_{x\in\mathcal X}\mathbb E\big[d(y_\ast,F(x,y_\ast,W_0))\big] < \infty.
\end{equation*}
\end{itemize}
Then there exists a stationary solution $Y = (Y_t)_{t\in\mathbb Z}$ to the recursion
of \cref{eq:randomized_markov_sim}
and the joint process $(X, Y)$ is stationary.
If, in addition, $X$ is ergodic, then $(X, Y)$ is ergodic.
\end{proposition}
These conditions are satisfied by many models of practical interest.
In \cref{app:subsec:randomized_heston_proofs},
we verify them for the discretized Heston model \citep{heston_1993_a}
used in our experiments in \cref{subsec:synth_hedging_problems}.
We discuss how to verify the same property
for more general simulator formulations,
beyond the form of \cref{eq:randomized_markov_sim},
in \cref{sec:practical_recommendations}.

\paragraph{Refresh latent model.}
When condition~(ii) is satisfied,
replacing $x$ with any stationary Markov process $X$ independent of $W$
results in a stationary Markov simulator $(X, Y)$.
We call such a randomization scheme a \emph{stationary randomization}.
In our experiments,
we consider the following simple process for~$X$:
\begin{equation}\label{eq:refresh_markov_chain}
	X_t \mid X_{t-1}
	\sim
	(1 - \alpha) \delta_{X_{t-1}} + \alpha \nu,
	\qquad t \in \mathbb{Z},
\end{equation}
where $\alpha \in [0,1]$  is the probability with which the latent parameter
may refresh at each point in time.
This transition kernel has invariant distribution $\nu$,
so initializing $X$ from $\nu$ gives a stationary chain.
We call the resulting randomized simulator $(X, Y)$ the
\emph{refresh latent model} (RLM).
The static latent model corresponds to the special case $\alpha = 0$.

The RLM is practically appealing because it applies
to both discrete and continuous parameter spaces $\mathcal X$
and introduces only a single additional parameter $\alpha$
relative to the SLM.
Like the SLM,
the RLM introduces ambiguity without distorting the dynamics of the base simulator
too much.
Each path segment between two refresh times
is generated exactly by the base simulator under a fixed parameter value.
This is desirable because the base simulator is usually a reasonable model
of the real dynamics,
whereas large distortions would make the randomized simulator unrealistic.

The key difference between the RLM with $\alpha > 0$ and the SLM ($\alpha = 0$) is that,
since the latent process changes occasionally,
observations from the distant past are unlikely to be informative
about the current value of $X_t$.
Therefore, the filter does not concentrate over time
and ambiguity about the latent process generally remains present.

\subsection{Policy optimization under stationary ambiguity}\label{subsec:approx_stationarity}
\Cref{prop:amb_stationary} guarantees filter stationarity
when $(X,Y)$ is stationary
and defined on the bi-infinite time axis $\mathcal{T} = \mathbb{Z}$.
This suggests that policy optimization under stationary ambiguity
can be achieved as follows:
choose stationary Markov dynamics for $(X,Y)$,
simulate paths with an infinite past,
and let the policy observe $Y_{-\infty:t}$ before choosing $U_{t+1}$.
If the policy observed only $Y_{1:t}$, then the ambiguity faced by the policy
would be described by $\mathbb P(X_t\in\cdot\mid Y_{1:t})$,
which is generally not stationary.
In practice, we approximate this idealized scheme by a finite warm-up:
we start the simulation $H$ steps before the control problem begins at $t=1$
and let the policy observe $Y_{1-H:t}$ before choosing $U_{t+1}$.
We use the following finite-warm-up scheme to approximate stationary ambiguity
(details in \cref{subsec:place_in_stationary}):
\begin{keypropertybox}
\begin{enumerate}[leftmargin=2em]
	\item Pick stationary Markov dynamics for $(X,Y)$ with stationary distribution $\mu$.
	\item Fix a warm-up length $H\in\mathbb N_0$
			and initialize the joint process at time $t = 1 - H$
			in the stationary distribution: $(X_{1 - H}, Y_{1 - H}) \sim \mu$.
			Forward simulate to obtain the trajectory $(X_t, Y_t)_{t = 1 - H}^T$.
	\item Let the policy observe the full available history $Y_{1-H:t}$
			when choosing $U_{t + 1}$.
\end{enumerate}
\end{keypropertybox}

To see why ambiguity in the finite-warm-up simulator behaves similarly to
ambiguity in the infinite-past simulator,
note that for every fixed $t\ge 0$, the
sigma-fields $\sigma(Y_{1-H:t})$ increase to $\sigma(Y_{-\infty:t})$ as
$H\to\infty$.
For any measurable $\varphi$ with $\varphi(X_t)\in L^2$,
the martingale convergence theorem gives\footnote{
The sequence of sigma-fields $\{\sigma(Y_{1-H:t})\}_{H = 0, 1, \dots}$
adds observations backwards in calendar time, but is increasing in $H$.
Applying martingale convergence to $\varphi(X_t)$ in $L^2$
and to $\varphi(X_t)^2$ in $L^1$ gives the claim.}
\begin{equation*}
\mathbb E\bigl[
	\bigl|
	\Var(\varphi(X_t)\mid Y_{-\infty:t})
	- \Var(\varphi(X_t)\mid Y_{1-H:t})
	\bigr|
\bigr]
\xrightarrow[H\to\infty]{} 0 .
\end{equation*}
For each fixed $t$, the amount of ambiguity about the property $\varphi(X_t)$
in the finite-warm-up simulator converges to that in the
infinite-past simulator as $H\to\infty$.
However,
for a chosen warm-up length $H < \infty$,
this asymptotic statement does not quantify the difference
in ambiguity between the finite-warm-up and infinite-past simulators.

Useful finite-warm-up guarantees require stronger assumptions on~$(X,Y)$.
One important class with such guarantees is the class of hidden Markov models,
in which $X$ is a Markov process and,
given the current latent state $X_t$,
the current observation $Y_t$ is independent of the past
$(X_{<t},Y_{<t})$.
For many hidden Markov models,
the effect of the remote past on the filter decays as more recent
observations are accumulated.
This property is called \emph{stability} or \emph{forgetting}
of the filter
\citep{douc_2009_forge,vanhandel_2009_the,vanhandel_2009_unifo}.
Concretely, one often obtains convergence of the form
\begin{equation*}
	\bigl\|
			\mathbb P(X_t\in\cdot\mid Y_{-\infty:t})
			-
			\mathbb P(X_t\in\cdot\mid Y_{1-H:t})
	\bigr\|_{\mathrm{TV}}
	\xrightarrow[H\to\infty]{} 0
	\qquad \text{a.s.}
\end{equation*}
Under suitable additional conditions on the hidden Markov model,
this convergence admits quantitative rates,
which are often exponential \citep{legland_2000_expon,douc_2009_forge}.
Such exponential forgetting results are also available beyond hidden Markov models,
for example, for certain pairwise Markov models \citep{lember_2021_expon}.

\section{Effects of ambiguity dynamics on hedging policies}\label{sec:synth_experiments}

In this section, we analyze the effects of ambiguity
on policies for hedging problems.
As in \cref{subsec:opt_inv_example},
we find that policies trained under static randomization
are robust initially
but specialize to the realized latent regime,
making them highly sensitive to regime changes.
Policies trained under stationary ambiguity
maintain robustness over time.

\subsection{Hedging problems with parameter ambiguity}\label{subsec:hedging_problems}

Throughout this section, we consider hedging problems over the time period $\{1,\dots,T\}$,
where one time step corresponds to
$\Delta t = 1 / 250$ units of calendar time and $T=64$.
Let $S$ denote a $d$-dimensional price process,
where $d$ is the number of tradable stocks.
The process $S$ is defined on the time axis $\{-H,\dots,T\}$
with $S_{-H}= \mathbf{1}$,
where $H \geq 0$ is the length of a pre-trading period
during which the policy observes the price process but cannot trade yet.
Denote the normalized price process by
\begin{equation*}
	\tilde S_t := \frac{S_t}{S_0},
	\qquad t \in \{0,\dots,T\},
\end{equation*}
where the division is componentwise.
We consider a terminal payoff $\psi(\tilde{S}_{1:T})$
that depends on the path of the normalized price process $\tilde S$, where
$\psi : \mathbb R^{T \times d} \to \mathbb R$ is a payoff function.
We assume that the policy can only trade in $\tilde S$
and not in any other assets or derivatives,
so that $U_t \in \mathbb{R}^d$ represents the position held in $\tilde{S}$
over the time interval~$[t - 1, t]$.
We parameterize controls with a neural network policy such that
\begin{equation*}
  U_t = U_t^\theta = f_{\theta}(Y_{-H+1:t-1}),
  \qquad t=1,\dots,T ,
\end{equation*}
where, as before,
$Y$ denotes the progressively observed information process
and $f_\theta$ is a causal sequence model with parameters $\theta$.
We implement $f_\theta$ with an LSTM-based architecture
\citep{hochreiter_1997_long}.
The hedging problem is therefore
\begin{equation}\label{eq:hedging_problem}
	\min_\theta
	\quad
	\mathcal{R}(L(\theta)), \qquad
	L(\theta) = \psi(\tilde S_{1:T})
		- \sum_{t=1}^T {U^\theta_t}^\top (\tilde S_t - \tilde S_{t-1}) .
\end{equation}
We choose $\mathcal{R}$ to be the exponential spectral risk measure
\citep{acerbi_2002_spect,dowd_2011_expon}
\begin{equation*}
	\mathcal R_{\mathrm{srm}}(L)
	:=
	\int_0^1 q_L(p)\,\phi_\gamma(p)\,dp,
	\qquad
	\phi_\gamma(p)
	=
	\frac{\gamma e^{\gamma(p-1)}}{1-e^{-\gamma}},
\end{equation*}
where $q_L$ is the quantile function of $L$.
We set the risk aversion to ${\gamma = 4}$.
In \cref{subsec:add_synth_results}, we provide additional results
for the choice $\mathcal{R}(L) = \mathrm{Var}(L)$.
Following \citet{buehler_2019_deep},
we optimize $\theta$ using Monte Carlo rollouts,
backpropagation through time, and Adam \citep{kingma_2015_adam}.

\begin{table}[t]
	\centering
	\small
	\caption{Overview of the three hedging problems
	analyzed in \cref{sec:synth_experiments}.
	In each problem, we consider randomizing a base simulator
	with observed process $Y$ and uncertain parameter $x$.
	Both randomization schemes (SLM and RLM)
	use the same randomization distribution $\nu$;
	details in \cref{app:subsec:synth_details}.
	}
	\label{tab:hedging_problems}
	\begin{tabular}{@{}lllll@{}}
			\toprule
			Hedging problem & Parameter $x$ & Randomization distribution $\nu$ & Observation $Y_t$ & Payoff \\
			\midrule
			\textsc{BS-Vol}
			&
			$\sigma$
			&
			$\sigma^2 \sim \mathrm{InvGamma}$
			&
			$\log(S_t/S_{t-1})$
			&
			straddle
			\\
			\textsc{Heston-Corr}
			&
			$\rho$
			&
			$\rho = 2\eta - 1,\ \eta \sim \mathrm{Beta}$
			&
			$(\log(S_t/S_{t-1}), V_t)$
			&
			straddle
			\\
			\textsc{BS-Cov}
			&
			$(\sigma_1,\sigma_2,\varrho)$
			&
			\begin{tabular}[t]{@{}l@{}}
				$\sigma_i^2 \sim \mathrm{InvGamma},\ i=1,2$ \\
				$\varrho = 2\eta - 1,\ \eta \sim \mathrm{Beta}$
			\end{tabular}
			&
			$(\log(S_t^{(1)}/S_{t-1}^{(1)}), \log(S_t^{(2)}/S_{t-1}^{(2)}))$
			&
			worst-of call
			\\
			\bottomrule
	\end{tabular}
\end{table}
We consider three hedging problems,
\textsc{BS-Vol}, \textsc{Heston-Corr}, and \textsc{BS-Cov},
which we define below and summarize in \cref{tab:hedging_problems}.
In each problem, we start from a parametric base simulator with an uncertain
parameter $x$.
We then construct two randomized simulators
by replacing $x$ with a static or refresh latent process
(SLM or RLM).
We use the same randomization distribution $\nu$ for both SLM and RLM.

\begin{problem}[\textsc{BS-Vol}: volatility uncertainty]
In the first hedging problem, the base simulator is a
univariate geometric Brownian motion with zero drift.
The uncertain parameter is the volatility $\sigma>0$,
so that in the notation above $x:=\sigma$.
The policy observes the log-return process
\begin{equation*}
	Y_t := \log(S_t/S_{t-1}),
	\qquad
	Y_t \mid x
	\stackrel{\mathrm{iid}}{\sim}
	N\bigl(-\frac{1}{2}x^2\Delta t,\, x^2\Delta t\bigr).
\end{equation*}
Note that since the initial value of $S$ is deterministic
($S_{-H} = \mathbf{1}$),
observing $Y$ is equivalent to observing $S$.
We choose the randomization distribution $\nu$ so that,
if $x\sim\nu$, then $x^2$ has an inverse-gamma distribution
(see the top panel of \cref{fig:bs1d_prior_density_and_p1_semi_dev}).
The payoff is the at-the-money straddle
$\psi(\tilde S_{1:T}) = |\tilde S_T - 1|$.
\end{problem}

\begin{problem}[\textsc{Heston-Corr}: spot--volatility correlation uncertainty]
For the second hedging problem, we again consider a single stock,
but now with dynamics given by the stochastic volatility model
of \citet{heston_1993_a}:
\begin{align*}
	d S_t & = S_t \sqrt{V_t}\,dW_t ,                 \\
	d V_t & = \kappa (\bar v - V_t)\,dt + \xi \sqrt{V_t}\,dB_t .
\end{align*}
We treat the parameters $\kappa$, $\bar v$, and $\xi$ as known,
but consider the correlation parameter
$\rho \in (-1, 1)$ between the Brownian motions $W$ and $B$ as uncertain,
so that $x := \rho$ here.
We assume that the policy observes both the spot process $S$
and the variance process $V$ and accordingly set
\begin{equation*}
	Y_t := (\log(S_t/S_{t-1}), V_t) .
\end{equation*}
The randomization distribution $\nu$ is the law of $2\eta - 1$,
where $\eta \sim \mathrm{Beta}$.
The payoff is again the at-the-money straddle.
\end{problem}

\begin{problem}[\textsc{BS-Cov}: covariance uncertainty]
In the third hedging problem, the base simulator is a two-dimensional
geometric Brownian motion with zero drift.
Here, we consider uncertainty over the entire covariance matrix $\Sigma$,
which is determined by the two marginal volatilities
$\sigma_1,\sigma_2>0$ and the correlation $\varrho\in(-1,1)$,
so that $x := (\sigma_1,\sigma_2,\varrho)$ here.
The policy observes the two-dimensional log-return process, so that
\begin{equation*}
	Y_t
	:=
	\bigl(
			\log(S_t^{(1)} / S_{t-1}^{(1)}),
			\log(S_t^{(2)} / S_{t-1}^{(2)})
	\bigr),
\end{equation*}
which satisfies
\begin{equation*}
	Y_t \mid x
	\stackrel{\text{iid}}{\sim}
	N\biggl(
			-\frac{1}{2}
			\begin{pmatrix}
					\sigma_1^2 \\
					\sigma_2^2
			\end{pmatrix}
			\Delta t,\,
			\Sigma(x)\Delta t
	\biggr),
	\qquad
	\Sigma(x)
	=
	\begin{pmatrix}
			\sigma_1^2 & \varrho \sigma_1 \sigma_2 \\
			\varrho \sigma_1 \sigma_2 & \sigma_2^2
	\end{pmatrix}.
\end{equation*}
We choose the randomization distribution $\nu$ so that
$\sigma_1^2$ and $\sigma_2^2$ are sampled independently
from inverse-gamma distributions and $\varrho$ has the law of $2\eta-1$,
where $\eta$ is beta-distributed.
The payoff is the at-the-money worst-of call
\begin{equation*}
	\psi(\tilde S_{1:T})
	=
	\bigl(\min\{\tilde S_T^{(1)}, \tilde S_T^{(2)}\} - 1\bigr)^+ .
\end{equation*}
\end{problem}

\paragraph{Randomized simulators.}
Each of the three base simulators is amenable to stationary randomization
in the sense of \cref{subsec:stationary_randomization}.
For the two geometric Brownian motion simulators this is immediate,
since the observed return process is i.i.d.\ under each fixed parameter value.
For the Heston simulator, the observed process $Y$ is not i.i.d.,
but we show in \cref{app:subsec:randomized_heston_proofs}
that the transition map for $Y$ induced by discrete observations of the
Heston model satisfies the conditions of \cref{prop:stationary_randomization}.
This guarantees that the Heston simulator is amenable to stationary
randomization, even when jointly randomizing the four parameters
$(\kappa,\bar v,\xi,\rho)$.
Here, we only treat $\rho$ as uncertain.

For each hedging problem, we apply the SLM and RLM randomization schemes
to the base simulator using the specified randomization distribution $\nu$.
For the RLM, we use the refresh probability $\alpha=0.01$ throughout,
so that $(1 - \alpha)^T \approx 0.53$ for $T = 64$,
meaning that on roughly half of the paths under the RLM
the latent state does not jump during the hedging period.
We train one policy under the SLM and one policy under the RLM.
Since the two randomized simulators use the same distribution $\nu$,
their main difference is the ambiguity dynamics induced during training:
ambiguity vanishes under the SLM and is stationary under the RLM.
Comparing the two trained policies therefore allows us to study
the effect of simulation-implied ambiguity dynamics on learned hedging strategies.

\subsection{Neural hedging policies learn to filter}\label{subsec:neural_policies_learn_filter}
In the investment problem of \cref{subsec:opt_inv_example},
we observed from the closed-form solutions
that the optimal controls depend on the filter implied by the simulator.
For the neural hedging policies considered here,
it is no longer transparent whether the learned decision rule
uses an estimate of the latent regime.
In this subsection, we focus on the \textsc{BS-Vol} problem
and analyze whether the neural hedging policies
implicitly learn to filter during training.

For this purpose, we introduce a policy-dependent notion of ambiguity.
The LSTM-based policies have the form
$U_{t + 1} = g(t + 1,h_t)$, where $h_t$ denotes the hidden state of the LSTM
after processing observations up to time $t$.
The control $U_{t + 1}$ can therefore depend on beliefs about $X_t$
only insofar as they are captured in $h_t$.
We therefore say that a policy's ambiguity about a property
$\varphi(X_t)$ is quantified by
the conditional variance $\Var(\varphi(X_t) \mid h_t)$.\footnote{
This should be interpreted as a lower bound on the policy's ``effective'' ambiguity,
since it corresponds to the best possible prediction based on $h_t$, whereas the
policy itself may only exploit a more restricted readout of its internal state.}
The policy's average ambiguity at time $t$ is therefore given by
\begin{equation*}
	\mathcal{E}_t
	\; := \;
	\mathbb E\bigl[
		\Var\bigl(\varphi(X_t)\mid h_t\bigr)
	\bigr]
	\; = \;
	\inf_{m_t \;\text{measurable}}
	\mathbb E\bigl[
		\bigl(\varphi(X_t)-m_t(h_t)\bigr)^2
	\bigr] .
\end{equation*}
We approximate $\mathcal E_t$ by replacing the infimum over all measurable
readouts $m_t$ with a minimization over a class of neural networks
with a single hidden layer.
That is, for each time $t$, we fit a neural network $m_{w_t}$ with weights $w_t$,
called a \emph{probe},
to predict $\varphi(X_t)$ from $h_t$ and report its squared prediction error.

\begin{figure}[t]
	\centering
	\includegraphics[width=\textwidth]{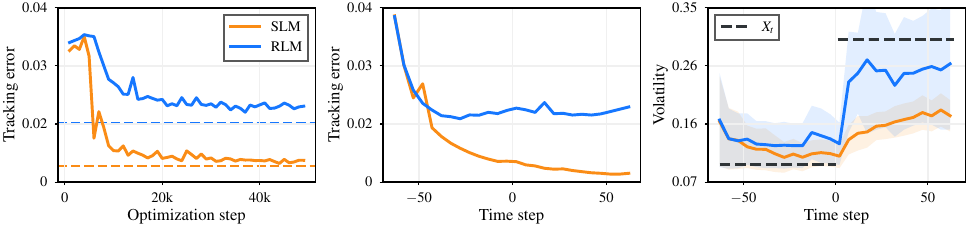}
	\caption{Neural hedging policies learn to filter.
	We probe the recurrent state $h_t$ of policies trained under the SLM and RLM
	by predicting $\log X_t$ from $h_t$.
	\emph{Left:} held-out probe error for
	$\mathcal E_0=\mathbb E[\Var(\log X_0\mid h_0)]$
	during hedging training;
	dashed lines show the minimal practically achievable tracking errors within
	each simulator.
	\emph{Middle:} after training, $\mathcal E_t$ decreases over time for the SLM policy
	but levels off for the RLM policy,
	matching the vanishing/stationary ambiguity dynamics of the simulators.
	\emph{Right:} probe-implied estimate of $X_t$ and uncertainty band
	on a path where $X_t = 0.1$ for $t \leq 0$ and $X_t = 0.3$ for $t \geq 1$.
	The RLM policy updates quickly and becomes more uncertain after the switch,
	whereas the SLM policy has a less accurate but more confident probe-implied
	belief.}
	\label{fig:model_uncertainty}
\end{figure}

The left panel of \cref{fig:model_uncertainty} shows the evolution of the estimated
average ambiguity $\mathcal E_0$ during training.
We use a warm-up period of length $H = 64$,
so that the recurrent state $h_0$ can depend on $64$ observations
of the process $Y$,
and set $\varphi(x)=\log x$, meaning that the target variable $\varphi(X_t)$
is the log-volatility.
The dashed lines indicate the lowest attainable tracking error in the
corresponding simulator.\footnote{
The dashed error lines are obtained from separate networks with the same
LSTM-based architecture as the policy, trained directly to predict $\log X_t$ from the
observation history, without the hedging objective.
}
For both policies, the estimated error decreases during training,
showing that
\textbf{the policies learn to encode information about the latent volatility
regime in their recurrent states}.
The observation that $h_t$ becomes increasingly informative about $\log X_t$
during training suggests that the learned decision rule uses
regime information when choosing $U_{t + 1}=g(t + 1, h_t)$.\footnote{
This observation does not imply that the action readout $g$
actually uses all regime information in $h_t$ \citep{hewitt_2019_desig}.}

The middle panel shows the map $t \mapsto \mathcal E_t$ at the end of training.
For the SLM policy, $\mathcal E_t$ decreases almost monotonically with~$t$.
For the RLM policy, the curve also declines initially, but levels off above zero.
Thus, the policy-dependent ambiguity level
$\mathbb E[\Var(\log X_t \mid h_t)]$
has the same qualitative dynamics as the
policy-independent level ${\mathbb E[\Var(\log X_t \mid \mathcal G_t)]}$,
which converges to zero under the SLM and to a positive limit under the RLM
(left panel of \cref{fig:fig1}).

For the right panel, we train the probe $m_{w_t}$ to return both
a mean and a standard deviation prediction for $\varphi(X_t)$.
We apply the probes $\{m_{w_t}\}_t$
to a single evaluation path $(X,Y)$
on which the latent volatility $X$ jumps at time $t = 1$.
We visualize the probe-implied estimate and uncertainty band for $X_t$,
which can be interpreted as an approximation of the
policy's internal belief distribution over the latent volatility.
Following the regime shift,
the RLM policy's probe-implied estimate moves quickly towards the
new volatility level and the uncertainty band widens.
This suggests that the RLM policy treats the observations after $t = 1$
as evidence for a possible regime change.
The SLM policy's estimate also moves towards the new volatility
level, but more slowly, while its uncertainty band stays comparatively tight.
After the regime shift, the SLM policy's belief is therefore less accurate
and more confident than that of the RLM policy.
Next, we analyze how these different filtering behaviors affect
the robustness of the two policies.

\subsection{Effects on initial and continual robustness}\label{subsec:synth_hedging_problems}

In this section, we evaluate the robustness properties of the SLM and RLM policies.
Our evaluation is based on two desiderata.
First, a robust policy should perform
well across a range of parameter values $x$ in the base simulator.
Second, the policy should maintain that robustness over time.
That is, even if the parameter drifts or shifts at some point,
the policy should still hedge effectively.
We refer to these two robustness notions
as \emph{initial} and \emph{continual} robustness
and study them on the three hedging problems
\textsc{BS-Vol}, \textsc{Heston-Corr}, and \textsc{BS-Cov}.

\paragraph{Initial robustness.}
For this evaluation,
we define each hedging problem with a pre-trading period of length $H = 0$,
so that the policy does not observe the process $Y$ prior to the start
of the trading period.
For a given simulator,
we train a policy on this $H = 0$ hedging problem and evaluate it
in the base simulator across different values of the parameter~$x$.
For a given value $x^\mathrm{eval} \in \mathcal{X}$,
we set $X_t = x^\mathrm{eval}$ for $t = 1, \dots, T$,
sample paths of $Y$ given $X$, compute the hedging loss~$L$ on each path,
and record the resulting risk $\mathcal{R}_\mathrm{eval}(L)$,
where $\mathcal{R}_\mathrm{eval}$ is a risk or deviation measure.
Repeating this for different values of $x^\mathrm{eval}$ produces the curve
\begin{equation*}
	x^\mathrm{eval} \; \mapsto \; \mathcal{R}_\mathrm{eval}(L \mid X_{1:T} = x^\mathrm{eval}),
\end{equation*}
which lets us assess how robust the policy is initially
to the latent regime, before it has made any observations.
A policy with strong initial robustness should perform well
over a broad range of plausible parameter values.
This is typically reflected in a flatter curve,
although the curve also depends on
the intrinsic difficulty of the hedging problem
under different values of~$x^\mathrm{eval}$.

\paragraph{Continual robustness.}
To study how the policy's robustness evolves over time,
we now train each policy with a pre-trading period of length $H = T$,
during which the policy observes $Y$ but does not trade.
We then evaluate the trained policy in the base simulator under a regime shift
at time $t = 1$:
for given values $x^\mathrm{pre}, x^\mathrm{eval} \in \mathcal{X}$,
we set $X_t = x^\mathrm{pre}$ for $t \leq 0$
and $X_t = x^\mathrm{eval}$ for $t \geq 1$,
sample paths of $Y$ given $X$, compute the hedging loss~$L$ on each path,
and record the resulting risk $\mathcal{R}_\mathrm{eval}(L)$.
Repeating this for different values of $x^\mathrm{eval}$ produces, for each fixed $x^\mathrm{pre}$,
the curve
\begin{equation*}
	x^\mathrm{eval} \; \mapsto \; \mathcal{R}_\mathrm{eval}(L \mid X_{-T+1:0} = x^\mathrm{pre},\; X_{1:T} = x^\mathrm{eval}) ,
\end{equation*}
which measures how the policy performs across values of $x^\mathrm{eval}$
after adapting to observations generated under $x^\mathrm{pre}$.
Varying $x^\mathrm{pre}$ then produces a family of such curves,
which characterize whether the policy maintains its robustness over time.
A continually robust policy should exhibit three properties:
the curves should lie at a low level overall, indicating good hedging performance;
each individual curve should be relatively flat in $x^\mathrm{eval}$,
indicating robustness across trading-period regimes;
and the spread across curves for different values of $x^\mathrm{pre}$ should be small,
indicating robustness to the pre-trading regime.

Because the regime shift occurs at the start of the trading period,
the pre-trading regime $x^\mathrm{pre}$
does not affect the intrinsic difficulty of the hedging problem.
The evaluation therefore isolates how past observations
affect the policy's current robustness.
It can also be viewed as measuring the robustness of later controls
$U_{T+1:2T}$ in a longer hedging problem of length $2T$.
In fact, this is how we jointly implement the initial and continual
robustness evaluations
(\cref{app:subsec:synth_details}).

We evaluate continual robustness using multiple choices for $\mathcal{R}_\mathrm{eval}$,
including both deviation measures (variance and semi-deviation)
and risk measures (conditional value-at-risk and exponential spectral risk measure).
We provide detailed results for all choices of $\mathcal{R}_\mathrm{eval}$
in \cref{app:sec:synth}.
In the main text we focus on the semi-deviation measure
$\mathcal{R}_{\mathrm{sd}}$,
defined by
\begin{equation*}
	\mathcal{R}_{\mathrm{sd}}(L)
	:=
	\bigl(
		\mathbb{E}\bigl[
			\max\{L-\mathbb{E}[L],0\}^2
		\bigr]
	\bigr)^{1/2}.
\end{equation*}
The conditional notation
$\mathcal{R}_{\mathrm{sd}}(L \mid X_{1:T} = x^\mathrm{eval})$
is understood as replacing all expectations
in the definition of $\mathcal{R}_{\mathrm{sd}}$
by the corresponding conditional expectations.

\paragraph{Both static and dynamic randomization improve initial robustness.}

\begin{wrapfigure}[21]{R}{0.5\textwidth}
	\centering
	\includegraphics[width=\linewidth]{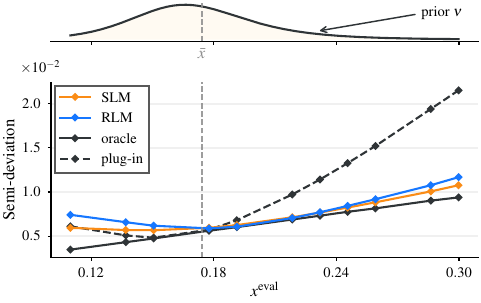}
	  \caption{Initial robustness from simulator randomization.
        \emph{Top:} prior $\nu$ over the volatility parameter.
        \emph{Bottom:} semi-deviation
		$\mathcal{R}_{\mathrm{sd}}$
		of learned and benchmark policies on GBM paths
		with fixed volatility $X_t \equiv x^\mathrm{eval}$.
		The two policies trained in randomized simulators (SLM and RLM)
        have flatter risk curves than the plug-in policy,
		meaning their performance is less sensitive to the realized volatility regime.}
		\label{fig:bs1d_prior_density_and_p1_semi_dev}
\end{wrapfigure}

\Cref{fig:bs1d_prior_density_and_p1_semi_dev} shows the initial robustness
evaluation for the \textsc{BS-Vol} hedging problem.
We compare the SLM and RLM policies with two analytic benchmark policies,
based on the Black--Scholes delta
$\Delta^{\psi}(s,\tau,\hat \sigma)$, where $s$ is the current spot, $\tau$ is
time to maturity, and $\hat \sigma$ is the volatility input.
The two benchmark policies differ in the volatility input $\hat \sigma$.
On evaluation paths with volatility $X_t\equiv x^\mathrm{eval}$,
the \emph{oracle} policy uses $\hat \sigma=x^\mathrm{eval}$,
whereas the \emph{plug-in} policy uses the prior mean
$\hat \sigma=\bar x:=\E_{x\sim\nu}[x]$.
The oracle policy therefore assumes knowledge of the true volatility regime
and provides an infeasible benchmark for the lowest attainable risk in each regime.
Under the oracle policy, the semi-deviation loss increases monotonically
as a function of the volatility level, which indicates that the
hedging problem becomes intrinsically more difficult with higher volatility.

The plug-in policy is the relevant feasible benchmark.
It performs well when the evaluated volatility regime is near
the plug-in value~$\bar x$,
matching closely the performance of the oracle policy.
For volatility regimes farther from~$\bar x$, and especially in
high-volatility regimes, the plug-in policy performs much worse
than the oracle policy.
Thus, in the \textsc{BS-Vol} hedging problem, optimizing in the base simulator
under a single parameter value produces a hedge whose performance depends
strongly on whether the realized volatility is close to $\bar x$.

Both the SLM and RLM policies achieve a much lower loss away from $\bar x$,
showing that
\textbf{hedging policies trained under simulator randomization
perform more robustly across volatility regimes}.
This robustness comes at a small but visible price.
In the low-volatility regimes $x^\mathrm{eval} \in (0.12, 0.17)$, which are close to $\bar x$,
the plug-in policy slightly outperforms both randomized policies.
The difference in performance between the SLM and the RLM policy is small,
with the SLM generally performing slightly better.
This is expected because the evaluation
paths have constant volatility and thus match the SLM dynamics more closely.

\paragraph{Stationary ambiguity preserves robustness over time.}

\begin{figure}[t]
	\centering
	\includegraphics[width=\textwidth]{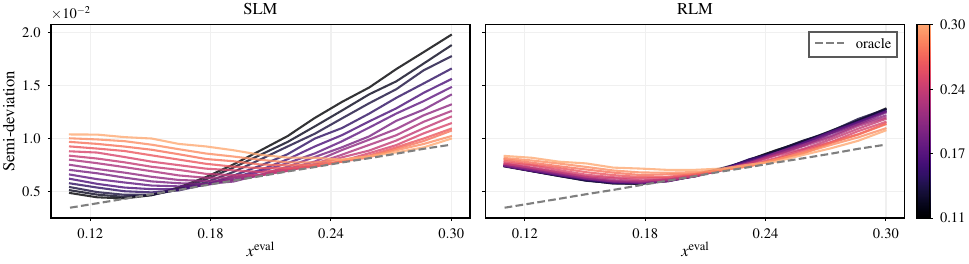}
	\caption{Continual robustness under volatility uncertainty (\textsc{BS-Vol}).
	We report $\mathcal{R}_{\mathrm{sd}}$ on paths
	in which the volatility regime shifts from $x^\mathrm{pre}$
	to $x^\mathrm{eval}$ at time $t = 1$.
	The left panel shows the performance of the SLM policy
	and the right panel that of the RLM policy.
	Colors indicate the pre-trading volatility
	$x^\mathrm{pre}$ and the horizontal axis indicates the evaluation volatility
	$x^\mathrm{eval}$.
	The dashed line shows the (infeasible) oracle benchmark.
	A larger spread across colored curves means stronger dependence on the
	volatility regime observed before trading.}
	\label{fig:bs1d_semi_dev_p1_p2}
	\vspace{-\baselineskip}
\end{figure}

Having established that both randomization schemes improve initial robustness
on \textsc{BS-Vol},
we now ask whether this robustness persists over time.
\Cref{fig:bs1d_semi_dev_p1_p2} shows the continual robustness curves
for \textsc{BS-Vol}.
The left panel corresponds to the SLM policy and shows a similar pattern
across all pre-trading regimes $x^\mathrm{pre}$.
The SLM policy performs close to optimal only when
the difference between $x^\mathrm{eval}$ and $x^\mathrm{pre}$ is small.
When $x^\mathrm{eval}$ is far from $x^\mathrm{pre}$,
the loss of the SLM policy is substantially larger than that of the oracle policy,
indicating that the SLM policy struggles to adapt to changes in the latent regime.
The SLM policy performs worst when volatility is low during the pre-trading period,
but high during the trading period.
Thus,
\textbf{the robustness of the SLM policy does not persist over time}.
After observing a pre-trading regime, the policy specializes to that regime
and becomes sensitive to regime changes.

Compared to the SLM, the RLM's continual robustness curves,
shown in the right panel,
are much less spread out across pre-trading regimes $x^\mathrm{pre}$.
For a fixed evaluation volatility $x^\mathrm{eval}$,
the loss is relatively similar across values of $x^\mathrm{pre}$.
The volatility regime observed before trading therefore has only a small
effect on performance,
meaning that
\textbf{the RLM policy keeps its robustness to the current latent regime over time}.
In particular, the RLM policy achieves a much lower loss
on paths where volatility jumps from low to high,
which is where both policies incur their largest losses.

\begin{wrapfigure}[24]{r}{0.37\textwidth}
	\centering
	\includegraphics[width=\linewidth]{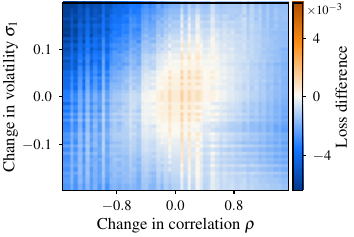}
	\caption{Continual robustness under covariance regime shifts (\textsc{BS-Cov}).
	We vary both
	$\sigma_1$ and $\varrho$ between the pre-trading and trading periods.
	Each cell groups scenarios by
	$\varrho^{\mathrm{eval}} - \varrho^{\mathrm{pre}}$
	and $\sigma_1^{\mathrm{eval}} - \sigma_1^{\mathrm{pre}}$,
	and shows the difference in semi-deviation loss between the policies
	(RLM minus SLM).
	Positive values indicate lower loss for the SLM policy;
	negative values indicate lower loss for the RLM policy.
	The RLM policy outperforms the SLM policy except
	for very small changes in the covariance regime.
	}
	\label{fig:bs2d_joint_grid}
\end{wrapfigure}

This robustness comes at the cost of some regime-specific performance.
In low-volatility regimes from the left tail of $\nu$,
the RLM policy performs notably worse than the oracle policy.
In high-volatility regimes from the right tail of $\nu$,
the relative gap between the RLM and oracle policies is smaller.
Compared with the SLM policy, the RLM policy therefore
mainly gives up performance on paths where volatility is low throughout.
This behavior is consistent with the tail-focused training objective
$\mathcal{R}_{\mathrm{srm}}$,
which favors avoiding large losses under regime shifts or high volatility
over obtaining oracle-level performance in easier low-volatility regimes.

\begin{figure}[t]
	\centering
	\includegraphics[width=\textwidth]{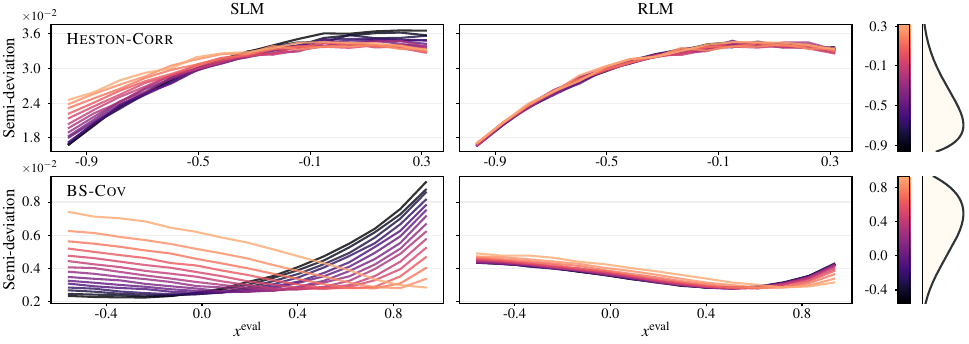}
	\caption{Continual robustness under correlation uncertainty.
	\emph{Top:} in the \textsc{Heston-Corr} problem, the spot-volatility
	correlation changes between the pre-trading and trading periods.
	\emph{Bottom:} in the \textsc{BS-Cov} problem, the asset correlation changes
	between the two periods while the marginal volatilities are fixed.
	In each row, colors indicate the pre-trading correlation and the horizontal
	axis indicates the trading-period correlation.
	The SLM curves spread out more than the RLM curves, showing that the SLM
	policies are more sensitive to the regime observed before trading.}
	\label{fig:heston_bs2d_semidev}
	\vspace{-0.5\baselineskip}
\end{figure}

\Cref{fig:heston_bs2d_semidev} repeats the continual-robustness evaluation
for the \textsc{Heston-Corr} and \textsc{BS-Cov} hedging problems.
For \textsc{Heston-Corr}, we observe that the intrinsic difficulty
of the hedging problem depends strongly on the correlation parameter $\rho$.
A large absolute correlation between spot and volatility allows
the policy to hedge volatility risk using the spot.
While the change in intrinsic difficulty is dominant,
the SLM curves still exhibit a visibly larger spread than the RLM curves,
indicating that the SLM hedge is again more sensitive to the previously
observed regime.

In the \textsc{BS-Cov} hedging problem, all three covariance parameters
$(\sigma_1,\sigma_2,\varrho)$ are randomized jointly during training.
For the plot in the bottom row of \cref{fig:heston_bs2d_semidev},
we fix the two marginal volatilities $\sigma_1, \sigma_2$ and vary
only the correlation parameter $\varrho$.
For the SLM policy, the continual robustness curves are highly spread out
whereas for the RLM policy the curves are close together.
Hence, the SLM policy is much less robust to changes
in the correlation between the two assets than the RLM policy.
As in the \textsc{BS-Vol} problem, this robustness comes with a small loss
of regime-specific performance.
When $\varrho$ is unusually low in both the pre-trading and trading periods
(see the low-$\varrho$ tail of the prior in \cref{fig:heston_bs2d_semidev}),
the SLM policy can outperform the RLM policy.

\Cref{fig:bs2d_joint_grid} gives a second visualization of
continual robustness in \textsc{BS-Cov}.
We now fix only the marginal volatility $\sigma_2$
and consider regime shifts from
$(\sigma_1^\mathrm{pre},\varrho^\mathrm{pre})$ 
to $(\sigma_1^\mathrm{eval},\varrho^\mathrm{eval})$.
For each regime shift scenario, we compute the loss of
both policies, group the results by $(\Delta \varrho,\Delta \sigma_1)$,
and plot the average loss difference
between the SLM and RLM policies as a heatmap.
The heatmap shows that when the covariance regime changes little,
the SLM has a small advantage over the RLM.
For larger regime changes, the RLM policy performs better,
and the gap generally widens with the size of the change.
We also observe that the RLM advantage is larger after decreases in
correlation than after increases of similar size.
This is natural since the worst-of call is harder to hedge under lower
correlation than under higher correlation.

\paragraph{In-simulator performance.}

\begin{wraptable}{R}{0.5\textwidth}
	\vspace{-\baselineskip}
	\centering
	\caption{In-simulator evaluation (\textsc{BS-Vol}).
	We compute $\mathcal{R}_\mathrm{srm}$ for the SLM-trained and RLM-trained policies
	on paths from the SLM simulator (SLM column) and from the RLM simulator
	(RLM column).
	The \emph{Gap} row reports SLM minus RLM, so
	\textcolor{myblue}{positive} values favor RLM
	and \textcolor{myred}{negative} values favor SLM.
	As a reference, the last two columns report points
	from the continual-robustness curves chosen such that
	the SLM--RLM gap is largest in either direction.
	\protect\footnotemark{}
	All entries are multiplied by $100$.}
	\label{tab:bs1d_semidev_p1_p2_worst}
	\small
\textsc{
\begin{tabular}{lcccc}
\toprule
\diagbox[width=5.8em,height=2.1em]
{\raisebox{-0.15em}{\hspace{-0.5em}Policy}}
{\raisebox{0.15em}{Eval.}} & SLM & RLM & \shortstack{shift\\max gap} & \shortstack{shift\\min gap} \\
\midrule
SLM & 10.31 & 10.19 & 16.42 & 5.42 \\
RLM & 10.33 & 10.15 & 14.94 & 5.93 \\
\cmidrule(lr){1-5}
\emph{Gap} & \textcolor{myred}{-0.02} & \textcolor{myblue}{0.04} & \textcolor{myblue}{1.48} & \textcolor{myred}{-0.51} \\
\bottomrule
\end{tabular}
}

\end{wraptable}

\footnotetext{For each regime-pair $(x^\mathrm{pre},x^\mathrm{eval})$
considered in the continual-robustness analysis,
we compute the SLM--RLM loss difference.
The ``max gap'' column reports the loss of each policy on the regime-pair
where this difference is largest,
and the ``min gap'' column uses the point where it is smallest.}

So far, we evaluated the policies on fixed regimes and controlled regime shifts.
We now ask how the same policies perform when evaluated directly under the
randomized simulators.
We take a trained policy and compute
$\mathcal{R}_{\mathrm{srm}}$
on paths of the SLM simulator,
meaning that the expectations in $\mathcal{R}_{\mathrm{srm}}$
are taken under SLM dynamics.
We report these evaluations in the SLM column of
\cref{tab:bs1d_semidev_p1_p2_worst}.
Analogously, the RLM column reports the same policies evaluated on paths
from the RLM simulator.
We again set $H = T$.

As expected, the SLM policy is the better policy under SLM dynamics,
and the RLM policy is the better policy under RLM dynamics,
but the gaps are small.
This may seem surprising given how differently the policies respond
to controlled regime shifts,
but the results are in fact consistent with our previous observations.
For instance, the continual robustness analysis showed that the SLM policy
fails mainly after large regime shifts.
Under the RLM simulator, such shifts have a low probability:
refreshes occur with probability $\alpha=0.01$,
and even after a refresh occurs,
the new parameter value can still be close to the previous one.
Thus, many RLM paths still look like fixed-regime paths,
where the SLM policy performs well.
This shows that aggregate simulator evaluations can miss poor robustness
properties of a policy,
and highlights the need for controlled stress tests such as the initial
and continual robustness evaluations above.
We report more detailed in-simulator results in \cref{subsec:add_synth_results}.

\paragraph{Sensitivity to refresh probability.}

\begin{wrapfigure}{R}{0.5\textwidth}
	\vspace{-1.0\baselineskip}
	\centering
	\includegraphics[width=\linewidth]{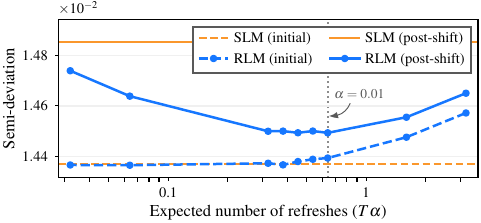}
	\vspace{-0.75\baselineskip}
	\caption{Sensitivity to the refresh probability in \textsc{BS-Vol}.
	We train RLM policies at different probabilities $\alpha$
	and report the semi-deviation~$\mathcal{R}_{\mathrm{sd}}$
	against the expected number of refreshes $T\alpha$.
	``Initial'' denotes evaluation without prior observations.
	``Post-shift'' denotes evaluation after a pre-trading period generated under an
	independently drawn volatility.}
	\label{fig:bs1d_refresh_prob_semidev}
	\vspace{-1.2\baselineskip}
\end{wrapfigure}

We study the sensitivity of the RLM policy to its additional parameter,
the refresh probability $\alpha$, in
\cref{fig:bs1d_refresh_prob_semidev}.
For the lines labeled ``initial'', we set $H = 0$
and evaluate policies on paths with constant volatility:
$X_t \equiv x^\mathrm{eval}$ for $t \geq 1$ and $x^\mathrm{eval}\sim\nu$.
For the lines labeled ``post-shift'', we set $H = T$
and evaluate policies on paths with a regime shift from
$x^\mathrm{pre}$ to $x^\mathrm{eval}$,
where $x^\mathrm{pre}$ and $x^\mathrm{eval}$ are
drawn independently from $\nu$.
The difference from the initial and continual robustness evaluations
is that here we report $\mathcal{R}_\mathrm{sd}$
averaged over parameter values drawn from~$\nu$.

The post-shift evaluation is harder for both policies
as they start trading with beliefs shaped by the wrong regime.
The SLM policy performs much worse on the post-shift evaluation than on the initial
evaluation due to its poor continual robustness.
For very small $\alpha$, the RLM policy behaves similarly,
because the RLM dynamics are then close to the SLM dynamics.
For large $\alpha$, the RLM policy performs relatively poorly in both
evaluations.
We attribute this to volatility refreshes occurring so often that simulated
paths no longer look locally like fixed-volatility geometric Brownian motion paths.
For many intermediate values of $\alpha$, including $\alpha=0.01$,
the RLM policy nearly matches the SLM policy on the initial evaluation
while achieving a much lower risk on the post-shift evaluation.
As long as $\alpha$ is not too large,
the RLM therefore induces regime-shift robustness at little cost
to the initial performance.

\paragraph{Why not i.i.d.\ randomization?}
A natural alternative to the RLM is to redraw the latent parameter
independently at each step.
Such a scheme is stationary and corresponds to the $\alpha = 1$ extreme of the RLM.
Since $X_t$ is redrawn independently at each step,
past observations carry no information about the current latent regime,
so ambiguity does not decay over time.
However, 
\cref{fig:bs1d_refresh_prob_semidev} showed
that increasing $\alpha$ towards this regime degrades performance
on fixed-parameter paths from the base simulator.
This is undesirable if we regard the base simulator
as a reasonable model of the real dynamics.
This issue could be mitigated by choosing a narrower distribution $\nu$,
but the resulting policy would then only be robust to a smaller range of parameter values.

\paragraph{Takeaway.}
Both static (SLM) and dynamic (RLM)
randomization improve the robustness of the hedging policies
with respect to the initial parameter regime.
Under static randomization, the robustness to the current latent regime
decreases as ambiguity vanishes.
Under randomization with stationary ambiguity,
this robustness is maintained over time.
We confirm in \cref{app:sec:synth} that this observation is not specific
to the training objective $\mathcal{R}_{\mathrm{srm}}$ or the payoff function $\psi$.
There, we repeat the \textsc{BS-Vol} experiment under a variance objective
and with a risk-reversal payoff.
In both cases, the qualitative picture remains the same.

\section{A case study on real market data}\label{sec:real_data}
The previous results showed how stationary ambiguity
improves continual robustness to latent simulator parameters.
In this section, we examine how this robustness affects hedging performance
in real financial markets, which are known to exhibit
parameter drift and regime changes \citep{lamoureux_1990_persi,ang_2012_regim}.
Because stationary ambiguity is a general principle for simulator design,
it cannot be validated directly on real data.
Still, we can study the effect of enforcing stationary ambiguity
on hedging performance
by training policies in simulators with and without stationary ambiguity
and backtesting them on historical market data.

\subsection{Setup}
We follow the general setup of \cref{sec:synth_experiments} and consider
hedging problems defined on a normalized spot process
$\tilde{S}$, where $\tilde{S}_0 = 1$.
We set the duration of each hedging problem to $T = 128$ trading days.
As in \textsc{BS-Vol}, we use a geometric Brownian motion as
the base simulator, set the drift to zero, and model the volatility parameter
$x := \sigma$ as uncertain.
We consider ten hedging problems that differ in the choice of the payoff
function $\psi$ (see details below).
For each hedging problem,
we train one policy under the static latent model and one under the
refresh latent model,
with both models sharing the same randomization distribution $\nu$.

\paragraph{Data.}
We use daily stock price data from the set of stocks that made up the
S\&P 100 index at the end of December 2015.
We fit the distribution $\nu$ based on
price data of these stocks from the period January 2006 to December 2015.
We backtest policies on the same set of stocks over the
period January 2016 to December 2025.
For each stock, we define a hedging problem on each overlapping
time window of length $T = 128$ trading days from the evaluation period.
At the start of the hedging period, we normalize the spot to $\tilde{S}_0 = 1$,
deploy the trained SLM and RLM policies, and compute the hedging loss.
We pool losses over all stocks and evaluation windows
and compute a set of evaluation metrics.

\paragraph{Randomization distribution.}
We choose the distribution $\nu$ based on historical levels of realized volatility.
From the training period January 2006 to December 2015,
we extract all overlapping return time series of length $T = 128$
trading days across all stocks.
We then model each of these time series as samples from the static latent model
\begin{equation*}
	X^2 \sim \mathrm{InvGamma}
    (\alpha^\mathrm{IG},\beta^\mathrm{IG}),
	\qquad
	Y_t \mid X \stackrel{\text{iid}}{\sim}
    N(-\frac{1}{2} X^2 \Delta t, X^2 \Delta t),
	\qquad t=1,\dots,T,
\end{equation*}
and estimate the parameters of the mixing distribution by maximum likelihood.
We obtain $\hat{\alpha}^\mathrm{IG} = 1.63$ and $\hat{\beta}^\mathrm{IG} = 0.07$
(see \cref{fig:real_data_sigma_prior} in \cref{app:sec:real_data}
for a visualization of $\nu$).
The resulting fitted law of $X$ defines the shared prior $\nu$ used
for both the SLM and RLM simulators.
We set the refresh probability of the RLM to $\alpha = 0.01$.

\paragraph{Warm-up.}
We use a warm-up period of $H=32$ trading days,
during which the policy observes prices but cannot trade yet.
This allows the initial hedge to depend on the recent
price history rather than being identical across stocks
and evaluation windows.
For the RLM, the warm-up brings the initial filter $\pi_0$
closer to the stationary regime as discussed in
\cref{subsec:approx_stationarity}.
For the SLM, adaptation to the recent past comes with a reduction in ambiguity.
We additionally report results for the SLM without warm-up in \cref{app:sec:real_data},
but find the performance difference to be small.

\paragraph{Benchmarks.}
We compare the SLM and RLM hedging policies to two
Black--Scholes delta-hedging benchmarks:
BS\nobreakdash-EWMA and BS\nobreakdash-HIST.
Let $\Delta^{\psi}(s,\tau, x)$ denote the
Black--Scholes delta of the payoff~$\psi$ at spot~$s$,
time-to-maturity~$\tau$, and volatility~$x$.
Both benchmarks choose the position held over
$[t-1,t)$ as
$U_t=\Delta^{\psi}(\tilde S_{t-1},\tau_t, \hat x_{t - 1})$,
with $\tau_t = \tfrac{T - t + 1}{250}$.
The benchmarks differ in how the volatility estimate $\hat x_{t - 1}$ is computed.
BS-EWMA uses an exponentially weighted moving-average volatility estimate,
while BS-HIST uses a historical volatility estimate based on equal weighting
over past observations.
Both volatility estimates are computed separately for each stock
(details in \cref{app:sec:real_data}).

\paragraph{Payoffs.}
We evaluate each method on ten payoffs:
call, put, straddle, strangle, bull call spread, butterfly, digital,
risk reversal, up-and-out call, and down-and-out put.
The first eight payoffs depend only on the terminal value $\tilde{S}_T$,
while the two barrier payoffs depend on the entire path $\tilde{S}_{1:T}$.
Definitions of all payoffs are given in \cref{app:subsec:real_data_payoffs}.

\subsection{Results}\label{subsec:real_data_results}
\begin{figure}[t]
	\centering
	\includegraphics[width=\textwidth]{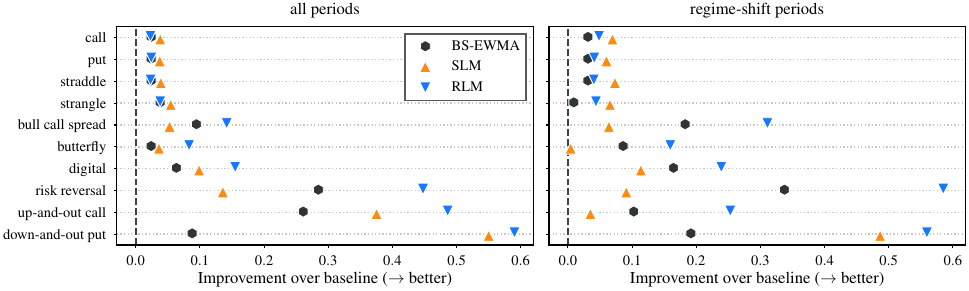}
	\caption{
        Hedging performance on S\&P 100 backtest.
		Each point reports, for one payoff and one policy,
		$\log(
		\mathcal{R}^{\mathrm{BS\text{-}HIST}}_{\mathrm{srm}} /
		\mathcal{R}_\mathrm{srm}
		)$,
		with spectral risk computed from hedging losses pooled across
		time periods and stocks.
        Positive values indicate an improvement over the BS-HIST policy.
        \emph{Left:} all evaluation windows.
        \emph{Right:} regime-shift periods.
  }
	\label{fig:real_data_results}
\end{figure}

The hedging performance of the two trained policies and the two benchmark policies
across the entire out-of-sample period and all stocks is shown in
the left panel of \cref{fig:real_data_results}.
For each payoff and each method,
we compute the spectral risk~$\mathcal{R}_\mathrm{srm}$
of the hedging losses pooled across time and stocks
and plot
$\log(
\mathcal{R}^{\mathrm{BS\text{-}HIST}}_{\mathrm{srm}} /
\mathcal{R}_\mathrm{srm}
)$,
where
$\mathcal{R}^{\mathrm{BS\text{-}HIST}}_{\mathrm{srm}}$
denotes the corresponding risk
of the BS-HIST benchmark.
Positive values indicate that the policy
achieves a lower risk and therefore improves over the benchmark policy.
The right panel shows the same relative risk pooled over a subset of evaluation windows
with large realized volatility changes.
To construct this subset,
we compute for each evaluation window the ratio of the maximum to the minimum
$21$-day rolling realized variance over the hedging period
and retain the top $10\%$ of windows by this score.

On the full sample evaluation, all four policies perform similarly
on the call, put, straddle, and strangle payoffs,
with the SLM policy having a slight edge.
On the other six payoffs, differences between methods are larger.
There, the RLM policy consistently achieves the lowest risk,
often by a substantial margin over the SLM policy.
The SLM policy mostly outperforms the baselines,
except on the risk reversal payoff,
where the simple BS-EWMA policy performs better.

On the regime-shift periods, the RLM policy again performs best overall.
The main difference from the full sample evaluation
is that the BS-EWMA benchmark now outperforms the SLM policy on five of the ten payoffs.
The poor performance of the SLM policy on the regime-shift periods
is consistent with its weak continual robustness
(\cref{subsec:synth_hedging_problems}).
In \cref{app:sec:real_data},
we report evaluations under other risk measures
and find the same qualitative pattern as in \cref{fig:real_data_results}.

\paragraph{Continuous retraining.}
A limitation of this case study is that we train a single policy
per method and payoff and deploy it across all stocks and times.
Each policy adapts to the current volatility regime only through the warm-up observations.
This is natural for the RLM policy,
because the warm-up is needed to place the simulation-implied
filter near its stationary regime.
In the SLM, it may be better to continuously retrain the policy.
At each hedge time~$t$, one could choose a new randomization distribution
$\nu_t$ from recent returns, train a new SLM policy for the remaining hedging
problem, and deploy only its first control.
Let $f^{(t)}$ denote this policy, indexed from the start of the remaining
horizon, so that the first deployed control is
$U_t = f^{(t)}_1(\tilde S_{\le t - 1})$.
The procedure could then be repeated at time $t+1$.
Such a retrained policy would likely adapt better to the current volatility regime,
since $\nu_t$ is updated over time.

While such retraining may decrease the gap between the SLM and RLM,
it generally does not make the two approaches equivalent.
In the SLM,
the control $U_t$ is still chosen on the presumption that ambiguity
will vanish over time and that the regime will not change.
How different the retrained SLM policy is from the RLM policy
depends on the particular hedging problem.\footnote{
Differences are likely large if the simulated market
lets the policy trade on future realized volatility,
for instance through options.
When the SLM policy detects a low-volatility regime,
it may choose to `sell volatility',
whereas the RLM policy anticipates possible regime shifts
and is less inclined to take such trades.
The retrained SLM policy may then perform substantially worse
than the RLM policy if the market suddenly shifts to a high-volatility regime.
}
We do not backtest such a retraining scheme here, as it would require
training a separate policy for every stock and hedge date and is
therefore computationally prohibitive.

\section{A recipe for stationary ambiguity}\label{sec:practical_recommendations}

In this section, we summarize the general recipe for policy optimization
under stationary ambiguity.
We describe
(i)~how to verify whether a given simulator induces stationary ambiguity;
(ii)~how to initialize the simulator and provide the policy with past observations
so that ambiguity is approximately stationary at the start of the control problem;
and (iii)~how to randomize a parametric base simulator
so that the randomized simulator induces stationary ambiguity.

\subsection{Does the simulator support stationary ambiguity?}\label{subsec:sim_supports_sa}

Verifying stationary ambiguity in the setup of \cref{subsec:setup}
is straightforward:
if the process $(X,Y)$ is stationary,
then \cref{prop:amb_stationary} implies that the filter is stationary.
However, many relevant control problems and simulators
are not written in the form of \cref{subsec:setup}.
For instance, financial control problems are usually defined
in terms of price processes
rather than the corresponding log-return processes (our convention).

Let us therefore start from a more general formulation.
Suppose that a simulator generates a process $S$,
which is progressively observed by the policy,
and that we are interested in a control problem of the form
\begin{equation}\label{eq:general_control_problem}
	\inf_{U\;\mathbb F\text{-predictable}}
	\quad
	\mathcal R\bigl(\tilde \ell(S_{1:T},U)\bigr),
	\qquad
	\mathcal F_t:=\sigma(S_s:s\le t).
\end{equation}
We say that the simulator of $S$ induces stationary ambiguity
for policies observing~$\mathbb F$ if
(i) problem \eqref{eq:general_control_problem} admits an equivalent formulation
as the control problem \eqref{eq:control_problem} in \cref{subsec:setup},
with observation process $Y$ and latent process $X$, and
(ii) this equivalent formulation has stationary ambiguity.
These two conditions are meant in the following sense.
Suppose that there exists a process $Y$ and
measurable maps $a_t$ and $b_t$ such that
\begin{equation*}
	Y_t = a_t(S_{\le t}),
	\quad
	S_t = b_t(Y_{\le t}),
	\qquad \text{a.s. for all $t$.}
\end{equation*}
Then, $S$ and $Y$ generate the same filtration,
so that the classes of
$\mathbb F$- and $\mathbb G$-predictable policies coincide,
where $\mathbb{G}$ again denotes the filtration generated by $Y$.
Moreover, the loss $\tilde \ell(S_{1:T},U)$ can be written,
up to almost sure equality,
in terms of $(Y_{1:T},U)$.
Hence the control problem is unchanged when written in terms of $Y$ instead of $S$.
If, in addition,
there exists a latent process $X$ such that
$(X,Y)$ is a stationary Markov process,
then \cref{prop:amb_stationary} implies that the filter
$\mathbb P(X_t\in\cdot\mid \mathcal G_t)$
is stationary.
Since $\mathcal F_t=\mathcal G_t$,
this is the same filter as
$\mathbb P(X_t\in\cdot\mid \mathcal F_t)$.
In this sense,
the original simulator of $S$ induces stationary ambiguity
for the policy observing~$\mathbb F$.

\begin{keypropertybox}
Consider a control problem defined in terms of a stochastic process $S$,
with policies adapted to the filtration~$\mathbb F$ generated by $S$.
A sufficient condition for the simulator of $S$ to induce stationary ambiguity is the existence of
\begin{enumerate}
	\item an information-equivalent observation process~$Y$;
	\item a latent process~$X$ such that $(X,Y)$ is a stationary Markov process.
\end{enumerate}
The observed process $S$ itself does not need to be stationary or Markov.
\end{keypropertybox}

The reformulation through $Y$ is only needed to establish
stationary ambiguity.
Once such a representation has been identified,
one may still simulate $S$ directly
and let the policy observe $S$.
For financial simulators in which $S$ is a price process,
a natural choice for $Y$ is often the process of
(log-)increments of $S$.

\begin{example}
Consider the Heston model from \cref{sec:synth_experiments},
but now assume that the policy observes only the price process~$S$,
while the variance process~$V$ is latent.
The price process $S$ is neither stationary nor Markov.
However,
if the initial price is deterministic,
then observing $S$ is equivalent to observing the log-return process
$Y_t:=\log(S_t/S_{t-1})$.
Moreover, the joint process $(V, Y)$ is Markov and admits a stationary regime.
The Heston simulator therefore induces stationary ambiguity,
provided that both $(V,Y)$ and the filter are in their stationary regimes.
\end{example}

\subsection{Placing simulator and policy in the stationary regime}\label{subsec:place_in_stationary}

Suppose that a given simulator can induce stationary ambiguity
in the sense of the previous section.
To ensure that the simulation-implied filter process
is indeed approximately in its stationary regime
before the first control $U_1$ is chosen,
we use the warm-up scheme from \cref{subsec:approx_stationarity}.
For a warm-up of $H>0$ observations,
we start the process $(X,Y)$ at time $1-H$,
with $(X_{1-H},Y_{1-H})$ drawn from the stationary distribution $\mu$.
We then forward simulate the process to obtain the trajectory
$(X, Y)_{t=1-H}^T$.
The policy should then be parameterized such that
$U_{t + 1}$ can depend on the trajectory $Y_{1 - H:t}$.
Besides approximating filter stationarity,
conditioning the policy on the recent path of $Y$
has the practical benefit that
it allows the same trained policy to be deployed on control problems
with different pasts.
We took advantage of this in \cref{sec:real_data},
as we backtested a single policy across many stocks and time periods.

A practical cost of warm-up is that the policy must process $H$ additional observations.
This cost can be reduced by replacing the pre-trading observations
with features that summarize their information about the latent state.
Suppose that, for each $t$,
there is a finite-dimensional statistic $B_t$ of the observed history
such that
$\P(X_t \in \cdot \mid \mathcal G_t)
\approx \P(X_t \in \cdot \mid B_t)$.
Then we can approximately recover the stationary filter regime
as follows.
We still simulate the trajectory $(X, Y)_{t=1-H}^T$
of length $H + T$.
But instead of providing the policy with the raw
pre-trading observation path $Y_{1-H:0}$,
we provide it with the summary feature $B_0$.
More generally, the policy is parameterized such that $U_{t + 1}$
can depend on the sequence $(Y_{1:t},B_{0:t})$.
In simple simulators,
$B_t$ may be a hand-chosen statistic,
such as a rolling volatility estimate.
In more complex simulators,
one can first train a separate filtering model
and provide its learned representation to the policy.

If one is only interested in deploying the policy
after one specific pre-trading history $y_{1-H:0}$,
then one may instead train a policy specialized to this past.
Concretely,
one can initialize the simulator directly at time $t = 0$ via
\begin{equation*}
	X_0 \sim \P(X_0 \in \cdot \mid Y_{1-H:0}=y_{1-H:0}),
\end{equation*}
and optimize the policy without warm-up observations.
The continuation law of $(X, Y)$ from time $t = 0$ onward
then matches that of the warmed-up simulator conditioned on $Y_{1-H:0}=y_{1-H:0}$,
so this formulation targets the same continuation problem
that the warmed-up policy faces after observing $y_{1-H:0}$.
This conditional initialization does not conflict with stationary ambiguity,
since $y_{1-H:0}$ is a realization of the random history~$Y_{1-H:0}$,
so that the initial distribution for $X_0$ is random.

Such a conditional initialization can be attractive,
since optimizing for a single past is often
easier than learning a policy that can be conditioned on different pasts.
The cost of this specialization is that one no longer obtains a single
policy that can be reused across different pasts.
Moreover, training across histories may itself help neural network policies
even when deployment ultimately concerns one fixed past,
since the policy may benefit from multi-task learning.
Finally, the conditional initialization requires sampling from the filter
$\P(X_0 \in \cdot \mid Y_{1-H:0} = y_{1-H:0})$,
which has to be estimated using separate numerical techniques.
The warm-up schemes are simpler and work well in our experiments.

\begin{keypropertybox}
\captionof{table}{
Initialization schemes for placing the simulator and policy
near the stationary regime.
}\label{tab:stationary-initialization}
\begin{center}
\vspace{-0.4em}
\begin{tabular}{@{}llll@{}}
\toprule
Method
&
Start time
&
Simulator initialization
&
Policy input for $U_t$\textsuperscript{\dag}
\\
\midrule
Warm-up
&
$1-H$
&
$(X_{1-H},Y_{1-H}) \sim \mu$
&
$Y_{\textcolor{myblue}{1-H}:t-1}$
\\
Warm-up + summary features
&
$1-H$
&
$(X_{1-H},Y_{1-H}) \sim \mu$
&
$(Y_{\textcolor{myblue}{1}:t-1},B_{0:t-1})$
\\
Conditional initialization
&
$0$
&
\begin{tabular}[t]{@{}l@{}}
$X_0 \sim \P(X_0\in\cdot\mid Y_{1-H:0}=y_{1-H:0})$ \\
$Y_0 \sim \delta_{y_0}$
\end{tabular}
&
$Y_{\textcolor{myblue}{1}:t-1}$
\\
\bottomrule
\end{tabular}
\par\vspace{0.3em}
\footnotesize
\hspace{-9em}
\textsuperscript{\dag}We use the convention $Y_{a:b}=\emptyset$ when $a>b$,
so that $U_t$ is always $\mathcal G_{t-1}$-measurable.
\end{center}
\end{keypropertybox}
\vspace{0.3em}

\begin{example}[Increasing ambiguity from deterministic initialization]
Consider again the Heston model. The variance process $V$
is commonly deterministically initialized as $V_0 = v_0$ for some constant $v_0 > 0$.
If the variance process~$V$ is unobserved by the policy,
then there is no ambiguity about the variance level at time $t = 0$,
whereas for $t > 0$ the policy is generally ambiguous about the current value
of $V_t$.
Such a systematic increase in ambiguity violates stationary ambiguity
and, like vanishing ambiguity, should be avoided
since it introduces a gap between simulation and reality.
If one believes that the current variance level can consistently be estimated accurately,
then $V$ should be observable to the policy at all times,
not just at $t = 0$.
If, on the other hand, one is generally uncertain about the current variance level,
then $V$ should be initialized via one of the three schemes
in \cref{tab:stationary-initialization},
and not deterministically.
\end{example}

This issue of increasing ambiguity due to deterministic initialization
arises in any simulator with a latent process,
including neural network simulators with a hidden state.
Any such latent state simulator should therefore be initialized
via one of the methods in \cref{tab:stationary-initialization}
(see e.g., \citet{mueller_2026_gener},
who use a conditional initialization approach
to randomly initialize the latent state of a GRU-based time series generator).

\begin{keypropertybox}
Deterministically initializing a simulator's latent state at time $t = 0$
induces a systematic increase in ambiguity after the control problem begins.
\end{keypropertybox}

\subsection{Stationary randomization}

In \cref{subsec:sim_supports_sa}, we described how to check whether
a given simulator induces stationary ambiguity.
We now ask the same question after randomizing a parametric base simulator.
Suppose that a simulator with parameter $x$ generates the process $S$,
and that this process defines the control problem as in \cref{subsec:sim_supports_sa}.
How should we randomize the parameter~$x$
so that the resulting randomized simulator induces stationary ambiguity?

As in \cref{subsec:sim_supports_sa},
let $Y$ be a process that generates the same filtration as $S$.
Let $\check X$ be a state process such that $(\check X,Y)$ is Markov.
Under a fixed parameter value $x$,
write the dynamics of $(\check X,Y)$ as
\begin{equation*}
	(\check X_t,Y_t)
	=
	F\bigl(x,(\check X_{t-1},Y_{t-1}),W_t\bigr),
\end{equation*}
where $W$ is an i.i.d.\ noise process and $F$ is a transition function.
Now replace the fixed parameter $x$ by a stochastic process $X$.
If $X$ is a stationary Markov process that is independent of $W$,
and if the map $F$ satisfies the uniform contraction and
boundedness condition of \cref{prop:stationary_randomization}
in the state variable $(\check X_t,Y_t)$,
then \cref{prop:stationary_randomization} implies that
$(X,\check X,Y)$ is stationary.
The latent state of the randomized simulator is then
$\bar X_t := (X_t,\check X_t)$,
which consists of the randomized parameter and the additional state needed
to make the base simulator Markov.
The additional state $\check X_t$ may be $\mathcal G_t$-measurable,
for instance when $(Y_t,\dots,Y_{t-k})$ is Markov under the base simulator.
It may also contain latent variables (see Heston example below).
Since $(\bar X,Y)$ is stationary,
\cref{prop:amb_stationary} implies that the filter
$\mathbb P(\bar X_t\in\cdot\mid\mathcal G_t)$
has a stationary version.
Because $Y$ generates the same filtration as $S$,
this is the filter faced by the policy in the original control problem.
Thus, the randomized simulator induces stationary ambiguity whenever
there exists a representation for which the above condition holds.

\begin{example}[Randomized Heston model with unobserved variance]
Consider the Heston model as the base simulator
and assume that the policy observes the price process
$S$, but not the variance process $V$.
If the initial value of $S$ is deterministic,
observing $S$ is equivalent to observing the
log-return process $Y_t := \log(S_t/S_{t-1})$.
Taking $\check X_t := V_t$, the pair $(\check X,Y)$ is Markov and can be written as
\begin{equation*}
	(V_t,Y_t)
	=
	F\bigl(x,(V_{t-1},Y_{t-1}),W_t\bigr),
\end{equation*}
for the Heston parameter vector $x=(\kappa,\bar v,\xi,\rho)$.
If $x$ is replaced by a stationary Markov process $X$ taking values in a compact
parameter set with $\kappa,\bar v,\xi$ bounded away from zero,
then \cref{app:subsec:randomized_heston_proofs} verifies the required
condition on $F$.
Thus $(X,V,Y)$ is stationary, and \cref{prop:amb_stationary} implies stationary
ambiguity for the latent state $\bar X_t = (X_t,V_t)$.
\end{example}

\paragraph{Choosing the randomization process $X$.}
The main modeling choice for randomizing a given base simulator
is the stochastic process used for $X$.
Stationary ambiguity only restricts $X$ to be a stationary Markov process
whose values lie in the parameter region
on which the uniform condition for $F$ holds.
In this paper, we considered two simple classes for the process $X$.
In \cref{subsec:opt_inv_example},
we randomized with an autoregressive process of order one
so that the latent parameter drifts continuously.
In our numerical experiments, we used the refresh latent model,
in which the latent parameter jumps occasionally to a new value.
The RLM is a natural default choice
because it
introduces only a single additional parameter compared to the SLM,
and preserves the dynamics of the base simulator
between parameter refreshes.

For the RLM, the two modeling choices are the stationary law $\nu$
and the refresh probability $\alpha$.
Choosing these quantities is part of the broader question
of what simulator is most appropriate for the real data problem,
and we do not attempt to provide a general rule for this modeling choice.
Instead, we give one concrete example of how to randomize the
volatility parameter of a geometric Brownian motion
via the RLM in \cref{sec:real_data}.
There, we fit $\nu$ by treating historical $128$-day return windows
as draws from the SLM.
We then choose $\alpha$
small enough that the base simulator is not overly distorted,
but large enough that parameter refreshes occur with non-negligible
probability during training.
Similar approaches can be applied to other base simulators
and uncertain parameters.

\section{Alternative approaches and related work}

\subsection{Ambiguity from information restrictions}\label{subsec:info_restrictions}
An alternative response to vanishing ambiguity is not to modify the
simulator, but to restrict the information available to the policy.
Since allowing the policy to use all observable information leads to
ambiguity reduction through filtering,
it may seem natural to restrict policies to a form such as
\begin{equation*}
	U_t = f(t, I_{t - 1}) ,
\end{equation*}
where $I_t$ is a $\mathcal{G}_t$-measurable feature vector.
For any measurable $\varphi$ with $\varphi(X_t) \in L^2$,
the law of total variance gives
\begin{equation*}
	\E\bigl[
			\Var(\varphi(X_t) \mid I_t)
	\bigr]
	\geq
	\E\bigl[
			\Var(\varphi(X_t) \mid \mathcal{G}_t)
	\bigr] .
\end{equation*}
Thus, on average, the information-restricted policy faces at least as much
ambiguity about the latent state as a policy with access to the full observation
history.
Even if the latent process $X$ is static,
ambiguity may then persist as $t \to \infty$.

While the additional ambiguity may make the policy more robust in some regards,
the information restriction fundamentally changes the control problem.
In the real problem, the full observation history is available through $\mathbb{G}$,
but the policy is optimized for a surrogate
problem in which part of this information has been deliberately removed.
The resulting sim-to-real gap can have undesirable effects on the learned policy.
For example, consider again modeling uncertainty about volatility
in the Black--Scholes model through static randomization
(SLM variant of \textsc{BS-Vol}),
and parameterize the policy as $U_t = f(t, S_{t - 1})$,
where $S$ is the stock price process.
Observing only the current stock price
prevents the policy from accurately estimating
and adapting to the current volatility regime.
Deploying such a policy may result in the following behavior.
In low-volatility periods,
not adapting to the current volatility regime makes the policy act more defensively than a
fully informed policy, which can look like a robustness benefit.
In high-volatility periods, the information restriction may have the opposite effect:
if the policy fails to recognize that volatility is high,
it may under-hedge at a time when this is most costly.\footnote{
In this example,
ambiguity still completely vanishes over time,
although at a slower rate than under full information.
Let $S$ be a geometric Brownian motion with $S_0 = 1$,
known drift $\mu = 0$,
and volatility $\sigma$.
Then,
$\tfrac{1}{t} \log S_t \to - \tfrac{1}{2} \sigma^2$
almost surely.
That is, the volatility $\sigma$
can be precisely identified as $t \to \infty$,
even under the non-increasing information structure
$\mathbb{F}$ with $\mathcal{F}_t = \sigma(S_t)$.
}

\begin{keypropertybox}
Restricting the policy's information artificially inflates ambiguity,
creating a gap between simulation and reality.
\end{keypropertybox}

\subsection{Vanishing ambiguity in other robustness schemes}\label{sec:other_robustness_schemes}

The issue of weak continual robustness due to vanishing ambiguity
can arise in robust optimization schemes other than static randomization
that also treat the simulator parameter as an unobserved but fixed quantity.
To see this, we revisit the optimal investment example from
\cref{subsec:opt_inv_example}
and compare two minimax-type approaches commonly used to make policies robust to
uncertain simulator parameters.
First, we consider the worst-case optimization
\begin{equation*}
        \inf_U \sup_{x\in[\underline x,\overline x]} \;
        - \sum_{t=1}^T \E_x\bigl[1 - \exp(- \lambda U_tY_t)\bigr] ,
\end{equation*}
where $\mathbb{E}_x$ denotes expectation under
$Y_t \stackrel{\text{iid}}{\sim} N(x, \sigma^2)$.
This problem admits the trivial solution
\begin{equation*}
        U_t^\star
        =
        \frac{x_\dagger}{\lambda\sigma^2},
        \qquad
        x_\dagger \in \mathrm{argmin}_{x\in[\underline x,\overline x]} |x| .
\end{equation*}
That is, the optimal policy is simply the known-drift optimal
policy evaluated at the drift in $[\underline x,\overline x]$ closest to zero.
This minimax policy does not specialize to observed returns,
but is overly conservative.
If $0\in[\underline x,\overline x]$
(e.g., both negative and positive drifts are plausible),
then $x_\dagger=0$
and hence $U_t^\star\equiv 0$.
In other words, the robust solution recommends never investing,
even if we believe that a positive drift is more likely.

\begin{wrapfigure}[11]{r}{0.5\textwidth}
	\vspace{-\baselineskip}
	\centering
	\includegraphics[width=\linewidth]{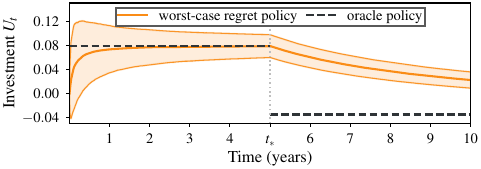}
	\caption{Minimax regret policy fails to adapt to a regime shift.
	We evaluate the minimax regret policy (\cref{eq:worst_regret})
	on the regime shift scenario of \cref{subsec:opt_inv_example}.
	The plot shows mean and $10$--$90\%$ interval of $U_t$ across draws of $Y$.
	}
	\label{fig:worst_case_regret_regime_switch_control_quantiles}
\end{wrapfigure}

A less conservative alternative is to take the worst case over \emph{regret}
\citep{savage_1951_the,xu_2009_param,cacador_2022_a}.
Regret removes the effect that some drifts are intrinsically less favorable
than others by comparing the policy's utility to that of the
oracle action $a^\star(x)=x/(\lambda\sigma^2)$.
The minimax-regret objective is
\begin{equation}\label{eq:worst_regret}
		\inf_U \sup_{x\in[\underline x,\overline x]} \;
		\sum_{t=1}^T
		\mathbb E_x\bigl[
				j_x(a^\star(x))-j_x(U_t)
		\bigr] ,
\end{equation}
where ${j_x(u):=\mathbb E_x[1-\exp(-\lambda uY_t)]}$.
We show in \cref{app:sec:optimal_investment_example}
that near-optimal policies for \cref{eq:worst_regret}
have a simple asymptotic characterization.
Let $U^{(T)}$ denote controls from such a policy for horizon $T$.
For every fixed drift $x\in[\underline x,\overline x]$ and every $\epsilon>0$,
\begin{equation*}
		\frac1T\sum_{t=1}^T
		\mathbb P_x\bigl(
				|U_t^{(T)}-a^\star(x)|>\epsilon
		\bigr)
		\longrightarrow 0 .
\end{equation*}
The fraction of times at which
$U_t^{(T)}$ differs noticeably from the known-drift action $a^\star(x)$ goes
to zero as the horizon grows.
The policy must therefore specialize to the latent parameter $x$ over time.
Worst-case regret over static parameters
therefore has the same qualitative effect on the policy as static randomization
and both methods share the same failure mode.
\Cref{fig:worst_case_regret_regime_switch_control_quantiles}
visualizes the behavior of a policy, numerically optimized
to solve \cref{eq:worst_regret},
on the regime shift scenario of \cref{subsec:opt_inv_example}.
We find that, just like the SLM policy, the
minimax-regret policy overly specializes to the initially observed drift
and fails to adapt to the regime change.

\subsection{Related work}\label{subsec:related_work}

\paragraph{Ambiguity and learning in finance.}
Stationary ambiguity has a motivation similar to that of the work of
\citet{epstein_2007_learn},
who argue that environments, such as financial markets, are often too ``complicated''
for agents to identify all the environment's ambiguous parameters.
They propose a model of how agents learn about their environment
that allows for some parameters to be identified,
but requires ambiguity over a set of likelihoods on the system's future evolution
to persist over time.
Unlike our work, they study decisions under the worst admissible likelihood,
so they do not specify a probabilistic model for the latent factors that generate
this persistent ambiguity.\footnote{
For example, they study an investment problem in which
an asset's mean log return is equal to $\theta + X_t$,
where $\theta$ is an identifiable parameter and
$X_t \in [-a,a]$ is an unobserved ambiguous component.
Since agents optimize decisions under the worst admissible value of $X_t$,
they do not treat $X$ as a stochastic process with a specified law.
}
See also \citet{epstein_2008_ambig} and \citet{leippold_2008_learn}.

The works of \citet{hansen_2007_recur} and \citet{ju_2012_ambig}
are conceptually closer to ours.
In both works, ambiguity arises due to a latent stochastic process
and is represented by the filter.
However, their focus is on how a decision maker should respond to such ambiguity,
and less on how such ambiguity arises.
\citet{ju_2012_ambig} study recursive smooth ambiguity preferences
\citep{klibanoff_2005_a,klibanoff_2009_recura},
which could also be used in the hedging problems studied here
by choosing the hedge $U_{t + 1}$ as the maximizer of
\begin{equation*}
V_t \; = \;
\max_{a \in \actspace}
\; \;
v^{-1}\Bigl(
\E_{x \sim \pi_t}
\Bigl[
v\bigl(
u^{-1}\bigl(
\E[u(V_{t + 1}) \mid X_t = x, U_{t + 1} = a]
\bigr)
\bigr)
\Bigr]
\Bigr) ,
\end{equation*}
with $V_T = -L$, where $L$ is the terminal hedging loss.
Here, the utility functions $u$ and $v$ separately account
for risk and ambiguity aversion.
Our argument for stationary ambiguity is independent of the decision criterion
and compatible with preferences such as the above
or the approach of \citet{hansen_2007_recur},
who induce robustness with respect to $X_t$ through regularized perturbations
of the filter $\pi_t$.
Practically, optimizing a policy under such dynamic criteria
is challenging and requires explicitly computing $\pi_t$.
The static objectives
used in our experiments do not distinguish between ambiguity and risk,
treating the randomized simulator as a compound lottery.
We refer to \citet{guidolin_2013_ambig} and \citet{ilut_2022_model}
for detailed reviews of models of ambiguity preferences.

\citet{nagel_2022_asset} study empirical asset pricing and,
similar to us,
argue that learning effects in static latent-type models
make parameter uncertainty decay too quickly.
They address this by directly imposing a constant-gain rule with which agents update
their beliefs.\footnote{
In their model, agents update their beliefs about the economy's mean growth rate
using an update rule that has the form of a steady-state Kalman-filter recursion.
Since the Kalman filter has exponential forgetting, the agent's memory fades over time.
}
In contrast, we do not impose a rule on how the controller should update beliefs.
We start with the simulator and require it to induce a stationary filter,
which is less specific than the learning rule imposed by \citet{nagel_2022_asset};
see also \citet{dangl_2022_conse}.
Similarly, \citet{zimper_2011_do} imposes a non-standard belief updating rule
under which ambiguity need not vanish.

Many works have studied the effects of learning about a random latent drift
on optimal investment policies
\citep{lakner_1995_utili,brennan_1998_the,karatzas_2001_bayes,bauerle_2024_optima}.
The optimal investment problem in which we study vanishing ambiguity
in \cref{subsec:opt_inv_example} is deliberately kept simple
for illustrative purposes.

\paragraph{Hidden Markov models in finance.}
We argue for simulators with latent state processes that induce
a stationary filter process.
Hidden Markov models, if placed in the stationary regime,
often satisfy this condition and have been used as ``simulators''
across many control problems in finance,
including optimal investment
\citep{honda_2003_optim,rieder_2005_portf,liu_2011_dynam}
and hedging problems
\citep{naik_1993_optio,dimasi_1995_meanv,ghosh_2009_risk,elliott_2023_hedgi}.
Several works have studied regime-switching variants of the Heston model,
including
\citet{elliott_2007_prici,goutte_2013_prici,papanicolaou_2014_a,elliott_2016_hesto}.

\paragraph{Robust hedging.}
The problem of obtaining robust hedging strategies has classically been studied
using worst-case approaches over uncertainty sets, such as intervals for the
instantaneous volatility in Black--Scholes-type models
\citep{avellaneda_1995_prici,avellaneda_1996_managa,lyons_1995_uncera}
or bounds on cumulative realized volatility
\citep{mykland_2000_conse}.
\citet{herrmann_2017_model} and \citet{herrmann_2017_hedgi}
consider regularized worst-case formulations around a reference model.

These works focus on specific models for the underlying asset prices,
such as the Black--Scholes model.
Deep hedging \citep{buehler_2019_deep}
has enabled numerical approximation of optimal hedging policies
for a much broader class of financial market simulators
\citep{buehler_2020_a,wiese_2021_multia,cont_2025_tailg,cont_2025_datad,mueller_2026_gener}.
With this flexibility,
simulator choice and robustness with respect to simulator parameters
are of particular importance \citep{cohen_2023_black}.
\citet{lutkebohmert_2022_robus} suggest randomizing the parameters $x$
of an SDE-based simulator by replacing them with an i.i.d.\ process $X$,
which can result in overly conservative strategies
(\citealp{jones_2025_ambig}; see also \cref{subsec:synth_hedging_problems}).
Adversarial approaches for robustifying deep hedging strategies
have been proposed by \citet{limmer_2024_robus},
who perturb the simulator,
\citet{wu_2023_robus},
who perturb the induced terminal-wealth distribution,
and \citet{he_2025_distr},
who perturb the empirical distribution of training paths.
\citet{buehler_2025_uncer} subsample mini-batches of paths
from a fixed simulator, compute the hedging loss on each mini-batch,
and train the policy on the worst tail of these losses.
\citet{jones_2025_ambig} use clustering to construct
a mixture simulator from a given empirical path distribution,
and optimize policies under smooth ambiguity preferences \citep{klibanoff_2005_a}.

\paragraph{Domain randomization.}
Randomizing uncertain simulator parameters to improve how well a policy
optimized in simulation transfers to the real environment
is common practice in reinforcement learning
\citep[see][and references within]{muratore_2022_robot}.
The randomized parameter process is usually taken to be static
or i.i.d.\ \citep{tobin_2017_domai,rajeswaran_2017_epopta,peng_2018_simto}.

\paragraph{Exploration.}
Stationary ambiguity is a modeling principle that is most applicable
to control problems with uncontrolled simulator dynamics,
meaning that the controller cannot affect future observations through its actions.
Thus, the controller cannot reduce ambiguity
by choosing informative actions (no exploration).\footnote{
For instance, the drift of a stock return process cannot be revealed
by picking a particular trading strategy.}
This is in contrast to
many partially observable Markov decision processes
\citep{kaelbling_1998_plann,ghavamzadeh_2015_bayes}
studied in reinforcement learning,
where actions affect future observations and can be used to gather information.
For those problems, stationary ambiguity is not a fitting assumption.
Consider for instance a robot that is placed in a new room.
The robot should be able to explore the room and systematically
decrease its ambiguity over time.

\section{Conclusion}\label{sec:conclusion}

We introduced stationary ambiguity,
a modeling principle for representing ambiguity in simulators
for control problems driven by exogenous stochastic processes.
Stationary ambiguity applies most naturally in domains such as finance,
when controls depend on
an observable process $Y$ whose dynamics are expected to remain uncertain over time.
In such problems, policies should remain robust to latent factors
governing the dynamics of $Y$.
Our experiments showed that simulators with stationary ambiguity induce such
continual robustness, whereas in simulators with vanishing ambiguity
robustness decays over time.
Stationary ambiguity can be used to inform
many modeling decisions in simulation-based policy optimization,
including what simulator to use, how to randomize an existing simulator,
how to initialize the simulation and the policy,
and what information to supply to the policy.

There are several relevant directions for future work.
One is to study how ambiguity dynamics affect robustness
in a wider class of control problems, such as
realistic portfolio selection problems under market frictions \citep{boyd_2017_multi},
hedging problems in which vanilla options are tradable
\citep{cao_2023_gamma,mueller_2024_fasta,francois_2025_deep},
and applications beyond finance such as inventory control,
energy storage, and queueing.
Another direction is to study the choice of the randomization process in more detail,
as we mostly worked with the simple refresh latent model.
Finally, it would be useful to further study
the consequences of stationary ambiguity
for the design of neural network-based simulators.

\clearpage
\section*{Acknowledgments}
KM is supported by JPMorgan Chase \& Co. through the EPSRC Centre for Doctoral Training in Mathematics of Random Systems: Analysis, Modelling and Simulation (EPSRC Grant EP/S023925/1).
{\small
\bibliography{references_current}
}
\clearpage
\appendix
\crefalias{section}{appendix}
\crefalias{subsection}{appendix}
\crefalias{subsubsection}{appendix}
\raggedbottom
\AppendixToC
\clearpage
\section{Additional details and results for the hedging problems}\label{app:sec:synth}

In this section, we provide additional details and results
for the three hedging problems studied in \cref{sec:synth_experiments}.

\subsection{Experiment details}\label{app:subsec:synth_details}

\subsubsection{Hedging problems}

The base simulator for each of the three hedging problems
is described in \cref{subsec:hedging_problems}.
We consider two randomized versions of each base simulator (SLM and RLM).
We use the same randomization distribution $\nu$ for both SLM and RLM
(see \cref{tab:synth_randomization_distributions}).
For all three hedging problems,
we use the refresh probability $\alpha = 0.01$
for the RLM simulator.

\begin{table}[h]
	\centering
	\small
	\caption{Randomization distributions used in \cref{sec:synth_experiments}.
			We write $\mathrm{InvGamma}(a,b)$ for an inverse-gamma distribution
			with shape $a$ and scale $b$, and $\mathrm{Beta}(a,b)$ for a beta
			distribution with shape parameters $a$ and $b$.
			For \textsc{BS-Cov}, each draw from $\nu$ samples
			$\sigma_1$, $\sigma_2$, and $\varrho$ independently.
			}
	\label{tab:synth_randomization_distributions}
	\begin{tabular}{@{}lll@{}}
			\toprule
			Problem & Randomized parameter & Randomization distribution $\nu$ \\
			\midrule
			\textsc{BS-Vol}
			&
			$\sigma$
			&
			$\sigma^2 \sim \mathrm{InvGamma}(5.93,\,0.16)$
			\\
			\textsc{Heston-Corr}
			&
			$\rho$
			&
			$\rho = 2\eta - 1,\quad \eta \sim \mathrm{Beta}(1.80,\,5.50)$
			\\
			\textsc{BS-Cov}
			&
			$(\sigma_1,\sigma_2,\varrho)$
			&
			\begin{tabular}[t]{@{}l@{}}
					$\sigma_i^2 \sim \mathrm{InvGamma}(5.93,\,0.16),\ i=1,2,$ \\
					$\varrho = 2\eta - 1,\quad \eta \sim \mathrm{Beta}(4.00,\,2.00)$
			\end{tabular}
			\\
			\bottomrule
	\end{tabular}
\end{table}

\paragraph{Base simulators are amenable to stationary randomization.}
In each of the three hedging problems,
the base simulator is amenable to stationary randomization
in the sense of \cref{prop:stationary_randomization}.
For \textsc{BS-Vol} and \textsc{BS-Cov}, this is straightforward.
The price process $S$ itself is not stationary.
But since $S_0$ is fixed, observing $S$ is equivalent to observing the
log-return process $Y$.
For \textsc{BS-Vol}, $Y_t=\log(S_t/S_{t-1})$.
For \textsc{BS-Cov}, $Y_t$ is the vector of log-returns.
For fixed parameter $x$, these observations have the form
\begin{equation*}
	Y_t = F(x,W_t),
\end{equation*}
with i.i.d.\ noise $W_t$ and no dependence on $Y_{t-1}$.
After replacing $x$ by a stationary randomization process $X_t$
that is independent of $W$, we get $Y_t=F(X_t,W_t)$.
Hence $(X,Y)$ is stationary, and
\cref{prop:amb_stationary} gives stationary ambiguity.

In \textsc{Heston-Corr}, we assume that the observed process is
$Y_t=(\log(S_t/S_{t-1}),V_t)$.
This process is Markov, but not i.i.d.,
because the next observation depends on the current variance level.
\cref{app:subsec:randomized_heston_proofs}
shows that the Heston transition map satisfies the contraction condition in
\cref{prop:stationary_randomization},
uniformly over compact parameter sets with $\kappa,\bar v,\xi>0$.
Hence, $(X, Y)$ is stationary and 
\cref{prop:amb_stationary} again gives stationary ambiguity.

In \textsc{Heston-Corr}, we only randomize the correlation parameter $\rho$
(see discussion below).
The other parameters of the Heston model are set to
$\kappa = 8.0$, $\bar v = 0.0625$, and $\xi = 1.0$.

\subsubsection{Evaluating initial and continual robustness}

The policies used for the initial and continual robustness analysis in
\cref{subsec:synth_hedging_problems}
differ in whether the policy is trained with a warm-up period
prior to the start of the hedging problem.
During the warm-up period, the policy observes $Y$ but cannot trade yet.
Instead of training two separate policies, one without warm-up and one with a
warm-up of length $T$, we train a single policy
on a longer hedging problem of length $2T$.
The first half represents the no-warm-up hedging problem,
and the second half represents the hedging problem after the warm-up period.

Concretely, we set $S_0 = 1$ and simulate one price path
$S_1,\dots,S_{2T}$ of length $2T$,
together with the corresponding observation path
$Y_1,\dots,Y_{2T}$.
Let $\mathbb G=(\mathcal G_t)_{t=0,\dots,2T}$ be the filtration generated by $Y$,
\begin{equation*}
	\mathcal G_t = \sigma(Y_s : s=1,\dots,t),
\end{equation*}
with $\mathcal G_0$ trivial.
We then train a single policy to produce a $\mathbb G$-predictable control process
$U=(U_1,\dots,U_{2T})$ on the full path.
This control process is then interpreted as the concatenation of
controls for two separate hedging problems defined
on the following split and normalized paths:
\begin{equation*}
	\tilde S^{(1)}_t := \frac{S_t}{S_0},
	\qquad
	\tilde S^{(2)}_t := \frac{S_{T+t}}{S_T},
	\qquad t = 0,\dots,T .
\end{equation*}
That is, we use the split paths $\tilde{S}^{(1)}$ and $\tilde{S}^{(2)}$
to define two identical and separate hedging problems
with the same payoff function $\psi$ and the same maturity $T$.
For a given control process $U=(U_1,\dots,U_{2T})$, we define the two hedging losses
\begin{align*}
	L^{(1)}(U)
	&:=
	\psi\bigl(\tilde S^{(1)}\bigr)
	-
	\sum_{t=1}^{T}
	U_t^\top
	\bigl(\tilde S^{(1)}_t-\tilde S^{(1)}_{t-1}\bigr), \\
	L^{(2)}(U)
	&:=
	\psi\bigl(\tilde S^{(2)}\bigr)
	-
	\sum_{t=1}^{T}
	U_{T+t}^\top
	\bigl(\tilde S^{(2)}_t-\tilde S^{(2)}_{t-1}\bigr).
\end{align*}
The controls $U_1, \dots, U_T$ are therefore used to hedge
the payoff $\psi(\tilde{S}^{(1)})$ by trading the normalized spot
$\tilde{S}^{(1)}$.
The controls $U_{T + 1}, \dots, U_{2T}$ are used to hedge
the payoff $\psi(\tilde{S}^{(2)})$ by trading the normalized spot
$\tilde{S}^{(2)}$.
As long as the distributions of
$\tilde{S}^{(1)}$ and $\tilde{S}^{(2)}$ are identical,
the two hedging problems only differ in the information available to the policy.
In the first hedging problem, the policy has no prior observations of the
path $S$, whereas in the second hedging problem it can observe
$S$ prior to the start of the hedging problem.

To keep both hedging problems separate,
the policy is trained to minimize
the sum of the two risks
\begin{equation*}
	\mathcal R\bigl(L^{(1)}(U)\bigr)
	+
	\mathcal R\bigl(L^{(2)}(U)\bigr) ,
\end{equation*}
so that the first term only depends on $U_{1:T}$ and the second
only depends on $U_{T + 1:2T}$.

To evaluate initial robustness, we compute $L^{(1)}$
on paths with $X_{1:T}=x^\mathrm{eval}$.
To evaluate continual robustness, we compute $L^{(2)}$
on paths with $X_{1:T}=x^\mathrm{pre}$
and $X_{T+1:2T}=x^\mathrm{eval}$.

\paragraph{Matching the two hedging problems.}
In both \textsc{BS-Vol} and \textsc{BS-Cov},
the normalized spot process is Markov.
If the latent parameter takes the same value during the hedging period,
$\tilde S^{(1)}$ and $\tilde S^{(2)}$ have the same distribution,
because both are normalized to start at $1$.
Hence, the first- and second-half hedging problems differ only in the information
available to the policy before trading starts.

In \textsc{Heston-Corr}, the normalized spot process is not Markov,
meaning that simply fixing the initial values does not guarantee
that $\tilde S^{(1)}$ and $\tilde S^{(2)}$ have the same distribution.
Still, the joint processes
$(\tilde S^{(1)}_t, V_t)_{t=0,\dots,T}$
and
$(\tilde S^{(2)}_t, V_{T+t})_{t=0,\dots,T}$
are Markov.
They have the same law whenever $\rho$ takes the same value during the hedging
period and $V_0$ is sampled from the stationary CIR distribution.
We therefore sample the initial variance from the stationary CIR
distribution induced by the fixed parameters $(\kappa,\bar v,\xi)$,
\begin{equation*}
	V_0
	\sim
	\mathrm{Gamma}\bigl(
			\frac{2\kappa\bar v}{\xi^2},
			\frac{\xi^2}{2\kappa}
	\bigr) .
\end{equation*}
This distribution does not depend on the randomized parameter $\rho$.
If $\kappa$, $\bar v$, or $\xi$ were randomized, the warm-up regime would also
change the stationary distribution of $V$ and therefore the variance level at
the start of the second hedging problem.
This would confound the continual robustness analysis.

\subsubsection{Loss functions}\label{app:subsec:loss_functions}

We report all metrics in terms of the hedging loss $L$, where larger values are worse.
Let $q_L$ denote the quantile function of $L$ and let $x_+ := \max\{x,0\}$.
The loss functionals used in the appendix figures are
\begin{align*}
	\mathcal R_{\mathrm{var}}(L)
	&:= 10^3 \times \operatorname{Var}(L), \\
	\mathcal R_{\mathrm{sd}}(L)
	&:=
	\biggl(
			\mathbb E\bigl[
					\bigl(L-\mathbb E[L]\bigr)_+^2
			\bigr]
	\biggr)^{1/2}, \\
	\mathcal R_{\mathrm{cvar}}(L)
	&:=
	\frac{1}{1-\beta}
	\int_\beta^1 q_L(p)\,dp,
	\qquad \beta = 0.95, \\
	\mathcal R_{\mathrm{srm}}(L)
	&:=
	\int_0^1 q_L(p)\,\phi_\gamma(p)\,dp,
	\qquad
	\phi_\gamma(p)
	=
	\frac{\gamma e^{\gamma(p-1)}}{1-e^{-\gamma}},
	\qquad \gamma = 4 .
\end{align*}
The policies in the experiments reported in the main text
are trained with $\mathcal R_{\mathrm{srm}}$.
In the next section, we provide results for an additional experiment on the
\textsc{BS-Vol} problem that uses a policy trained with $\mathcal R_{\mathrm{var}}$.

\clearpage

\subsection{Additional results}\label{subsec:add_synth_results}

This section reports additional continual-robustness results for the
experiments in \cref{sec:synth_experiments}.

\paragraph{Detailed and relative continual-robustness curves.}
We show the continual-robustness curves under the four evaluation criteria:
$\mathcal R_{\mathrm{var}}$,
$\mathcal R_{\mathrm{sd}}$,
$\mathcal R_{\mathrm{cvar}}$,
and $\mathcal R_{\mathrm{srm}}$.
Unless stated otherwise, the policies are trained under $\mathcal R_{\mathrm{srm}}$.

We also report relative continual-robustness curves.
These curves compare a policy's loss after a regime change
from $x^\mathrm{pre}$ to $x^\mathrm{eval}$
with the loss of the SLM policy evaluated on paths where
the pre-trading and trading regimes are both $x^\mathrm{eval}$.
For a policy $m$, write
\begin{equation*}
\begin{aligned}
	\mathcal{L}^{m}(x^\mathrm{pre}, x^\mathrm{eval})
	&:=
	\mathcal{R}_\mathrm{eval}\bigl(
			L^m
			\mid
			X_{-T+1:0} = x^\mathrm{pre},
			X_{1:T} = x^\mathrm{eval}
	\bigr), \\
	x^\mathrm{eval}
	&\mapsto
	\mathcal{L}^{m}(x^\mathrm{pre}, x^\mathrm{eval})
	-
	\mathcal{L}^{\mathrm{SLM}}(x^\mathrm{eval}, x^\mathrm{eval}) .
\end{aligned}
\end{equation*}
Since larger losses are worse, values near zero mean that policy $m$
matches the SLM policy adapted to $x^\mathrm{eval}$,
while positive values indicate worse performance.

\paragraph{Additional ablations.}

We include two ablations for \textsc{BS-Vol}.
First, we repeat the experiment with policies trained under
$\mathcal R_{\mathrm{var}}$ instead of $\mathcal R_{\mathrm{srm}}$.
We find that the policy trained under
$\mathcal R_{\mathrm{var}}$
has the same continual-robustness pattern
as the policy trained under $\mathcal R_{\mathrm{srm}}$
(\cref{fig:app:synth:bs1d_var_metric_abs_rel_overview}).
Second, we replace the straddle by a risk-reversal payoff.
As shown in \cref{fig:app:synth:bs1d_risk_reversal_metric_abs_rel_overview},
we again find the same qualitative difference between
the SLM and RLM policies that we discuss in the main text.

\paragraph{Variable warm-up length ablation.}

\Cref{fig:fig1} on the first page shows a further ablation
of the experiments in \cref{sec:synth_experiments}.
There, we repeat the \textsc{BS-Vol} regime-shift evaluation
for varying pre-trading lengths $H \in \{4,8,16,32,64,128\}$
(there called time $t$).
For each $H$, we train one SLM policy and one RLM policy
using the same training objective as in \cref{sec:synth_experiments}.
The ambiguity panel reports the exact filter quantity
$\E[\Var(\log X_H \mid \mathcal G_H)]$.
The loss panel reports
\begin{equation*}
	\max_{(x^\mathrm{pre},x^\mathrm{eval}) \in \mathcal X_{\mathrm{grid}}^2}
	\mathcal R_{\mathrm{sd}}
	\bigl(
			L
			\mid
			X_{1:H}=x^\mathrm{pre},
			X_{H+1:H+T}=x^\mathrm{eval}
	\bigr) ,
\end{equation*}
where $\mathcal X_{\mathrm{grid}}$ is the volatility grid used for the
continual robustness evaluation.
Thus, each point in the loss panel summarizes the continual-robustness curves
by their worst regime-switch loss.

\subsubsection{\textsc{BS-Vol}}

\begin{figure}[H]
	\centering
	\includegraphics[width=\textwidth]{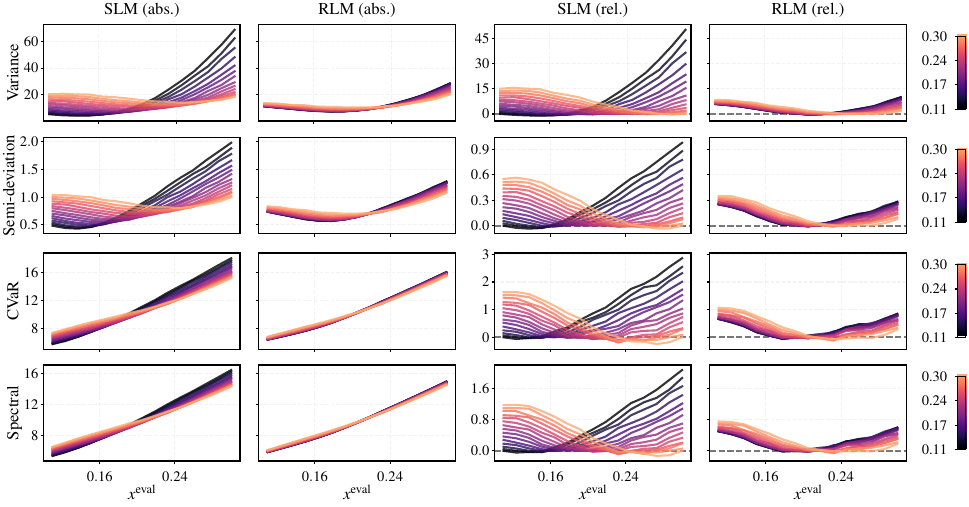}
	\caption{Continual robustness results for \textsc{BS-Vol}.}
	\label{fig:app:synth:bs1d_metric_abs_rel_overview}
\end{figure}

\subsubsection{\textsc{BS-Vol} (variance objective)}

\begin{figure}[H]
	\centering
	\includegraphics[width=\textwidth]{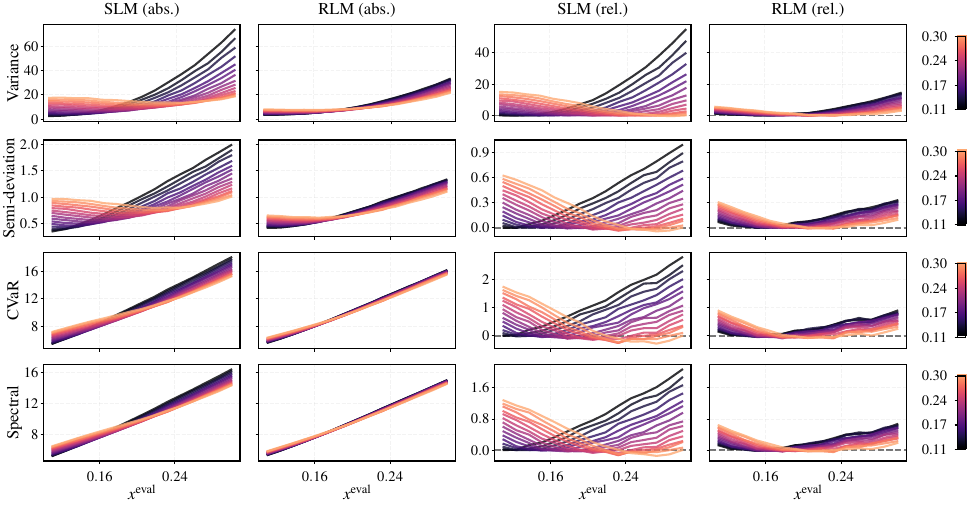}
	\caption{Continual robustness results for \textsc{BS-Vol},
	when the policy is optimized with $\mathcal R_{\mathrm{var}}$.}
	\label{fig:app:synth:bs1d_var_metric_abs_rel_overview}
\end{figure}

\subsubsection{\textsc{BS-Vol} (risk reversal)}

\begin{figure}[H]
	\centering
	\includegraphics[width=\textwidth]{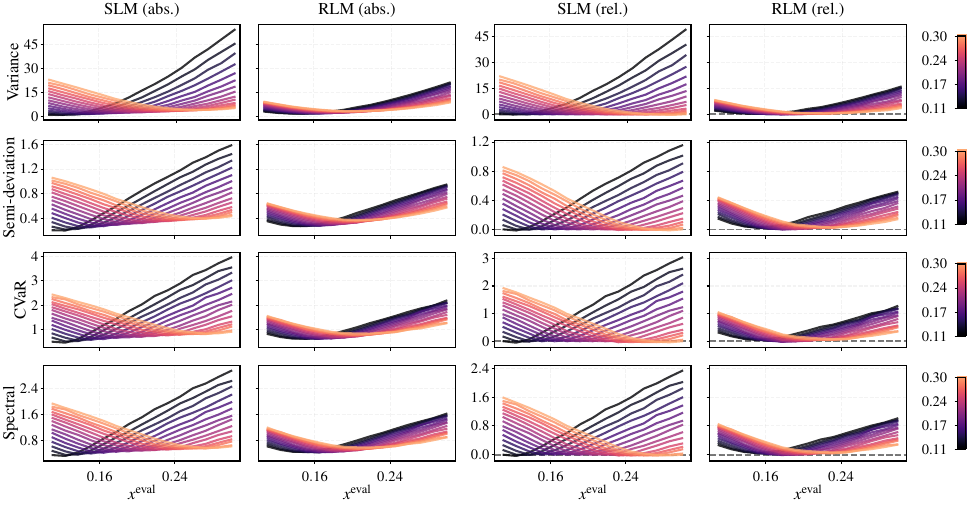}
	\caption{Continual robustness results for \textsc{BS-Vol} when the payoff is a risk reversal.}
	\label{fig:app:synth:bs1d_risk_reversal_metric_abs_rel_overview}
\end{figure}

\subsubsection{\textsc{Heston-Corr}}

\begin{figure}[H]
	\centering
	\includegraphics[width=\textwidth]{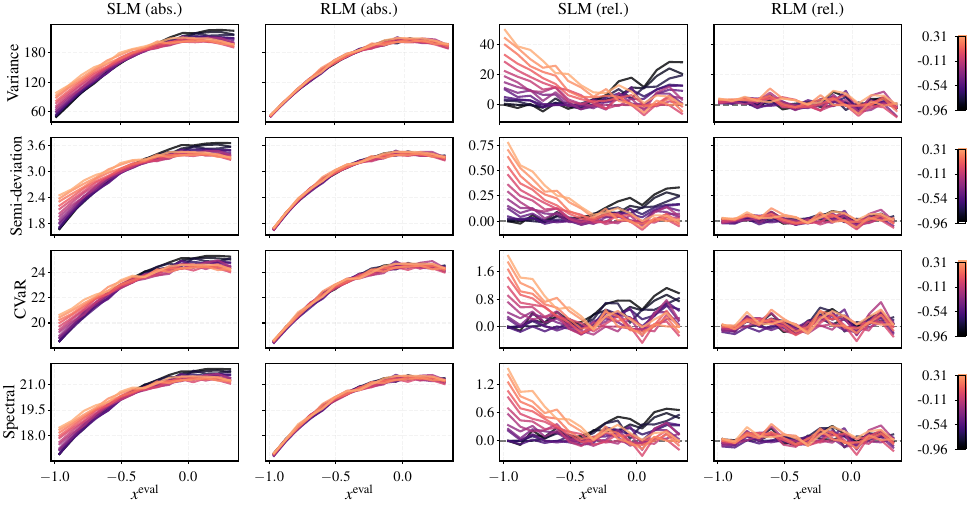}
	\caption{Continual robustness results for \textsc{Heston-Corr}.}
	\label{fig:app:synth:heston_metric_abs_rel_overview}
\end{figure}

\subsubsection{\textsc{BS-Cov}}

\begin{figure}[H]
	\centering
	\includegraphics[width=\textwidth]{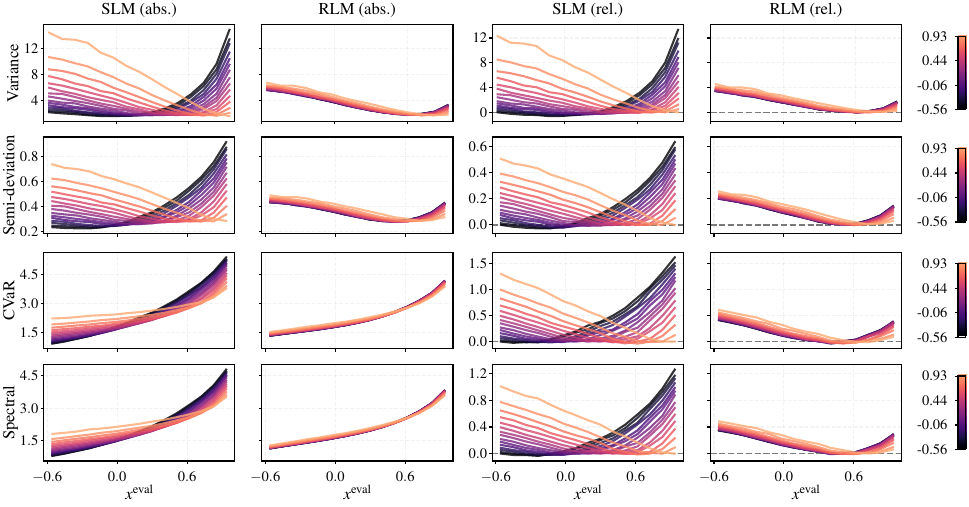}
	\caption{Continual robustness results for \textsc{BS-Cov}.}
	\label{fig:app:synth:bs2d_metric_abs_rel_overview}
\end{figure}

\subsubsection{In-simulator evaluation}

\begin{table}[H]
	\centering
	\caption{In-simulator evaluation for the three hedging problems.
	We compute $\mathcal{R}_\mathrm{srm}$ for the SLM-trained and RLM-trained policies
	on paths from the SLM simulator (SLM column) and from the RLM simulator
	(RLM column).
	The \emph{Gap} row reports SLM minus RLM, so
	\textcolor{myblue}{positive} values favor RLM
	and \textcolor{myred}{negative} values favor SLM.
	As a reference, the last two columns report points
	from the continual-robustness curves chosen such that
	the SLM--RLM gap is largest in either direction.
	\protect\footnotemark{}
	All entries are multiplied by $100$.}
	\small
\textsc{
\begin{tabular}{llcccccc}
\toprule
 & & \multicolumn{2}{c}{without warm-up} & \multicolumn{4}{c}{with warm-up} \\
\cmidrule(lr){3-4}\cmidrule(lr){5-8}
Problem & \diagbox[width=5.8em,height=2.1em]
{\raisebox{-0.15em}{\hspace{-0.5em}Policy}}
{\raisebox{0.15em}{Eval.}} & SLM & RLM & SLM & RLM & \shortstack{shift\\max gap} & \shortstack{shift\\min gap} \\
\midrule
BS-Vol & SLM & 10.34 & 10.17 & 10.31 & 10.19 & 16.42 & 5.42 \\
 & RLM & 10.35 & 10.17 & 10.33 & 10.15 & 14.94 & 5.93 \\
\cmidrule(lr){2-8}
 & \emph{Gap} & \textcolor{myred}{-0.01} & 0.00 & \textcolor{myred}{-0.02} & \textcolor{myblue}{0.04} & \textcolor{myblue}{1.48} & \textcolor{myred}{-0.51} \\
\midrule
Heston-Corr & SLM & 19.90 & 19.96 & 19.85 & 19.99 & 18.47 & 21.41 \\
 & RLM & 19.91 & 19.93 & 19.87 & 19.93 & 16.98 & 21.46 \\
\cmidrule(lr){2-8}
 & \emph{Gap} & \textcolor{myred}{-0.01} & \textcolor{myblue}{0.03} & \textcolor{myred}{-0.02} & \textcolor{myblue}{0.06} & \textcolor{myblue}{1.49} & \textcolor{myred}{-0.05} \\
\midrule
BS-Cov & SLM & 2.65 & 2.62 & 2.61 & 2.66 & 4.77 & 0.81 \\
 & RLM & 2.66 & 2.61 & 2.63 & 2.61 & 3.79 & 1.16 \\
\cmidrule(lr){2-8}
 & \emph{Gap} & \textcolor{myred}{-0.01} & \textcolor{myblue}{0.01} & \textcolor{myred}{-0.02} & \textcolor{myblue}{0.05} & \textcolor{myblue}{0.98} & \textcolor{myred}{-0.35} \\
\bottomrule
\end{tabular}
}

\end{table}

\footnotetext{For each regime-pair $(x^\mathrm{pre},x^\mathrm{eval})$
considered in the continual-robustness analysis,
we compute the SLM--RLM loss difference.
The ``max gap'' column reports the loss of each policy on the regime-pair
where this difference is largest,
and the ``min gap'' column uses the point where it is smallest.}

\clearpage
\section{Additional details and results for the real data case study}\label{app:sec:real_data}

\subsection{Experiment details}

\subsubsection{Payoffs}\label{app:subsec:real_data_payoffs}

In the real data experiments of \cref{sec:real_data}, we use the following payoffs.
All payoffs are defined on the normalized price path
$\tilde{S}_{1:T} := (\tilde{S}_1,\dots,\tilde{S}_T)$,
for which $\tilde{S}_0 = 1$.
All payoffs have a maturity of $T = 128$ trading days.
\begin{align*}
	\psi^{\mathrm{call}}(\tilde S_{1:T})
	&=
	(\tilde S_T - 1.0)^+, \\
	\psi^{\mathrm{put}}(\tilde S_{1:T})
	&=
	(1.0 - \tilde S_T)^+, \\
	\psi^{\mathrm{straddle}}(\tilde S_{1:T})
	&=
	\lvert \tilde S_T - 1.0 \rvert, \\
	\psi^{\mathrm{strangle}}(\tilde S_{1:T})
	&=
	(0.85 - \tilde S_T)^+
	+
	(\tilde S_T - 1.15)^+, \\
	\psi^{\mathrm{bull-call-spread}}(\tilde S_{1:T})
	&=
	(\tilde S_T - 0.95)^+
	-
	(\tilde S_T - 1.05)^+, \\
	\psi^{\mathrm{butterfly}}(\tilde S_{1:T})
	&=
	(\tilde S_T - 0.85)^+
	-
	2(\tilde S_T - 1.0)^+
	+
	(\tilde S_T - 1.15)^+, \\
	\psi^{\mathrm{digital-option}}(\tilde S_{1:T})
	&=
	\frac{
		(\tilde S_T - 1.09)^+
		-
		(\tilde S_T - 1.11)^+
	}{
		0.02
	}, \\
	\psi^{\mathrm{risk-reversal}}(\tilde S_{1:T})
	&=
	(0.95 - \tilde S_T)^+
	-
	(\tilde S_T - 1.05)^+, \\
	\psi^{\mathrm{up-and-out-call}}(\tilde S_{1:T})
	&=
	(\tilde S_T - 1.0)^+
	\cdot
	\mathbf{1}_{\{\max_{1\le t\le T}\tilde S_t < 1.25\}}, \\
	\psi^{\mathrm{down-and-out-put}}(\tilde S_{1:T})
	&=
	(1.0 - \tilde S_T)^+
	\cdot
	\mathbf{1}_{\{\min_{1\le t\le T}\tilde S_t > 0.75\}}.
\end{align*}

\subsubsection{Benchmark policies}
For each stock, let
$r_t := \log(S_t/S_{t-1})$
and $\Delta t := 1/250$.
In BS-EWMA, volatility is the EWMA estimator
\begin{equation*}
\hat v_t^{\mathrm{EWMA}}
=
\lambda \hat v_{t-1}^{\mathrm{EWMA}}
+
(1-\lambda)\frac{r_{t-1}^2}{\Delta t},
\qquad
\hat\sigma_t^{\mathrm{EWMA}}=\sqrt{\hat v_t^{\mathrm{EWMA}}},
\end{equation*}
with $\lambda=0.94$.
In BS-HIST, volatility is the equal-weight historical estimator
\begin{equation*}
\hat v_t^{\mathrm{HIST}}
=
\frac{1}{t - 1}\sum_{s=1}^{t - 1}\frac{r_s^2}{\Delta t},
\qquad
\hat\sigma_t^{\mathrm{HIST}}=\sqrt{\hat v_t^{\mathrm{HIST}}}.
\end{equation*}
These estimators are computed separately for each stock and then used as the
volatility input in $\Delta^{\psi}$.

\subsubsection{Randomization distribution}

\begin{wrapfigure}{r}{0.5\textwidth}
	\centering
	\includegraphics[width=\linewidth]{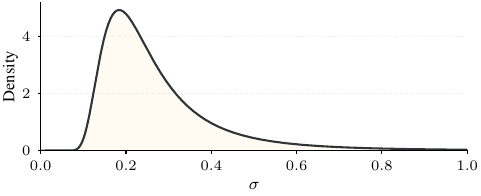}
	\caption{
		Density of the randomization distribution $\nu$
		used in the real data experiments.
	}
	\label{fig:real_data_sigma_prior}
\end{wrapfigure}

As described in \cref{sec:real_data}, we fit an inverse-gamma distribution
via maximum likelihood under the assumption that the sample paths of
$128$ days are generated under the static latent model.
The fitted parameters are
$\hat{\alpha}^\mathrm{IG} = 1.63$ and $\hat{\beta}^\mathrm{IG} = 0.07$.
The inverse-gamma distribution is the randomization distribution of the
variance, i.e., the squared latent volatility.
The resulting prior for the volatility is visualized
in \cref{fig:real_data_sigma_prior}.
The fitted distribution has a heavy upper tail.
For training, we therefore cap sampled volatility values at $1.0$,
which avoids excessive sensitivity to extremely high-volatility draws.

\clearpage 

\subsection{Additional results}

\begin{figure}[h]
	\centering
	\includegraphics[width=\textwidth]{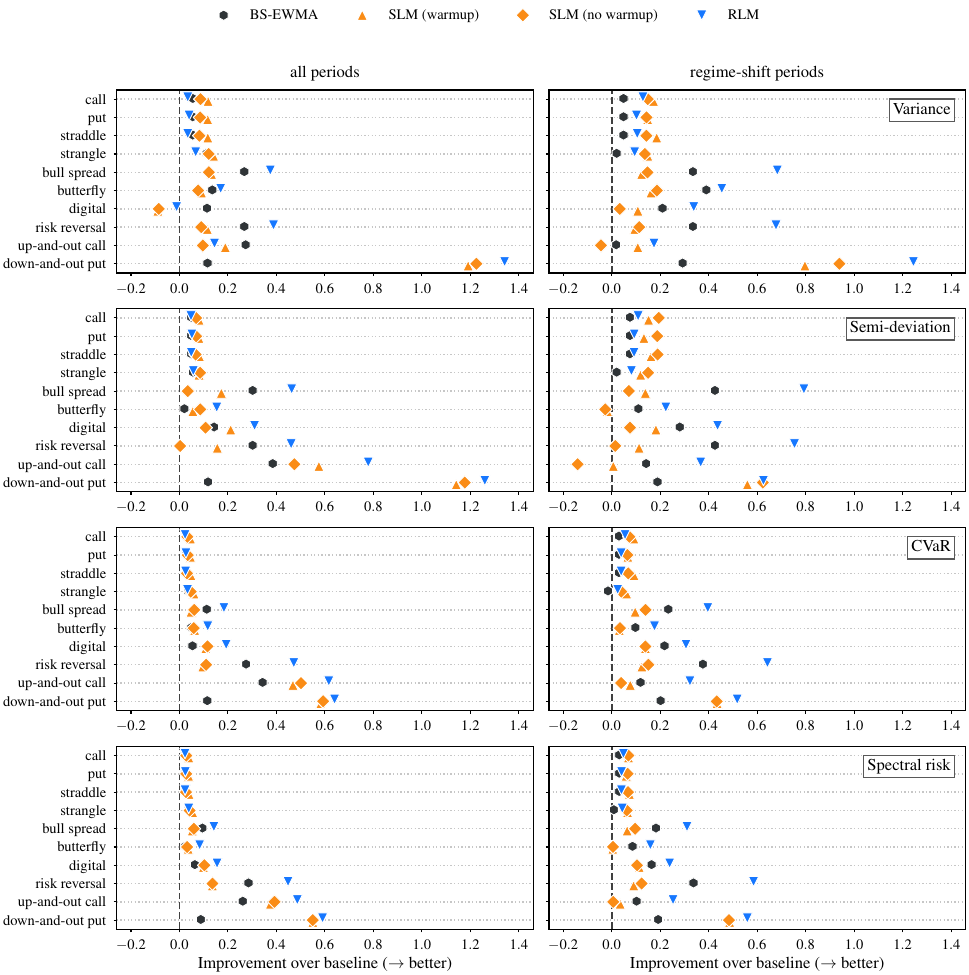}
	\caption{
		Additional results for the pooled S\&P 100 real data backtest.
		Each point reports, for one payoff and one policy,
		the log improvement over BS-HIST,
		$\log(
		\mathcal{R}^{\mathrm{BS\text{-}HIST}} /
		\mathcal{R}
		)$,
		so positive values indicate lower risk than BS-HIST.
		The first column shows results on all evaluation windows;
		the second shows results on regime-shift periods.
		Each row corresponds to a different evaluation criterion.
		In addition to the policies discussed in the main text, we show
		results for the SLM policy when no warm-up is used.
		For the SLM policy, warm-up makes the policy more adaptive to the recent
		volatility regime, but also causes ambiguity about the static latent volatility
		to decay before trading starts.
		The results are mixed on whether warm-up improves the SLM policy.
		As discussed in \cref{subsec:real_data_results},
		continual retraining with an updated randomization distribution may be better,
		but would be much more expensive.
	}
\end{figure}

\clearpage
\section{Proofs}\label{app:sec:proofs}

\subsection{Vanishing and stationary ambiguity}\label{app:subsec:main_proofs}

\subsubsection{Doob's theorem}

We now state a more detailed version of \cref{prop:slm_vanishing_amb}.
This is a version of Doob's well-known posterior consistency theorem;
see, for example, \citet{miller_2018_aa}.

\begin{proposition*}[Doob's theorem]
	Let $\mathcal X$ and $\mathcal Y$ be Polish spaces, equipped with their
	Borel $\sigma$-algebras. Let $\nu$ be a probability measure on
	$\mathcal X$, and let $p(\cdot\mid x)$ be a probability measure on
	$\mathcal Y$ for each $x\in\mathcal X$. Assume that
	$x\mapsto p(A\mid x)$ is measurable for every Borel set
	$A\subseteq\mathcal Y$. Consider the static latent model
	\begin{equation*}
		X\sim \nu,
		\qquad
		Y_1,Y_2,\ldots\mid X=x
		\stackrel{\mathrm{iid}}{\sim}p(\cdot\mid x),
	\end{equation*}
	and let $\mathbb P$ denote the induced joint law of
	$(X,Y_1,Y_2,\ldots)$ on $\mathcal X\times\mathcal Y^{\mathbb N}$.
	Let $\mathcal G_t=\sigma(Y_1,\ldots,Y_t)$. If the model is
	identifiable, in the sense that
	\begin{equation*}
		p(\cdot\mid x)\neq p(\cdot\mid x')
		\qquad
		\text{whenever } x\neq x',
	\end{equation*}
	then, $\mathbb P$-almost surely,
	\begin{equation*}
		\mathbb P(X\in\cdot\mid\mathcal G_t)
		\Rightarrow
		\delta_X ,
	\end{equation*}
	where $\Rightarrow$ denotes weak convergence.
	Moreover, for every measurable $\varphi:\mathcal X\to\mathbb R$
	with $\varphi(X)\in L^2$,
	\begin{equation*}
		\Var(\varphi(X)\mid\mathcal G_t)\to 0,
		\qquad
		\mathbb P\text{-a.s.}
	\end{equation*}
\end{proposition*}

\begin{proof}
The weak convergence statement follows from
\citet[Theorem~2.4]{miller_2018_aa}, which gives
posterior concentration on every open neighborhood of the true
parameter, for $\nu$-almost every true parameter. Equivalently, under
the joint law $\mathbb P$, the posterior converges weakly to
$\delta_X$ (the Portmanteau theorem).
The variance convergence follows from \citet[Theorem~2.2]{miller_2018_aa}, applied
to $\varphi$ and $\varphi^2$.
\end{proof}

\subsubsection{Filter stationarity}

We state a more general version of \cref{prop:amb_stationary}
and provide a proof for completeness.
The result is well-known
\citep{kunita_1971_asymp,vanhandel_2012_on}.
A much more difficult question is under what conditions the filter
has a unique stationary distribution
\citep[see e.g.,][]{vanhandel_2012_on}.

\begin{proposition*}[Sufficient conditions for a stationary filter]
	Let $\mathcal X$ and $\mathcal Y$ be Polish spaces, and let
	$\Omega = (\mathcal X \times \mathcal Y)^{\mathbb Z}$.
	Let $(X,Y)$ be the canonical process on the probability space
	$(\Omega,\mathcal B(\Omega),\mathbb P)$, and let
	$\vartheta:\Omega\to\Omega$ be the left shift,
	\begin{equation*}
			(\vartheta \omega)_s = \omega_{s+1}.
	\end{equation*}
	Let $\mathbb F = (\mathcal F_t)_{t\in\mathbb Z}$ be a sequence of
	$\sigma$-algebras with $\mathcal F_t\subset\mathcal B(\Omega)$ for each $t$.
	Assume that
	\begin{enumerate}[label=(\alph*)]
			\item \label{item:amb_stationary_joint_process}
					the joint process $(X,Y)$ is strictly stationary;
			\item \label{item:amb_stationary_information_structure}
					$\mathbb{F}$ satisfies
					\begin{equation*}
							\mathcal F_t
							=
							\vartheta^{-t}\mathcal F_0
							:=
							\{(\vartheta^t)^{-1}(A):A\in\mathcal F_0\},
							\qquad t\in\mathbb Z.
					\end{equation*}
	\end{enumerate}
	For $t\in\mathbb Z$, let
	\begin{equation*}
			\pi_t := \mathbb P(X_t\in\cdot\mid\mathcal F_t).
	\end{equation*}
	Then $(\pi_t)_{t\in\mathbb Z}$ has a strictly stationary version.
	Moreover, for any measurable $\varphi:\mathcal X\to\mathbb R$
	with $\varphi(X_0)\in L^2$, the conditional variance process
	\begin{equation*}
			\mathcal V_t^\varphi
			:=
			\Var(\varphi(X_t)\mid\mathcal F_t),
			\qquad t\in\mathbb Z,
	\end{equation*}
	has a strictly stationary version.
\end{proposition*}

Condition (a) ensures that, in distribution,
the statistical problem of estimating the current latent state from the
current information is the same at every time $t$.
Condition (b) ensures that
the information available at time $t$ on a path~$\omega$ coincides with the
information available at time $0$ on the shifted path $\vartheta^t \omega$.
The condition is satisfied when the controller receives
information according to the same rule at every date.
This includes the following two cases:
  \begin{enumerate}[label=(\roman*)]
        \item \label{item:amb_stationary_lookback}
              fixed lookback window:
              $\mathcal F_t = \mathcal G_t^k := \sigma(Y_t,\dots,Y_{t-k+1})$
              for $t \in \mathbb Z$ and a fixed $k \in \mathbb{N}$;
        \item \label{item:amb_stationary_full_history}
              full observation history:
              $\mathcal F_t = \mathcal G_t := \sigma(Y_s : s \le t)$
              for $t \in \mathbb Z$.
  \end{enumerate}
We focus on the second case because information structures as in~(i)
enforce stationary ambiguity by artificially introducing ambiguity over past
values of the process $Y$.
In reality, the entire process $Y$ is progressively observable.
We discuss why such artificially introduced ambiguity is undesirable in
\cref{subsec:info_restrictions}.

\begin{proof}
	Choose a version $\pi_0$ of
	$\mathbb P(X_0 \in \cdot \mid \mathcal F_0)$.
	For each $t \in \mathbb Z$, define
	\begin{equation*}
		\tilde{\pi}_t(\omega) := \pi_0(\vartheta^t \omega),
		\qquad \omega \in \Omega.
	\end{equation*}
	Since $\pi_0$ is $\mathcal F_0$-measurable and
	$\mathcal F_t = \vartheta^{-t}\mathcal F_0$,
	the random measure $\tilde{\pi}_t$ is $\mathcal F_t$-measurable.
	Fix $t \in \mathbb Z$.
	Let $\phi : \mathcal X \to \mathbb R$ be bounded and measurable, and let
	$A \in \mathcal F_t$.
	Write $A = \vartheta^{-t}A_0$ for some $A_0 \in \mathcal F_0$.
	\begin{align*}
		\mathbb E\bigl[
			1_A \int_{\mathcal X} \phi(x)\,\tilde{\pi}_t(dx)
			\bigr]
		 & =
		\mathbb E\bigl[
			\bigl(
			1_{A_0}
			\int_{\mathcal X} \phi(x)\,\pi_0(dx)
			\bigr)
			\circ \vartheta^t
		\bigr]                                          \\
		 & =
		\mathbb E\bigl[
			1_{A_0}
			\int_{\mathcal X} \phi(x)\,\pi_0(dx)
		\bigr]                                          \\
		 & =
		\mathbb E[1_{A_0} \phi(X_0)]                     \\
		 & =
		\mathbb E[(1_{A_0} \phi(X_0)) \circ \vartheta^t] \\
		 & =
		\mathbb E[1_A \phi(X_t)].
	\end{align*}
	Here we used that
	$\mathbb P \circ \vartheta^{-t} = \mathbb P$
	and
	$X_t = X_0 \circ \vartheta^t$ on the canonical path space.
	Hence $\tilde{\pi}_t$ is a version of
	$\mathbb P(X_t \in \cdot \mid \mathcal F_t)$.

	Moreover, for any $h,t_1,\dots,t_n \in \mathbb Z$,
	\begin{equation*}
			(\tilde{\pi}_{t_1+h},\dots,\tilde{\pi}_{t_n+h})
			=
			(\tilde{\pi}_{t_1},\dots,\tilde{\pi}_{t_n})
			\circ \vartheta^h .
	\end{equation*}
	Since $\mathbb P$ is invariant under $\vartheta$,
	$(\tilde{\pi}_t)_{t\in\mathbb Z}$ is strictly stationary.
	Thus $(\tilde{\pi}_t)_{t\in\mathbb Z}$ is a strictly stationary
	version of the filter process.

	Finally, define
	\begin{equation*}
			\tilde{\mathcal V}_t^\varphi
			=
			\int_{\mathcal X}\varphi(x)^2\,\tilde{\pi}_t(dx)
			-
			\Bigl(
					\int_{\mathcal X}\varphi(x)\,\tilde{\pi}_t(dx)
			\Bigr)^2 .
	\end{equation*}
	This is a version of
	$\Var(\varphi(X_t)\mid\mathcal F_t)$.
	Since $\tilde{\mathcal V}_t^\varphi$ is a measurable function of
	$\tilde{\pi}_t$, the process
	$(\tilde{\mathcal V}_t^\varphi)_{t\in\mathbb Z}$
	is strictly stationary.
\end{proof}

\subsubsection{Stationary randomization}

\begin{proof}[Proof of \cref{prop:stationary_randomization}]

This result is due to \citet{stenflo_2001_marko}, who studies iterations in
which the map applied at each step is drawn along a stationary sequence.
Our recursion is of this form, with the latent process $X$ as the stationary
driving sequence.
Since the setup and notation in that work differ substantially from ours,
we give the argument in detail.

Recall that the randomized simulator is
\begin{equation}\label{eq:app_stenflo_recursion}
	Y_t = F(X_t,Y_{t-1},W_t),
	\qquad t\in\mathbb Z,
\end{equation}
where $X$ is stationary, the noise $W$ is i.i.d., and $X$ and $W$ are
independent.
By assumption~(ii), $d$ is a complete and separable metric on $\mathcal Y$, and
there exist $c\in(0,1)$ and $y_\ast\in\mathcal Y$ such that
\begin{align}
	\mathbb E\big[d(F(x,y,W_0),F(x,\tilde y,W_0))\big]
		&\le c\,d(y,\tilde y),
		\qquad x\in\mathcal X,\ y,\tilde y\in\mathcal Y,
		\label{eq:app_stenflo_contraction}\\
	\sup_{x\in\mathcal X}\mathbb E\big[d(y_\ast,F(x,y_\ast,W_0))\big]
		&< \infty .
		\label{eq:app_stenflo_moment}
\end{align}
It is convenient to write the recursion in terms of the random map applied at
each step.
For $t\in\mathbb Z$, define
\begin{equation*}
	\Phi_t(y) := F(X_t,y,W_t),
	\qquad y\in\mathcal Y,
\end{equation*}
so that \eqref{eq:app_stenflo_recursion} becomes $Y_t=\Phi_t(Y_{t-1})$.
Each $\Phi_t$ is a random map on $\mathcal Y$, and the sequence
$(\Phi_t)_{t\in\mathbb Z}$ is driven by the stationary process $(X,W)$.
\citet[Theorem~1]{stenflo_2001_marko} shows that the
backward composition of the most recent maps,
\begin{equation*}
	G_t^{(n)} := \Phi_t\circ\Phi_{t-1}\circ\cdots\circ\Phi_{t-n+1}(y_\ast),
\end{equation*}
converges, as $n\to\infty$, almost surely and in $L^1(d)$ to a limit $Z_t$ that
does not depend on the base point $y_\ast$.\footnote{
Note that this construction uses only stationarity of the driving sequence; its
ergodicity is not needed here.}
Note that $Z_t$ is obtained from the history $(X_s, W_s)_{s \le t}$ of the driving
process by the same measurable map for every $t$.
Since $(X, W)$ is stationary,
the joint process $(X, Z)$ is therefore also stationary.

Finally, note that $Z=(Z_t)_{t\in\mathbb Z}$ solves the recursion, that is,
$Z_t=\Phi_t(Z_{t-1})$ almost surely for every $t$.
Let ${\mathcal H:=\sigma\big(X,(W_s)_{s\le t-1}\big)}$.
The variables $G_{t-1}^{(n-1)}$, $Z_{t-1}$, and $X_t$ are $\mathcal H$-measurable,
while $W_t$ is independent of $\mathcal H$.
Hence, by~\eqref{eq:app_stenflo_contraction},
\begin{equation*}
      \mathbb E\big[d\big(\Phi_t(G_{t-1}^{(n-1)}),\Phi_t(Z_{t-1})\big)\mid\mathcal H\big]
      \le c\,d\big(G_{t-1}^{(n-1)},Z_{t-1}\big).
\end{equation*}
Taking expectations and using that $G_{t-1}^{(n-1)}\to Z_{t-1}$ in $L^1(d)$,
\begin{equation*}
      \mathbb E\big[d\big(\Phi_t(G_{t-1}^{(n-1)}),\Phi_t(Z_{t-1})\big)\big]
      \le c\,\mathbb E\big[d\big(G_{t-1}^{(n-1)},Z_{t-1}\big)\big]\to 0 .
\end{equation*}
By the triangle inequality, for every $n$,
\begin{equation*}
      \mathbb E\big[d\big(Z_t,\Phi_t(Z_{t-1})\big)\big]
      \le \mathbb E\big[d\big(Z_t,G_t^{(n)}\big)\big]
      + \mathbb E\big[d\big(\Phi_t(G_{t-1}^{(n-1)}),\Phi_t(Z_{t-1})\big)\big],
\end{equation*}
and both terms vanish as $n\to\infty$.
As the left side does not depend on $n$, it is zero, so
${Z_t=\Phi_t(Z_{t-1})}$ almost surely.

Finally, suppose $X$ is in addition ergodic.
Then, $(X,W)$ is ergodic and since $(X,Z)$ is a measurable factor of $(X,W)$,
it is ergodic as well.
\end{proof}

\subsection{Optimal investment}\label{app:sec:optimal_investment_example}

In this section, we provide all the results for the optimal investment problems
from \cref{subsec:opt_inv_example}, which are of the type
\begin{equation}\label{eq:app_opt_inv_example_objective}
	\inf_{U \; \mathbb G\text{-predictable}}
	\quad
	-
	\sum_{t=1}^T
	\mathbb E\bigl[
		u(U_tY_t)
	\bigr],
	\qquad
	u(r) := 1-\exp(-\lambda r),
\end{equation}
where $\lambda > 0$ and
$\mathbb G$ is the filtration generated by $Y$.
Equivalently, one may drop the additive constant and minimize
$\sum_{t=1}^T \mathbb E[\exp(-\lambda U_tY_t)]$.
The numerical parameters used in the figures are listed in
\cref{tab:app_opt_inv_parameters}.

\begin{table}[t]
	\centering
	\small
	\caption{Parameter values for the optimal investment experiments
	and visualizations.}
	\label{tab:app_opt_inv_parameters}
	\begin{tabular}{ll}
	\toprule
	Parameter & Value \\
	\midrule
	Trading days per year & $252$ \\
	Horizon & $T=10 \cdot 252=2520$ \\
	Prior mean drift & $\bar x=0.10/252$ \\
	Prior drift standard deviation & $\tau=0.20/252$ \\
	Conditional return standard deviation & $\sigma=0.15/\sqrt{252}$ \\
	Risk aversion & $\lambda=200$ \\
	LG-SSM persistence & $\phi=0.9925$ \\
	Regime-switch upper drift & $x_+=\bar x+\tau\Phi^{-1}(0.90)\approx 0.36/252$ \\
	Regime-switch lower drift & $x_-=\bar x+\tau\Phi^{-1}(0.10)\approx -0.16/252$ \\
	Regime-switch time & $t_\ast=T/2=1260$ \\
	\bottomrule
	\end{tabular}
\end{table}

\subsubsection{Known drift}
We begin with the base simulator in which the drift is known and constant.
Since the drift is known, the policy does not learn about it from its observations
and the optimal investment amount is constant over time.

\begin{proposition}[Known drift]
\label{prop:app_opt_inv_example_known}
	Suppose that $\bar x \in \mathbb R$, $\sigma^2 > 0$, and
	$Y_t \stackrel{\mathrm{iid}}{\sim} N(\bar x,\sigma^2)$
	for $t = 1,\dots,T$.
	Then the unique optimizer of \cref{eq:app_opt_inv_example_objective}
	is the constant control
	\begin{equation*}
		U_t^\star = \frac{\bar x}{\lambda \sigma^2},
		\qquad t = 1,\dots,T .
	\end{equation*}
\end{proposition}
\begin{proof}
For each $t$, the $t$-th summand in the objective depends only on $U_t$.
Since $U_t$ is $\mathcal G_{t-1}$-measurable and $Y_t$ is independent of $\mathcal G_{t-1}$ with law $N(\bar x,\sigma^2)$,
\begin{equation*}
	\mathbb E\bigl[
		\exp(-\lambda U_t Y_t)
		\,\mid \,
		\mathcal G_{t-1}
	\bigr]
	=
	\exp\bigl(
		-\lambda \bar x U_t + \frac{1}{2}\lambda^2 \sigma^2 U_t^2
	\bigr).
\end{equation*}
The right-hand side is strictly convex in $U_t$, hence uniquely minimized at
$U_t^\star = \bar x / (\lambda\sigma^2)$.
Since the objective is a sum of these coordinate-wise terms, this choice is optimal for every $t$, and uniqueness follows from strict convexity of each summand.
\end{proof}

\subsubsection{Statically randomized drift}

We next consider statically randomizing the drift of the base simulator.
Here, the optimal control is stochastic as it depends on the posterior
distribution over the realization of the random drift.

\begin{proposition}[Optimal investment under static randomization]
\label{prop:app_opt_inv_example_slm}
	Suppose that
	\begin{equation*}
		X_1 \sim N(\bar x,\tau^2),
		\qquad
		Y_t \mid X_1 \stackrel{\mathrm{iid}}{\sim} N(X_1,\sigma^2),
		\qquad t=1,\dots,T,
	\end{equation*}
	with $\sigma^2,\tau^2>0$.
	Let
	\begin{equation*}
		m_t := \mathbb E[X_1 \mid \mathcal G_{t-1}],
		\qquad
		r_t := \Var(X_1 \mid \mathcal G_{t-1})
		=
		\frac{\tau^2\sigma^2}{\sigma^2+(t-1)\tau^2}.
	\end{equation*}
	Then the unique optimal control in
	\cref{eq:app_opt_inv_example_objective} is
	\begin{equation}\label{eq:app_opt_inv_example_slm_control}
		U_t^\star
		=
		\frac{m_t}{\lambda(\sigma^2+r_t)} .
	\end{equation}
	Moreover,
	\begin{equation*}
		U_t^\star
		\sim
		N\Bigl(
			\frac{\bar x}{\lambda(\sigma^2+r_t)},
			\frac{\tau^2-r_t}{\lambda^2(\sigma^2+r_t)^2}
		\Bigr).
	\end{equation*}
\end{proposition}

\begin{proof}
	Standard Gaussian conjugacy, see for instance \citet{murphy_2007_conju},
	gives
	\begin{equation*}
		X_1 \mid \mathcal G_{t-1} \sim N(m_t,r_t) ,
	\end{equation*}
	and the explicit, deterministic expression for $r_t$
	stated above.
	Hence
	\begin{equation*}
		Y_t \mid \mathcal G_{t-1}
		\sim
		N(m_t,\sigma^2+r_t).
	\end{equation*}
	Therefore, for any $U_t \in L^0(\mathcal G_{t-1})$,
	\begin{equation*}
		\mathbb E\bigl[
			\exp(-\lambda U_t Y_t)
			\,\mid \,
			\mathcal G_{t-1}
		\bigr]
		=
		\exp\bigl(
			-\lambda m_t U_t
			+
			\frac{1}{2}\lambda^2(\sigma^2+r_t)U_t^2
		\bigr).
	\end{equation*}
	The right-hand side is strictly convex in $U_t$, and is uniquely minimized at
	\begin{equation*}
		U_t^\star
		=
		\frac{m_t}{\lambda(\sigma^2+r_t)}.
	\end{equation*}
	Since the objective in \cref{eq:app_opt_inv_example_objective} is a sum of such terms, this proves the formula for the optimal control.

	Finally, $(X_1,Y_1,\dots,Y_{t-1})$ is jointly Gaussian, so $m_t$ is Gaussian.
	Moreover,
	\begin{equation*}
		\mathbb E[m_t]=\mathbb E[X_1]=\bar x,
	\end{equation*}
	and by the law of total variance,
	\begin{equation*}
		\tau^2
		=
		\mathbb E\bigl[\Var(X_1\mid \mathcal G_{t-1})\bigr]
		+
		\Var\bigl(\mathbb E[X_1\mid \mathcal G_{t-1}]\bigr)
		=
		r_t+\Var(m_t).
	\end{equation*}
	Thus
	\begin{equation*}
		m_t \sim N(\bar x,\tau^2-r_t),
	\end{equation*}
	and the stated law of $U_t^\star$ follows.
\end{proof}

\subsubsection{Dynamically randomized drift}

We next randomize the drift dynamically via a stationary autoregressive process of order $1$.
This results in a scalar linear Gaussian state-space model (LG-SSM),
so the filtering equations are given by the Kalman filter
\citep[see e.g.,][]{durbin_2012_time}.
Throughout this subsection, we work on the bi-infinite time axis and write
\begin{equation*}
	\mathcal G_t := \sigma(Y_s : s \le t),
	\qquad t \in \mathbb Z.
\end{equation*}
We consider the stationary version of the model, so that the Kalman filter is in steady state and the law of the optimal control is time-homogeneous.

\begin{proposition}[Optimal investment under dynamic randomization]
\label{prop:app_opt_inv_example_lg}
	Suppose that $|\phi|<1$ and $0<\tau^2<\sigma^2$, and consider
	\begin{equation}\label{eq:app_opt_inv_example_lgssm}
		X_t = \bar x + \phi(X_{t-1}-\bar x) + \varepsilon_t,
		\qquad
		Y_t = X_t + \eta_t,
		\qquad t \in \mathbb Z,
	\end{equation}
	with independent centered Gaussian noises satisfying
	\begin{equation*}
		\Var(\eta_t)=\sigma^2,
		\qquad
		\Var(\varepsilon_t)
		=
		\tau^2
		\bigl(
			\frac{\sigma^2}{\sigma^2-\tau^2}-\phi^2
		\bigr).
	\end{equation*}
	Let
	\begin{equation*}
		m_t := \mathbb E[X_t \mid \mathcal G_{t-1}],
		\qquad
		s_\star := \sigma^2 + \frac{\tau^2\sigma^2}{\sigma^2-\tau^2}.
	\end{equation*}
	Then the unique optimal control in
	\cref{eq:app_opt_inv_example_objective} is
	\begin{equation}\label{eq:app_opt_inv_example_lg_control}
		U_t^\star
		=
		\frac{m_t}{\lambda s_\star}.
	\end{equation}
	Moreover,
	\begin{equation*}
		U_t^\star
		\sim
		N\Bigl(
			\frac{\bar x}{\lambda s_\star},
			\frac{\phi^2\tau^4}
			     {\lambda^2 s_\star^2(\sigma^2-\tau^2)(1-\phi^2)}
		\Bigr).
	\end{equation*}
\end{proposition}

\begin{proof}
	For the scalar linear Gaussian state-space model
	\eqref{eq:app_opt_inv_example_lgssm}, the standard steady-state Kalman filter formulas yield
	\begin{equation*}
		X_t \mid \mathcal G_{t-1}
		\sim
		N(m_t,p_\star),
		\qquad
		p_\star
		=
		\frac{\tau^2\sigma^2}{\sigma^2-\tau^2},
	\end{equation*}
	and hence
	\begin{equation*}
		Y_t \mid \mathcal G_{t-1}
		\sim
		N(m_t,s_\star).
	\end{equation*}
	Therefore, for any $U_t \in L^0(\mathcal G_{t-1})$,
	\begin{equation*}
		\mathbb E\bigl[
			\exp(-\lambda U_t Y_t)
			\, \mid \,
			\mathcal G_{t-1}
		\bigr]
		=
		\exp\bigl(
			-\lambda m_t U_t
			+
			\frac{1}{2}\lambda^2 s_\star U_t^2
		\bigr).
	\end{equation*}
	This is strictly convex in $U_t$, and is uniquely minimized at
	\begin{equation*}
		U_t^\star
		=
		\frac{m_t}{\lambda s_\star}.
	\end{equation*}
	Since the objective in \cref{eq:app_opt_inv_example_objective} is a sum of such terms, this proves \eqref{eq:app_opt_inv_example_lg_control}.

	Define
	\begin{equation*}
		k_\star := \frac{p_\star}{p_\star+\sigma^2} = \frac{\tau^2}{\sigma^2},
		\qquad
		\xi_t := Y_t - m_t.
	\end{equation*}
	Again by the steady-state Kalman filter equations,
	$(\xi_t)_{t \in \mathbb Z}$ is i.i.d.\ $N(0,s_\star)$ and
	\begin{equation*}
		m_{t+1}-\bar x
		=
		\phi(m_t-\bar x) + \phi k_\star \xi_t.
	\end{equation*}
	Thus $(m_t)_{t \in \mathbb Z}$ is a stationary Gaussian AR$(1)$ process with mean $\bar x$ and variance
	\begin{equation*}
		\Var(m_t)
		=
		\frac{\phi^2 k_\star^2 s_\star}{1-\phi^2}
		=
		\frac{\phi^2\tau^4}
		     {(\sigma^2-\tau^2)(1-\phi^2)}.
	\end{equation*}
	Hence
	\begin{equation*}
		m_t
		\sim
		N\bigl(
			\bar x,
			\frac{\phi^2\tau^4}
			     {(\sigma^2-\tau^2)(1-\phi^2)}
		\bigr),
	\end{equation*}
	and the stated law of $U_t^\star$ follows from
	\eqref{eq:app_opt_inv_example_lg_control}.
\end{proof}

\subsubsection{Worst-case drift}
\label{app:worst_case_regret_investment}

We use the fixed-drift model to contrast worst-case value with worst-case regret.
Throughout, $\mathbb P_x$ denotes the law under
$Y_t \stackrel{\mathrm{iid}}{\sim} N(x,\sigma^2)$,
where $x\in[\underline x,\overline x]$ is fixed throughout the episode, and
$\mathbb E_x$ denotes the corresponding expectation.

\begin{proposition}[Worst-case value]
\label{prop:app_opt_inv_worst_case_value}
Let
\begin{equation*}
		x_\dagger
		\in
		\operatorname*{argmin}_{x\in[\underline x,\overline x]} |x| .
\end{equation*}
Then the worst-case value problem
\begin{equation*}
		\inf_{U \; \mathbb G\text{-predictable}}
		\sup_{x\in[\underline x,\overline x]}
		\sum_{t=1}^T
		\mathbb E_x[\exp(-\lambda U_tY_t)]
\end{equation*}
is solved by
\begin{equation*}
		U_t^\star
		=
		\frac{x_\dagger}{\lambda\sigma^2},
		\qquad t=1,\dots,T .
\end{equation*}
\end{proposition}

\begin{proof}
Evaluating the supremum at $x_\dagger$ and applying
\cref{prop:app_opt_inv_example_known} gives, for any policy $U$,
\begin{equation*}
		\sup_{x\in[\underline x,\overline x]}
		\sum_{t=1}^T
		\mathbb E_x[\exp(-\lambda U_tY_t)]
		\ge
		T\exp\bigl(-\frac{x_\dagger^2}{2\sigma^2}\bigr).
\end{equation*}
For $U_t^\star=x_\dagger/(\lambda\sigma^2)$,
\begin{equation*}
		\mathbb E_x[\exp(-\lambda U_t^\star Y_t)]
		=
		\exp\bigl(
				\frac{x_\dagger^2}{2\sigma^2}
				-
				\frac{x_\dagger x}{\sigma^2}
		\bigr).
\end{equation*}
The exponent is maximized over $[\underline x,\overline x]$ at $x=x_\dagger$
(or is constant when $x_\dagger=0$).
Thus $U^\star$ attains the lower bound.
\end{proof}

Define the known-drift action and fixed-$x$ utility by
\begin{equation*}
		a^\star(x)
		:=
		\frac{x}{\lambda\sigma^2},
		\qquad
		j_x(u)
		:=
		\mathbb E_x[1-\exp(-\lambda uY_t)] .
\end{equation*}
For a policy $U$, define its fixed-$x$ regret by
\begin{equation*}
		R_T(U,x)
		:=
		\sum_{t=1}^T
		\mathbb E_x\bigl[
				j_x(a^\star(x))-j_x(U_t)
		\bigr].
\end{equation*}

\begin{proposition}[Worst-case regret]
\label{prop:app_opt_inv_worst_case_regret}
Suppose $U^{(T)}$ solves
\begin{equation*}
		\inf_{U \; \mathbb G\text{-predictable}}
		\sup_{x\in[\underline x,\overline x]} R_T(U,x).
\end{equation*}
Then, for every $x\in[\underline x,\overline x]$ and every $\epsilon>0$,
\begin{equation*}
		\frac{1}{T}
		\sum_{t=1}^T
		\mathbb P_x\bigl(
				|U_t^{(T)}-a^\star(x)|>\epsilon
		\bigr)
		\longrightarrow 0 .
\end{equation*}
\end{proposition}

\begin{proof}
First, the minimax-regret value is at most logarithmic in $T$.
For $t\ge2$, let
\begin{equation*}
		\widehat x_{t-1}
		=
		\Pi_{[\underline x,\overline x]}
		\bigl(
				\frac{1}{t-1}\sum_{s=1}^{t-1}Y_s
		\bigr),
		\qquad
		U_t
		=
		\frac{\widehat x_{t-1}}{\lambda\sigma^2},
\end{equation*}
with any fixed $U_1$.
Projection onto $[\underline x,\overline x]$ gives
\begin{equation*}
		\mathbb E_x[(\widehat x_{t-1}-x)^2]
		\le
		\frac{\sigma^2}{t-1}.
\end{equation*}
Moreover, the one-period regret is
\begin{equation*}
		j_x(a^\star(x))-j_x(u)
		=
		\exp\bigl(-\frac{x^2}{2\sigma^2}\bigr)
		\bigl[
				\exp\bigl(
						\frac{(\lambda\sigma^2u-x)^2}{2\sigma^2}
				\bigr)
				-1
		\bigr].
\end{equation*}
Taking $u=z/(\lambda\sigma^2)$, this becomes
\begin{equation*}
		j_x(a^\star(x))
		-
		j_x\bigl(\frac{z}{\lambda\sigma^2}\bigr)
		=
		\exp\bigl(-\frac{x^2}{2\sigma^2}\bigr)
		\bigl[
				\exp\bigl(\frac{(z-x)^2}{2\sigma^2}\bigr)
				-1
		\bigr].
\end{equation*}
Since $x,z\in[\underline x,\overline x]$, the quantity $(z-x)^2/(2\sigma^2)$ is bounded,
so that $e^a-1\le C_0a$ on this bounded interval for some $C_0<\infty$.
Therefore, with $C:=C_0/(2\sigma^2)$,
\begin{equation*}
		j_x(a^\star(x))
		-
		j_x\bigl(\frac{z}{\lambda\sigma^2}\bigr)
		\le
		C(z-x)^2
\end{equation*}
for all $x,z\in[\underline x,\overline x]$.
Applying this with $z=\widehat x_{t-1}$ gives, for $t\ge2$,
\begin{equation*}
		\mathbb E_x\bigl[
				j_x(a^\star(x))-j_x(U_t)
		\bigr]
		\le
		C\mathbb E_x[(\widehat x_{t-1}-x)^2]
		\le
		\frac{C\sigma^2}{t-1}.
\end{equation*}
The first-period regret is bounded uniformly in $x\in[\underline x,\overline x]$ because
the interval is compact and $U_1$ is fixed.
Summing over time therefore gives, for some constant $C^\prime<\infty$,
\begin{equation*}
		\sup_{x\in[\underline x,\overline x]} R_T(U,x)
		\le
		C^\prime(1+\log T).
\end{equation*}
Since $U^{(T)}$ is minimax optimal, it satisfies the same bound:
\begin{equation*}
		\sup_{x\in[\underline x,\overline x]} R_T(U^{(T)},x)
		\le
		C^\prime(1+\log T).
\end{equation*}

Now fix $x\in[\underline x,\overline x]$ and $\epsilon>0$.
If $|u-a^\star(x)|>\epsilon$, then
$|\lambda\sigma^2u-x|>\lambda\sigma^2\epsilon$.
Since $[\underline x,\overline x]$ is compact, the exact regret formula gives a constant
$c_\epsilon>0$ such that
\begin{equation*}
		j_x(a^\star(x))-j_x(u)
		\ge
		c_\epsilon
		\mathbf 1\{|u-a^\star(x)|>\epsilon\}
\end{equation*}
for all $x\in[\underline x,\overline x]$ and $u\in\mathbb R$.
Therefore
\begin{equation*}
		c_\epsilon
		\sum_{t=1}^T
		\mathbb P_x\bigl(
				|U_t^{(T)}-a^\star(x)|>\epsilon
		\bigr)
		\le
		R_T(U^{(T)},x)
		\le
		C^\prime(1+\log T).
\end{equation*}
Dividing by $T$ proves the claim.
\end{proof}

\paragraph{Numerical minimax-regret policy.}
We visualize a policy that approximately solves the worst-case regret
problem of \cref{prop:app_opt_inv_worst_case_regret}
in \cref{fig:worst_case_regret_regime_switch_control_quantiles}.
The visualized policy is parameterized as
\begin{equation*}
	U_t^\theta
	=
	\frac{
			w_{1,\theta}(q_t)\bar Y_{t-1}
			+
			w_{2,\theta}(q_t)m_t^{\mathrm{KF}}
			+
			d_\theta(q_t)
	}{
			\lambda(\sigma^2+r_\theta(q_t))
	},
	\qquad
	q_t=\frac{\log t}{\log(1+T)} ,
\end{equation*}
where $\bar Y_{t-1}$ is the empirical mean of the observed returns
and $m_t^{\mathrm{KF}}$ is the predictive mean from the steady-state
LG-SSM Kalman filter in \cref{prop:app_opt_inv_example_lg}.
This parameterization lets the policy use both the full historical sample mean
$\bar Y_{t-1}$ and the Kalman mean $m_t^{\mathrm{KF}}$,
which forgets older observations.
It therefore does not force the policy to inherit the specialization behavior
of the static latent policy.
The functions $w_{1,\theta},w_{2,\theta},d_\theta,r_\theta$
are piecewise-linear splines in $q_t$, with $r_\theta\ge0$.
For each fixed drift $x$, the pair
$(\bar Y_{t-1},m_t^{\mathrm{KF}})$ is jointly Gaussian,
so each one-step regret term can be evaluated in closed form.
We optimize the spline coefficients by minimizing a log-sum-exp approximation
of $\sup_x R_T(U^\theta,x)$ over a finite grid of drift values.

\subsection{Randomized Heston model}\label{app:subsec:randomized_heston_proofs}

\begin{proposition}[Uniform contraction of the Heston update]
\label{prop:heston_uniform_contraction}
Fix $\Delta>0$ and let
\begin{equation*}
\mathcal P \subset
[\underline\kappa,\overline\kappa]
\times
[\underline v,\overline v]
\times
[\underline\xi,\overline\xi]
\times
[-1,1],
\end{equation*}
where
\begin{equation*}
0<\underline\kappa\le \overline\kappa<\infty,
\qquad
0<\underline v\le \overline v<\infty,
\qquad
0<\underline\xi\le \overline\xi<\infty.
\end{equation*}
For $x=(\kappa,\bar v,\xi,\rho)\in\mathcal P$ and $v\in\mathbb R_+$,
let $(V_t^{x,v})_{0\le t\le \Delta}$ be the nonnegative solution of
\begin{equation*}
dV_t^{x,v}
=
\kappa(\bar v-V_t^{x,v})\,dt
+
\xi\sqrt{V_t^{x,v}}\,dB_t,
\qquad
V_0^{x,v}=v,
\end{equation*}
where $B$ is a Brownian motion on $[0,\Delta]$. Let $B^\perp$ be an
independent Brownian motion and define
\begin{equation*}
I^{x,v}:=\int_0^\Delta V_t^{x,v}\,dt
\end{equation*}
and
\begin{equation*}
R^{x,v}
:=
-\frac12 I^{x,v}
+
\rho\int_0^\Delta \sqrt{V_t^{x,v}}\,dB_t
+
\sqrt{1-\rho^2}
\int_0^\Delta \sqrt{V_t^{x,v}}\,dB_t^\perp .
\end{equation*}
Writing $U=(B,B^\perp)$, define the one-step update on
$\mathcal Y=\mathbb R\times\mathbb R_+$ by
\begin{equation*}
F_x((r,v),U):=(R^{x,v},V_\Delta^{x,v}).
\end{equation*}
Then there exist $\beta>0$, $c\in(0,1)$, $M<\infty$, and
$y_\star\in\mathcal Y$ such that, with
\begin{equation*}
d_\beta((r,v),(\tilde r,\tilde v))
:=
\beta |r-\tilde r|
+
|v-\tilde v|
+
\sqrt{|v-\tilde v|},
\end{equation*}
one has
\begin{equation}
\label{eq:heston_uniform_contraction}
\mathbb E\bigl[
d_\beta(F_x(y,U),F_x(\tilde y,U))
\bigr]
\le
c\,d_\beta(y,\tilde y)
\end{equation}
for all $x\in\mathcal P$ and all $y,\tilde y\in\mathcal Y$, and
\begin{equation}
\label{eq:heston_uniform_moment_bound}
\sup_{x\in\mathcal P}
\mathbb E\bigl[
d_\beta(y_\star,F_x(y_\star,U))
\bigr]
\le M.
\end{equation}
Moreover, $d_\beta$ is a complete separable metric on $\mathcal Y$.
\end{proposition}

\begin{proof}
First note that $d_\beta$ is a metric. Indeed, $|\cdot|$ is a metric and
$a\mapsto\sqrt a$ is subadditive on $\mathbb R_+$. Hence
$\sqrt{|v-\tilde v|}$ satisfies the triangle inequality. Completeness and
separability follow from the corresponding properties of
$\mathbb R\times\mathbb R_+$ under the Euclidean metric, since convergence
in $d_\beta$ is equivalent to coordinate-wise Euclidean convergence.

Set
\begin{equation*}
q:=e^{-\underline\kappa\Delta}\in(0,1),
\qquad
h_\ast
:=
\sup_{\kappa\in[\underline\kappa,\overline\kappa]}
\int_0^\Delta e^{-\kappa t}\,dt
=
\frac{1-e^{-\underline\kappa\Delta}}{\underline\kappa}.
\end{equation*}
Fix $x=(\kappa,\bar v,\xi,\rho)\in\mathcal P$ and
$y=(r,v),\tilde y=(\tilde r,\tilde v)\in\mathcal Y$. Write
\begin{equation*}
\delta:=|v-\tilde v|.
\end{equation*}
Couple $V^{x,v}$ and $V^{x,\tilde v}$ by driving them with the same
Brownian motion $B$, and set
\begin{equation*}
D_t:=V_t^{x,v}-V_t^{x,\tilde v}.
\end{equation*}
Then
\begin{equation*}
dD_t
=
-\kappa D_t\,dt
+
\xi\bigl(\sqrt{V_t^{x,v}}-\sqrt{V_t^{x,\tilde v}}\bigr)dB_t.
\end{equation*}
By the one-dimensional comparison theorem for SDEs
\citep[see e.g.,][chapter~IX]{revuz_1999_conti},
the two variance paths cannot cross
so that $D_t$ keeps the sign of $D_0=v-\tilde v$.
Moreover, the stochastic integral in the
equation for $D_t$ has mean zero, and hence
\begin{equation*}
\mathbb E\bigl[D_t\bigr]
=
D_0-\kappa\int_0^t \mathbb E\bigl[D_s\bigr]\,ds .
\end{equation*}
Solving this linear integral equation gives
\begin{equation*}
\mathbb E\bigl[D_t\bigr]
=
e^{-\kappa t}(v-\tilde v).
\end{equation*}
Since $D_t$ has the same sign as $v-\tilde v$, it follows that
\begin{equation*}
\mathbb E\bigl[|D_t|\bigr]
=
e^{-\kappa t}\delta
\le
e^{-\underline\kappa t}\delta .
\end{equation*}

Consequently,
\begin{equation}
\label{eq:heston_v_terminal_l1}
\mathbb E\bigl[
|V_\Delta^{x,v}-V_\Delta^{x,\tilde v}|
\bigr]
\le q\delta ,
\end{equation}
and by Jensen's inequality,
\begin{equation}
\label{eq:heston_v_terminal_sqrt}
\mathbb E\bigl[
\sqrt{|V_\Delta^{x,v}-V_\Delta^{x,\tilde v}|}
\bigr]
\le
q^{1/2}\sqrt\delta .
\end{equation}
For the integrated variance, the triangle inequality, Fubini's theorem,
and the definition of $h_\ast$ give
\begin{equation}
\label{eq:heston_integrated_variance_l1}
\mathbb E\bigl[
|I^{x,v}-I^{x,\tilde v}|
\bigr]
\le
\mathbb E\bigl[
\int_0^\Delta |D_t|\,dt
\bigr]
=
\int_0^\Delta
\mathbb E\bigl[|D_t|\bigr]\,dt
\le
\delta\int_0^\Delta e^{-\underline\kappa t}\,dt
=
h_\ast\delta .
\end{equation}

For the return coordinate,
\begin{align*}
R^{x,v}-R^{x,\tilde v}
&=
-\frac12(I^{x,v}-I^{x,\tilde v}) \\
&\quad
+
\rho
\int_0^\Delta
\bigl(\sqrt{V_t^{x,v}}-\sqrt{V_t^{x,\tilde v}}\bigr)dB_t \\
&\quad
+
\sqrt{1-\rho^2}
\int_0^\Delta
\bigl(\sqrt{V_t^{x,v}}-\sqrt{V_t^{x,\tilde v}}\bigr)dB_t^\perp .
\end{align*}

Using Cauchy--Schwarz and Itô's isometry, we obtain
\begin{align*}
&\mathbb E\bigl[
\bigl|
\int_0^\Delta
\bigl(\sqrt{V_t^{x,v}}-\sqrt{V_t^{x,\tilde v}}\bigr)dB_t
\bigr|
\bigr] \\
&\qquad\le
\biggl(
\mathbb E\bigl[
\int_0^\Delta
\bigl(\sqrt{V_t^{x,v}}-\sqrt{V_t^{x,\tilde v}}\bigr)^2 dt
\bigr]
\biggr)^{1/2} \\
&\qquad\le
\biggl(
\mathbb E\bigl[
\int_0^\Delta |V_t^{x,v}-V_t^{x,\tilde v}|\,dt
\bigr]
\biggr)^{1/2}
=
\biggl(
\mathbb E\bigl[
\int_0^\Delta |D_t|\,dt
\bigr]
\biggr)^{1/2}.
\end{align*}
Here the second inequality uses
$(\sqrt a-\sqrt b)^2\le |a-b|$ for $a,b\ge 0$.
By \eqref{eq:heston_integrated_variance_l1} we therefore obtain
\begin{equation*}
\mathbb E\bigl[
\bigl|
\int_0^\Delta
\bigl(\sqrt{V_t^{x,v}}-\sqrt{V_t^{x,\tilde v}}\bigr)dB_t
\bigr|
\bigr]
\le
h_\ast^{1/2}\sqrt\delta .
\end{equation*}
The same bound holds with $B^\perp$ in place of $B$. Since $|\rho|\le 1$
and $\sqrt{1-\rho^2}\le 1$, it follows that
\begin{equation}
\label{eq:heston_return_l1}
\mathbb E\bigl[
|R^{x,v}-R^{x,\tilde v}|
\bigr]
\le
\frac12 h_\ast\delta
+
2h_\ast^{1/2}\sqrt\delta .
\end{equation}

Combining
\eqref{eq:heston_v_terminal_l1},
\eqref{eq:heston_v_terminal_sqrt}, and
\eqref{eq:heston_return_l1}, we get
\begin{align*}
&\mathbb E\bigl[
d_\beta(F_x(y,U),F_x(\tilde y,U))
\bigr] \\
&\qquad=
\beta\mathbb E\bigl[
|R^{x,v}-R^{x,\tilde v}|
\bigr]
+
\mathbb E\bigl[
|V_\Delta^{x,v}-V_\Delta^{x,\tilde v}|
\bigr]
+
\mathbb E\bigl[
\sqrt{|V_\Delta^{x,v}-V_\Delta^{x,\tilde v}|}
\bigr] \\
&\qquad\le
\bigl(q+\frac12\beta h_\ast\bigr)\delta
+
\bigl(q^{1/2}+2\beta h_\ast^{1/2}\bigr)\sqrt\delta \\
&\qquad\le
\underbrace{
\max\bigl\{
q+\frac12\beta h_\ast,\,
q^{1/2}+2\beta h_\ast^{1/2}
\bigr\}
}_{=:c_\beta}
(\delta+\sqrt\delta)
\le
c_\beta d_\beta(y,\tilde y) ,
\end{align*}
where the last inequality follows from
$\delta+\sqrt\delta\le d_\beta(y,\tilde y)$. Because $q<1$ and
$q^{1/2}<1$, we may choose
\begin{equation*}
0<\beta<
\min\bigl\{
\frac{2(1-q)}{h_\ast},
\frac{1-q^{1/2}}{2h_\ast^{1/2}}
\bigr\}.
\end{equation*}
For this choice of $\beta$, we have $c_\beta<1$. Taking
$c:=c_\beta$ proves \eqref{eq:heston_uniform_contraction}.

It remains to prove the uniform moment bound. Fix any $v_\star>0$ and set
$y_\star:=(0,v_\star)$. For $x=(\kappa,\bar v,\xi,\rho)\in\mathcal P$, the
first moment of the CIR process is
\begin{equation*}
\mathbb E\bigl[V_t^{x,v_\star}\bigr]
=
\bar v+e^{-\kappa t}(v_\star-\bar v).
\end{equation*}
Hence, with $C_\star:=\max\{v_\star,\overline v\}$,
\begin{equation*}
\mathbb E\bigl[V_t^{x,v_\star}\bigr]
\le C_\star
\qquad
\text{for all }t\in[0,\Delta]\text{ and all }x\in\mathcal P.
\end{equation*}
Therefore
\begin{equation*}
\mathbb E\bigl[I^{x,v_\star}\bigr]
\le
\Delta C_\star,
\qquad
\mathbb E\bigl[
|V_\Delta^{x,v_\star}-v_\star|
\bigr]
\le
C_\star+v_\star,
\end{equation*}
and Jensen's inequality gives
\begin{equation*}
\mathbb E\bigl[
\sqrt{|V_\Delta^{x,v_\star}-v_\star|}
\bigr]
\le
\sqrt{C_\star+v_\star}.
\end{equation*}
For the return coordinate, Cauchy--Schwarz and Itô's isometry yield
\begin{align*}
\mathbb E\bigl[|R^{x,v_\star}|\bigr]
&\le
\frac12\mathbb E\bigl[I^{x,v_\star}\bigr]
+
\mathbb E\bigl[
\bigl|
\int_0^\Delta \sqrt{V_t^{x,v_\star}}\,dB_t
\bigr|
\bigr]
+
\mathbb E\bigl[
\bigl|
\int_0^\Delta \sqrt{V_t^{x,v_\star}}\,dB_t^\perp
\bigr|
\bigr] \\
&\le
\frac12\Delta C_\star
+
2(\Delta C_\star)^{1/2}.
\end{align*}
Combining these estimates, uniformly over $x\in\mathcal P$,
\begin{align*}
\mathbb E\bigl[
d_\beta(y_\star,F_x(y_\star,U))
\bigr]
&\le
\beta\bigl(
\frac12\Delta C_\star
+
2(\Delta C_\star)^{1/2}
\bigr)
+
C_\star+v_\star
+
\sqrt{C_\star+v_\star}.
\end{align*}
The right-hand side is finite and independent of $x$. This proves
\eqref{eq:heston_uniform_moment_bound}.
\end{proof}

\end{document}